\documentclass[10pt,journal,compsoc,romanappendices]{IEEEtran}
\usepackage{amsmath,amsfonts,amssymb}
\usepackage{array}
\usepackage[caption=false,font=normalsize,labelfont=sf,textfont=sf]{subfig}
\usepackage{textcomp}
\usepackage{stfloats}
\usepackage{url}
\usepackage{verbatim}
\usepackage{graphicx}
\usepackage{cite}
\usepackage{xcolor}
\usepackage{colortbl}
\usepackage{multirow}
\usepackage{booktabs}
\usepackage{makecell}
\usepackage{pifont}
\usepackage{adjustbox}
\newcommand{\eqfit}[1]{%
  \sbox0{$\displaystyle#1$}%
  \ifdim\wd0>\linewidth
    \resizebox{\linewidth}{!}{\usebox0}%
  \else
    \usebox0
  \fi}
\usepackage{tabularx}
\usepackage{amsthm}
\usepackage{pgfplots}
\pgfplotsset{compat=1.18}
\usepackage{microtype}
\usepackage{float}

\definecolor{blue}{rgb}{0,0,0}
\definecolor{cvprblue}{rgb}{0.21,0.49,0.74}
\definecolor{dropred}{RGB}{180,32,32}
\definecolor{darkblue}{RGB}{0,102,153}
\definecolor{darkred}{RGB}{153,0,0}
\definecolor{darkgreen}{RGB}{0,102,0}
\definecolor{purple}{RGB}{102,0,102}
\newcommand{\gain}[1]{\textcolor{cvprblue}{#1}}
\newcommand{\drop}[1]{\textcolor{dropred}{#1}}

\usepackage[colorlinks=true, linkcolor=cvprblue, citecolor=cvprblue, urlcolor=cvprblue, pdfborder={0 0 0}]{hyperref}

\newcommand{\aref}[1]{\ref{#1}}
\newcommand{\sref}[1]{\ref{#1}}
\newcommand{\mref}[1]{\ref{#1}}
\newcommand{\meqref}[1]{\eqref{#1}}

\newtheorem{corollary}{Corollary}
\newtheorem*{corollary*}{Corollary}
\theoremstyle{remark}

\begin{document}

\title{MotionRFT: Unified Reinforcement Fine-Tuning for Text-to-Motion Generation}

\author{Xiaofeng~Tan*, Wanjiang~Weng*, Hongsong~Wang\textsuperscript{\dag}, Fang~Zhao, Xin~Geng, Senior Member, IEEE, and Liang Wang, Fellow, IEEE
\IEEEcompsocitemizethanks{\IEEEcompsocthanksitem This work was supported by National Natural Science Foundation of China (62302093, 52441503, 62125602, U24A20324, 92464301), Jiangsu Province Natural Science Foundation (BK20230833, BG2024036, BK20243012), the New Cornerstone Science Foundation through the XPLORER PRIZE, and the Fundamental Research Funds for the Central Universities (2242025K30024), the Open Research Fund of the State Key Laboratory of Multimodal Artificial Intelligence Systems under Grant E5SP060116, and the Big Data Computing Center of Southeast University.
\IEEEcompsocthanksitem X. Tan, W. Weng, H. Wang and X. Geng are with School of Computer Science and Engineering, Southeast University, Nanjing 211189, China, and also with Key Laboratory of New Generation Artificial Intelligence Technology and Its Interdisciplinary Applications (Southeast University), Ministry of Education, China (email: \{xiaofengtan, wjweng, hongsongwang, xgeng\}@seu.edu.cn).
\IEEEcompsocthanksitem F. Zhao is with State Key Laboratory for Novel Software Technology, School of Intelligence Science and Technology, Nanjing University, Suzhou 215163, China (email: zhaofang0627@gmail.com).
\IEEEcompsocthanksitem L. Wang is with New Laboratory of Pattern Recognition (NLPR), State Key Laboratory of Multimodal Artificial Intelligence Systems (MAIS), Institute of Automation, Chinese Academy of Sciences (CASIA), Beijing 100190, China, and also with School of Artificial Intelligence, University of Chinese Academy of Sciences, Beijing 100049, China.
\IEEEcompsocthanksitem *Equal Contribution. \textsuperscript{\dag} Corresponding author: Hongsong Wang.}}

\markboth{IEEE Transactions on Pattern Analysis and Machine Intelligence}%
{Tan \MakeLowercase{\textit{et al.}}: MotionRFT: Unified Reinforcement Fine-Tuning for Text-to-Motion Generation}

\IEEEtitleabstractindextext{%
\begin{abstract}
Text-to-motion generation has advanced with diffusion- and flow-based generative models, yet supervised pre-training remains insufficient to align models with high-level objectives such as semantic consistency, realism, and human preference. Existing post-training methods have key limitations:  they (1) target a specific motion representation, such as joints, (2) optimize a particular aspect, such as text-motion alignment, and may compromise other factors; and (3) incur substantial computational overhead, data dependence, and coarse-grained optimization. We present a reinforcement fine-tuning framework that comprises a heterogeneous-representation, multi-dimensional reward model, {MotionReward}, and an efficient, fine-grained fine-tuning method, {EasyTune}. To obtain a unified semantics representation, {MotionReward} maps heterogeneous motions into a shared semantic space anchored by text, enabling multi-dimensional reward learning; Self-refinement Preference Learning further enhances semantics without additional annotations. For efficient and effective fine-tuning, we identify the recursive gradient dependence across denoising steps as the key bottleneck, and propose {EasyTune}, which optimizes step-wise rather than over the full trajectory, yielding dense, fine-grained, and memory-efficient updates. Extensive experiments validate the effectiveness of our framework, achieving FID 0.132 at 22.10\,GB peak memory for MLD model and saving up to 15.22 GB over DRaFT. It reduces FID by 22.9\% on joint-based ACMDM, and achieves a {\color{blue}12.1\%} R-Precision gain and {\color{blue}21.7\%} FID improvement on rotation-based {\color{blue}HY-Motion}. Our \href{https://xiaofeng-tan.github.io/projects/MotionRFT/}{project page} with code is publicly available.
\end{abstract}

}

\maketitle

\IEEEdisplaynontitleabstractindextext
\IEEEpeerreviewmaketitle

\IEEEraisesectionheading{\section{Introduction}\label{sec:intro}}
\IEEEPARstart{T}{ext-to-motion} generation aims to synthesize realistic and coherent human motions from natural language, enabling applications in animation \cite{azadi2023make}, robotics \cite{mu2025smp}, and virtual reality \cite{tashakori2025flexmotion}. The rapid advancement of generative model \cite{ho2020denoising, lipman2023flow} has propelled significant progress in text-to-motion generation. Mainstream generative paradigms, including diffusion models~\cite{Ma2025Efficient,tevet2023human,chen2023executing,shen2025finextrolcontrollablemotiongeneration} and autoregressive models~\cite{jiang2023motiongpt,zhang2023generating,guo2023momask}, have demonstrated remarkable capabilities in synthesizing realistic human motions from natural language descriptions~\cite{Zhu2024Human}.

Despite significant advances in generation quality and efficiency, existing models~\cite{ho2020denoising,tevet2023human,chen2023executing,zhang2022motiondiffuse} are still primarily trained under supervised learning, fitting annotated data distributions~\cite{guo2022generating}. However, maximizing data likelihood does not necessarily optimize high-level objectives such as perceptual quality~\cite{motioncritic2025} or semantic consistency~\cite{tan2024sopo,tan2026consistentrftreducingvisualhallucinations}, thereby limiting real-world performance. As a result, current text-to-motion models still face considerable challenges in semantic alignment \cite{tan2024sopo}, human preferences \cite{motioncritic2025}, and motion realism \cite{HanReinDiffuse2025}.

Reinforcement learning (RL) \cite{schulman2017proximal} has been explored to align diffusion-based generation with task-specific goals via reinforcement learning from human feedback (RLHF) \cite{kirstain2023pick, ziegler2019fine}. Existing approaches include differentiable reward methods \cite{clark2024directly}, policy-gradient RL \cite{black2023training, xue2025dancegrpo, liu2025flow}, and direct preference optimization (DPO) \cite{Wallace_2024_CVPR}, which treat the denoising trajectory as a Markov Decision Process and optimize reward signals from motion discriminative models \cite{petrovich2023tmr,weng2025, wang2025foundation, wang2025heterogeneous, weng2025usdrl} to correct behaviors not addressed by supervised training~\cite{wu2025generalization}.

In motion generation, recent works have explored use RL to enhance the motion quality. Specifically, SoPo~\cite{tan2024sopo} constructs semi-online preference pairs for semantic consistency, MotionCritic~\cite{motioncritic2025} employs PPO for human-preference alignment, and ReinDiffuse~\cite{HanReinDiffuse2025} adopts rule-based rewards for motion realism.

However, these methods still face several limitations. \textbf{First}, they are typically developed for generative models with a specific motion representation, which makes it difficult to generalize to alternative representations. Beyond the widely used kinematic representation~\cite{guo2022generating}, recent studies~\cite{meng2025absolute, zhao2024dartcontrol, fan2025go, hymotion2025, barquero2024seamless} have explored generating motions directly in rotation  or joint-based representations ~\cite{zhao2024dartcontrol, hymotion2025, meng2025absolute} and have demonstrated improved scalability~\cite{hymotion2025}; however, existing methods do not readily transfer to these settings. \textbf{Second}, most prior approaches optimize for a specific aspect, such as text-motion alignment~\cite{tan2024sopo} or human preference~\cite{motioncritic2025}, which may bias the model toward that objective and potentially compromise other dimensions of motion quality. \textbf{Finally}, these methods often come with substantial computational overhead  \cite{wu2025generalization}, coarse-grained optimization, or heavy data dependence, which limits practical adoption in real-world scenarios.

\noindent \textbf{Contributions.} We propose a reinforcement fine-tuning framework for text-to-motion generation. It comprises a heterogeneous-representation, multi-dimensional reward model, \textit{MotionReward}, and an efficient and effective fine-tuning method, \textit{EasyTune}. Our main contributions are summarized as follows.

Our first contribution lies in a unified heterogeneous-representation reward model, MotionReward, for multi dimensional evaluation in text conditioned motion generation. MotionReward aims to overcome three challenges: \textbf{(1)} the discrepancy across heterogeneous motion representations that makes reward signals hard to compare and generalize, \textbf{(2)} the limitation of single objective rewards that fail to jointly capture motion realism, motion to text alignment, and human perceived preference, and \textbf{(3)} the scarcity of large scale, high quality preference annotations. To address the representation discrepancy, \underline{\textbf{(1)}} we establish a unified semantic representation, where the heterogeneous motion representations are mapped into a shared embedding space via lightweight {\color{blue}projection layers} and aligned with the textual anchor. This unified semantic representation serves as the natural text-motion alignment reward and provides a unified representation for heterogeneous motions. \underline{\textbf{(2)}} Building upon this unified semantic representation, we adopt a lightweight LoRA parameter to transfer the semantic representation into other reward domain, where the specific reward, such as human preference and motion authenticity, can be captured and learned in an efficient way. \underline{\textbf
{(3)}} To address the lack of large-scale datasets, we propose Self-refinement Preference Learning (SPL) to refine the pre-trained reward model without additional data. SPL dynamically constructs preference pairs by using pre-training samples as positives and failed retrieval results as negatives, and use them to fine-tune the pre-trained model. With this unified semantic modeling, MotionReward can be seamlessly applied to reinforcement fine-tuning of text-to-motion generators across kinematic-, rotation-, and joint-based representations in semantics, preference, and authenticity aspects.

Second, we analyze the drawbacks of existing differentiable reward based methods and propose an efficient reinforcement fine tuning method, EasyTune. Specifically, we theoretically, in Corollary~\ref{thm:t1}, and empirically, in Fig.~\ref{fig:memory}, identify the key factor of computational and memory overhead: optimization is recursively coupled across the multi step denoising trajectory, meaning that the reward of the final motion $\mathcal{R}_\phi(\mathbf{x}^\theta)$ depends on all intermediate states and the gradient $\nabla_\theta \mathbf{x}_t^\theta$ must backpropagate through the entire reverse chain, which requires storing the full computational graph. We further show in Fig.~\ref{fig:norm} that such coarse-grained chain optimization causes vanishing gradients, making early denoising steps difficult to improve. Fortunately, intermediate motions are more perceivable than intermediate images or videos, which makes step aware optimization feasible, as illustrated in Fig.~\ref{fig:motion_image_cmp} in Sec.~\ref{sec:motivation}. Building upon these, we propose {\emph{EasyTune}}, which performs reinforcement fine tuning at each denoising step to decouple gradients from the full trajectory, enabling dense and fine-grained optimization and substantial memory savings. To accurately estimate the reward of noisy intermediate motions, we design step specific reward perception mechanisms, including single step prediction rewards for ODE based models and noise-aware rewards for SDE based models. Together, these techniques enable efficient and effective reinforcement fine-tuning of text-to-motion generation.

\begin{figure}[t!]
    \centering
    \begin{minipage}[t]{0.48\columnwidth}
        \centering
        \includegraphics[width=\linewidth]{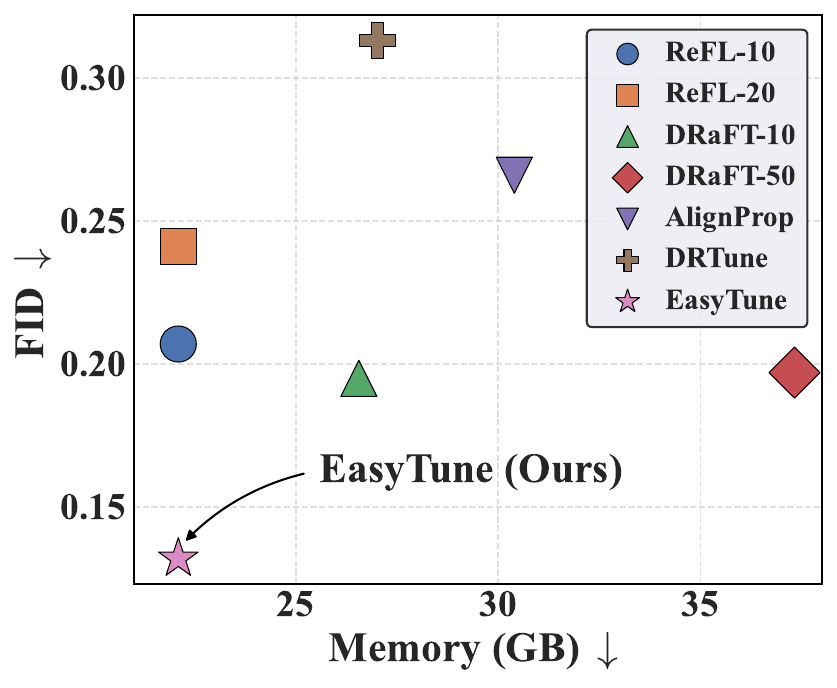}
    \end{minipage}
    \hfill
    \begin{minipage}[t]{0.48\columnwidth}
        \centering
        \includegraphics[width=\linewidth]{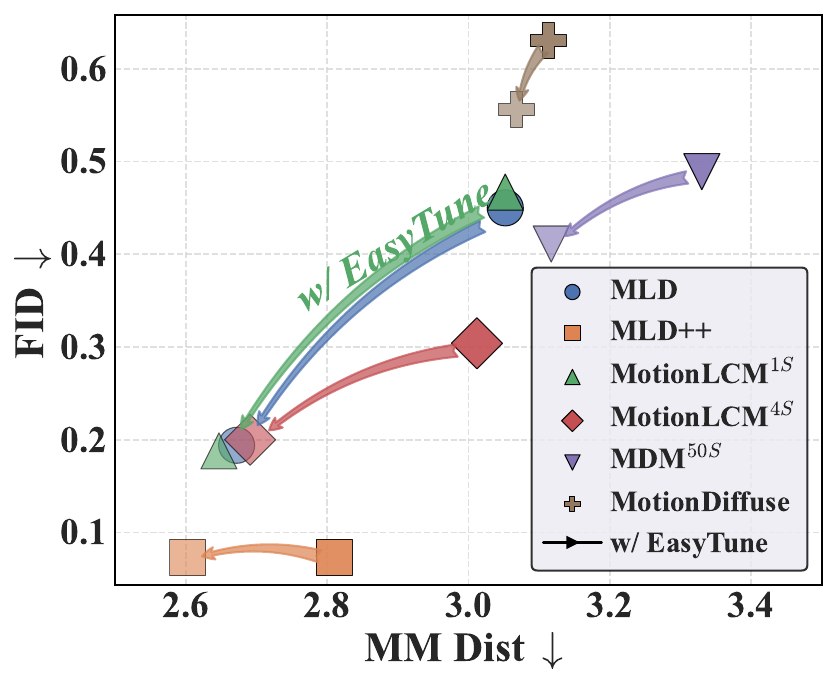}
    \end{minipage}
    \caption{ {Comparison of the training costs and generation performance on HumanML3D \cite{guo2022generating}.} {(a)} Performance comparison of different fine-tuning methods \cite{clark2024directly, prabhudesai2023aligning, wu2025drtune}. {(b)} Generalization performance across six pre-trained diffusion-based models \cite{chen2023executing, motionlcm-v2, Dai2025, tevet2023human, zhang2022motiondiffuse}.}
    \label{fig:intro}
    \vspace{-18pt}
\end{figure}

As shown in Fig.~\ref{fig:intro}, EasyTune achieves state-of-the-art performance with an FID of {0.132}, a {70.7\%} improvement over MLD~\cite{chen2023executing}, while reducing peak memory usage to {22.10}~GB. Moreover, we demonstrate that our reinforcement fine-tuning framework generalizes across motion representations. It is effective not only for kinematic-based models such as MLD~\cite{chen2023executing} and MDM~\cite{tevet2023human}, but also for joint-based models such as ACMDM~\cite{meng2025absolute}, achieving a {22.9\%} FID reduction in terms of generation quality, and for rotation-based models such as {\color{blue}HY-Motion}~\cite{hymotion2025}, achieving a {\color{blue}12.1\%} R-Precision gain and {\color{blue}21.7\%} FID improvement, consistently improving generation quality and text-motion consistency.

\noindent This work is an extension of our published conference version~\cite{tan2026easytune}. Compared to the preliminary version, the major extensions are as follows:
\begin{itemize}
    \item We propose a unified motion reward model for post-training text-to-motion generation models across heterogeneous motion representations, including kinematic-, joint-, and rotation-based representations.
    \item We introduce a unified reward modeling framework in a shared semantic latent space, enabling a single model to jointly learn motion realism, text-to-motion alignment, and human-perceived preference.
    \item We introduce Curriculum Timestep Scheduling for EasyTune \cite{tan2026easytune} to balance the optimization between early and late  steps in diffusion/flow models.
    \item Beyond kinematic-based diffusion models, extensive experiments show that our method consistently improves generation quality and text-motion consistency across diverse architectures and motion representations, including rotation-based flow models such as {\color{blue}HY-Motion}~\cite{hymotion2025} and joint-based diffusion models such as ACMDM~\cite{meng2025absolute}.
\end{itemize}

\section{Related Works}
\label{sec:related}

\noindent \textbf{Text-to-Motion Generation.} Text-to-motion generation synthesizes human motion from text prompts and is commonly studied with three representations: kinematic-based~\cite{chen2023executing, guo2023momask}, rotation-based~\cite{holden2017phase, hymotion2025}, and joint-based formats~\cite{meng2025absolute}. Early methods largely follow HumanML3D~\cite{guo2022generating} and operate on kinematic-based local, relative features, making kinematic-based representations the prevailing settings~\cite{tevet2023human, zhang2022motiondiffuse, Dai2025, tan2024sopo}. {\color{blue}Recent works further extend text-driven generation to discrete kinematic tokenization~\cite{gu2026bridgingsemantickinematicconditions} and collective motion~\cite{cao2026crowdmogen}.} {\color{blue}Yet, kinematic-based representations may be redundant~\cite{meng2024rethinking} and suffer from error accumulation~\cite{meng2025absolute}.} Rotation-based representations~\cite{martinez2017human, petrovich2021actor} have also been widely explored, including quaternions~\cite{pavllo2018quaternet}, exponential maps~\cite{martinez2017human}, and axis--angle rotations~\cite{pavlakos2019expressive}. {\color{blue}For instance, rotation-based methods such as PRIMAL~\cite{zhang2025primal} and Interact2Ar~\cite{ruizponce2026interact2ar} map rotation representations to 3D joint positions and velocities via a differentiable forward-kinematics layer.} However, these pretrained models still suffer from semantic alignment~\cite{tan2024sopo}, human preferences~\cite{motioncritic2025}, and motion realism~\cite{HanReinDiffuse2025}.

\noindent \textbf{Reinforcement Fine-tuning in Motion Generation.} To address the aforementioned issues, recent studies have explored reinforcement fine-tuning for text-to-motion models, focusing on two key aspects: semantic coherence~\cite{tan2024sopo} and physical realism~\cite{yuan2023physdiff, HanReinDiffuse2025}. Specifically, Tan et al.~\cite{tan2024sopo} construct semi-online preference pairs of semantically aligned and misaligned motions and optimize the model with a DPO-based objective. Wang et al.~\cite{motioncritic2025} curate human preference data and apply PPO to enhance realism. Han et al.~\cite{HanReinDiffuse2025} adopt rule-based rewards. Tan et al.~\cite{tan2024frequency} dynamically generate fake motions to improve motion detection ability. However, these methods largely remain confined to mainstream kinematic-based representations and typically focus on a single aspect, such as semantic coherence~\cite{tan2024sopo} or physical realism~\cite{yuan2023physdiff, HanReinDiffuse2025}. This limitation motivates us to develop a unified reward model that improves human motion across multiple dimensions and accommodates diverse motion representations.

\noindent \textbf{Reinforcement Fine-tuning for Diffusion Models.} Reward-based alignment of diffusion models broadly falls into three categories. Policy-gradient methods, such as DDPO and DPOK \cite{black2023training, NEURIPS2023_fc65fab8}, formulate denoising as an MDP and optimize expected rewards, while DPO-based methods, such as Diffusion-DPO \cite{Wallace_2024_CVPR} and SoPo \cite{tan2024sopo}, optimize a preference objective from reward-induced preference pairs.  However, policy-gradient updates typically rely on exactly computable model likelihoods, an assumption natural for autoregressive models but is not always satisfied by diffusion models \cite{zheng2025diffusionnft}, and DPO introduces an explicit dependence on preference data. These challenges often hinder reward-based fine-tuning from achieving satisfactory results. To seek a simple and effective alternative, we employ differentiable reward-based methods by maximizing reward values \cite{kim2022diffusionclip, clark2024directly, prabhudesai2023aligning, wu2025drtune} to obtain satisfactory results.  Despite its effectiveness, this method may require recursive gradients through denoising and thus suffer from sparse gradients, slow convergence, and high memory costs, discussed in Sec. \ref{sec:motivation}.

\section{Unified Reward}
\label{sec:unified_reward}
\begin{figure*}[t]
    \centering
    \includegraphics[width=\textwidth]{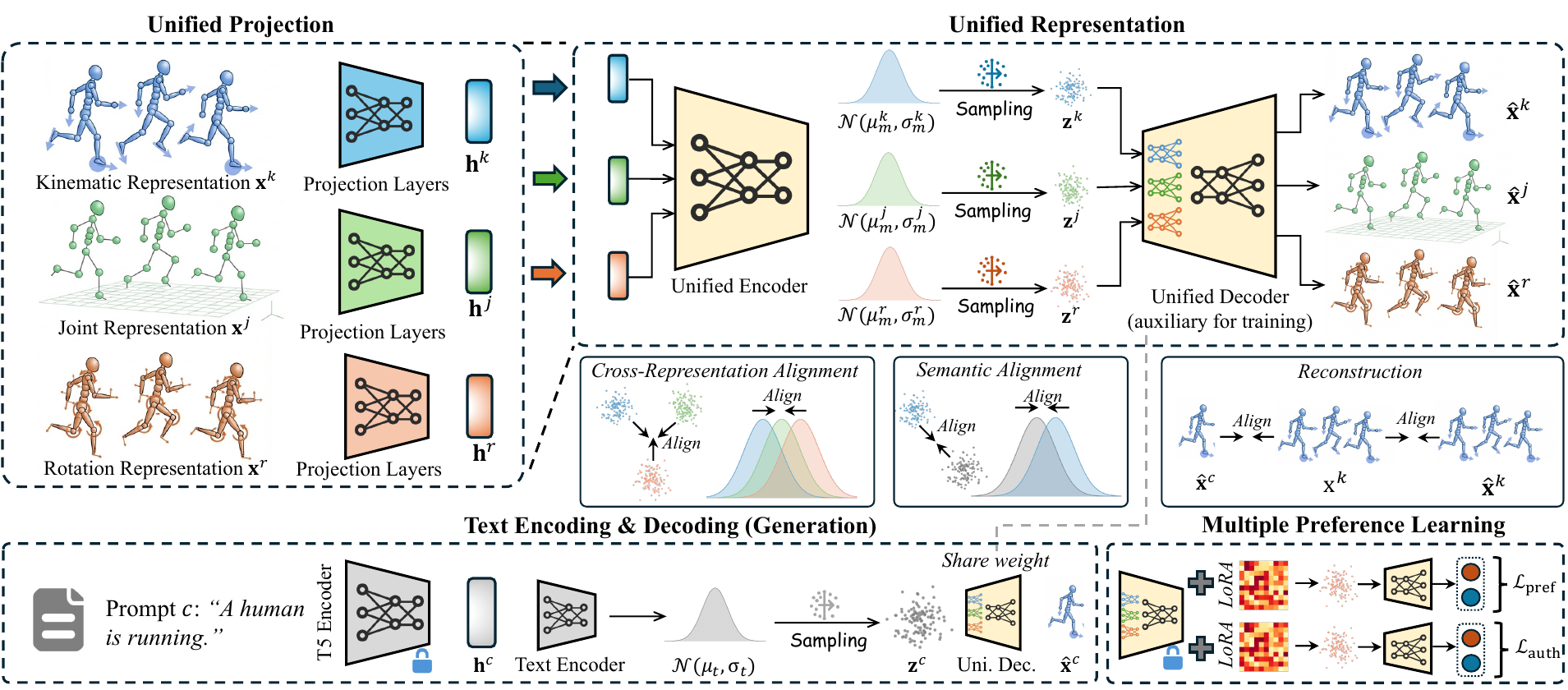}
    \caption{Overview of MotionReward, consisting of unified projection, representation, and multiple preference learning.}
    \label{fig:motionreward_framework}
\end{figure*}

\subsection{Overview}
\label{sec:overview}

In this section, we present a unified reward model for multi-dimensional motion evaluation across kinematic, joint, and rotation representations, covering semantic alignment, human preference, and realism. We begin with the problem formulation.

\noindent \textbf{Problem Formulation.} We consider three motion representations $o\in\mathcal{O}=\{k,j,r\}$, where $k$, $j$, and $r$ correspond to kinematic features $\mathbf{x}^{k}$, joint coordinates $\mathbf{x}^{j}$, and rotation parameters $\mathbf{x}^{r}$. Our goal is to learn a unified reward model $\mathcal{R}_\phi$ that evaluates motion quality in terms of semantic alignment between motion and text, human preference, and motion realism. Given a text-motion dataset $\mathbb{D}_{\mathrm{sem}}=\{(\mathbf{x}^{o}, c)\}$, a preference dataset $\mathbb{D}_{\mathrm{pref}}=\{(\mathbf{x}^{\mathrm{w},o}, \mathbf{x}^{\mathrm{l},o})\}$, and a deepfake motions dataset $\mathbb{D}_{\mathrm{df}}=\{\mathbf{x}^{+,o}, \mathbf{x}^{-,o}\}$, we train $\mathcal{R}_\phi$ to satisfy $\mathcal{R}_\phi(\mathbf{x}^{o}, c) > \mathcal{R}_\phi(\mathbf{x'}^{o}, c)$ for any $\mathbf{x'}^{o}\in\mathbb{D}_{\mathrm{sem}}$, $\mathcal{R}_\phi(\mathbf{x}^{\mathrm{w}, o}) > \mathcal{R}_\phi(\mathbf{x}^{\mathrm{l}, o})$, and $\mathcal{R}_\phi(\mathbf{x}^{+,o}) > \mathcal{R}_\phi(\mathbf{x}^{-, o})$.

\noindent \textbf{Data Construction.} Training a unified reward model requires text-motion, preference, and deepfake data across three representations, yet existing datasets \cite{guo2022generating,Zhu2024Human} are designed for a single format. We convert them into all three representations via forward kinematics and generate synthetic motions for deepfake detection.

\noindent \textbf{Overview.} As shown in Fig.~\ref{fig:motionreward_framework}, our model integrates \emph{unified representation learning} and \emph{multiple preference learning}. We first map heterogeneous motions into a shared embedding space with text anchors, which inherently yields a semantic alignment reward. We then transfer this semantic space to task-specific subspaces via lightweight LoRA adapters with a critic head for preference scoring and an MLP head for deepfake classification, achieving performance comparable to or better than task-specific reward models.

\subsection{Unified Representation Learning}
\label{sec:unified_repr}
MotionReward learns a unified semantic space by projecting heterogeneous motions into a shared feature space with lightweight adapters and aligning them with text embeddings.

\noindent \textbf{Architecture.} The model consists of representation-specific projection layers $\{\phi^{o}\}_{o\in\mathcal{O}}$, a shared motion encoder $\mathcal{E}_\mathrm{M}$, a text encoder $\mathcal{E}_\mathrm{T}$, and decoders $\{\mathcal{D}^{o}\}_{o\in\mathcal{O}}$, where each decoder is a shared module followed by a representation-specific linear layer. For three representations $o\in\mathcal{O}=\{k,j,r\}$, corresponding to kinematic features $\mathbf{x}^{k}\!\in\!\mathbb{R}^{T\times 263}$, joint coordinates $\mathbf{x}^{j}\!\in\!\mathbb{R}^{T\times 22\times 3}$, and rotation parameters $\mathbf{x}^{r}\!\in\!\mathbb{R}^{T\times 135}$, each $\mathbf{x}^o$ is projected into a shared feature space:
\begin{equation}\small
\mathbf{h}^o=\phi^{o}(\mathbf{x}^o)\in\mathbb{R}^{T\times d},\quad \forall o\in\mathcal{O},
\label{eq:proj_unified}
\end{equation}
{\color{blue} where $d{=}256$ is the dimension of the unified feature space. The different projection layers share an identical two-layer MLP architecture differing only in input dimension, minimizing representation-specific capacity before the shared encoder.} We adopt a VAE backbone where $\mathcal{E}_\mathrm{M}$ and $\mathcal{E}_\mathrm{T}$ parameterize diagonal Gaussian posteriors $q_m^o{=}\mathcal{N}(\mu_m^o,\sigma_m^{o,2})$ and $q_c{=}\mathcal{N}(\mu_c,\sigma_c^2)$:
\begin{equation}\small
\mu_m^o,\sigma_m^o=\mathcal{E}_\mathrm{M}(\mathbf{h}^o),\quad \mu_c,\sigma_c=\mathcal{E}_\mathrm{T}(c),\quad \forall o\in\mathcal{O}.
\label{eq:vae_params}
\end{equation}
Latent embeddings $\mathbf{z}_m^o\!\sim\! q_m^o$ and $\mathbf{z}_c\!\sim\! q_c$ are sampled via reparameterization, and each $\mathbf{z}_m^o$ is decoded as $\hat{\mathbf{x}}^o{=}\mathcal{D}^{o}(\mathbf{z}_m^o)$ for reconstruction. We next describe the optimization objectives.

\noindent \textbf{Text-Motion Semantic Alignment.} Motivated by ReAlign~\cite{weng2025}, we align text and motion embeddings through four complementary objectives to construct semantics space.

\noindent \textit{\textbf{(i)} Semantic reconstruction loss}~\cite{weng2025}. To ensure semantic fidelity, we reconstruct motions from both motion and text embeddings using the decoder $\mathcal{D}^{o}$, yielding $\hat{\mathbf{x}}_m^o=\mathcal{D}^{o}(\mathbf{z}_m^o)$ and $\hat{\mathbf{x}}_c^o=\mathcal{D}^{o}(\mathbf{z}_c)$. The reconstruction loss is
\begin{equation}
\small
\mathcal{L}_{\mathrm{rec}}
=\ell_{1}(\hat{\mathbf{x}}_c^o,\mathbf{x}^o)+\ell_{1}(\hat{\mathbf{x}}_m^o,\mathbf{x}^o)+\ell_{1}(\hat{\mathbf{x}}_m^o,\hat{\mathbf{x}}_c^o),
\label{eq:rec}
\end{equation}
where $\ell_{1}$ is the smooth Huber loss~\cite{petrovich2022temos}. The three terms encourage text-decoded reconstruction, motion reconstruction fidelity, and cross-modal alignment, respectively.

\noindent \textit{\textbf{(ii)} Distribution regularization loss}~\cite{weng2025}: To regularize the latent space, we constrain the posteriors to be mutually consistent and close to the standard Gaussian prior $p_0=\mathcal{N}(\mathbf{0},\mathbf{I})$:
\begin{equation*}
\eqfit{\mathcal{L}_{\mathrm{kl}}
=D_{\mathrm{KL}}(q_c\|q_m^o)+D_{\mathrm{KL}}(q_m^o\|q_c) +D_{\mathrm{KL}}(q_m^o\|p_0)+D_{\mathrm{KL}}(q_c\|p_0).}
\label{eq:kl}
\end{equation*}
The first two terms align motion and text distributions, while the last two terms prevent posterior collapse.

\noindent \textit{\textbf{(iii)} Semantic consistency loss}~\cite{weng2025}: To align motion and text embeddings in the latent, we minimize their $\ell_1$ distance:
\begin{equation}
\small
\mathcal{L}_{\mathrm{lat}}
=\ell_{1}(\mathbf{z}_m^o,\mathbf{z}_c),
\label{eq:lat}
\end{equation}
which encourages the  semantic alignment.

\noindent \textit{\textbf{(iv)} InfoNCE loss}~\cite{weng2025}: To enhance discriminative alignment, we adopt contrastive learning over a batch of $B$ motion-text pairs $\{(\mathbf{x}_i^o,c_i)\}_{i=1}^B$, as follow:
\begin{equation}
\small
\mathcal{L}_{\mathrm{cl}}=\mathcal{L}_{\mathrm{info}}(\mathbf{z}_c, \mathbf{z}_m^o),
\end{equation}
where the InfoNCE loss $\mathcal{L}_{\mathrm{info}}(\mathbf{a},\mathbf{b})$ is defined as:
\begin{equation}
\small
\mathcal{L}_{\mathrm{info}}(\mathbf{a},\mathbf{b})
=-\mathbb{E}\Biggl[\log\frac{\exp(s_{ii}/\tau)}{\sum_{j}\exp(s_{ij}/\tau)}+\log\frac{\exp(\bar{s}_{ii}/\tau)}{\sum_{j}\exp(\bar{s}_{ij}/\tau)}\Biggr],
\label{eq:infonce}
\end{equation}
with $s_{ij}=\cos(\mathbf{a}_i,\mathbf{b}_j)$, $\bar{s}_{ij}=\cos(\mathbf{b}_i,\mathbf{a}_j)$, and $\tau$ as a temperature parameter. This InfoNCE loss encourages the latent semantics alignment by contrastive learning.

\noindent \textbf{Cross-Representation Alignment.} Beyond text-motion alignment, cross-representation alignment is also essential for motion representations. To this end, we further align embeddings across different motion formats. Given the same motion in different representation $o_1, o_2$, we encourage their latent embeddings $\mathbf{z}_m^{o_1}$ and $\mathbf{z}_m^{o_2}$ to be consistent, as follows:
\begin{equation}
\small
\begin{split}
\mathcal{L}_{\mathrm{CRA}}
=\mathbb{E}_{o_1,o_2}\Bigl[
\alpha_1\ell_{1}(\mathbf{z}_m^{o_1},\mathbf{z}_m^{o_2})
&+\alpha_2 D_{\mathrm{JS}}(q_m^{o_1}\|q_m^{o_2})\\
&\quad \quad  +\alpha_3\mathcal{L}_{\mathrm{info}}(\mathbf{z}_m^{o_1},\mathbf{z}_m^{o_2})\Bigr],
\end{split}
\label{eq:CRA}
\end{equation}
where $D_{\mathrm{JS}}(\cdot\|\cdot)$ denotes the JS divergence~\cite{petrovich2022temos} and $\mathcal{L}_{\mathrm{info}}$ is defined in Eq.~\eqref{eq:infonce}. Here $\alpha_i$ are weighting coefficients, set as $\alpha_1=0.1$, $\alpha_2=10^{-5}$, and $\alpha_3=0.1$.

\noindent \textbf{Training Objective.} Overall, the final loss is defined as
\begin{equation}
\small
\mathcal{L}_{\mathrm{sem}}
=\mathcal{L}_{\mathrm{rec}}
+\lambda_1\mathcal{L}_{\mathrm{kl}}
+\lambda_2\mathcal{L}_{\mathrm{lat}}
+\lambda_3\mathcal{L}_{\mathrm{cl}}
+\lambda_4\mathcal{L}_{\mathrm{CRA}},
\label{eq:s1}
\end{equation}
where $\lambda_i$ are weighting coefficients, set as $\lambda_1\!=\!10^{-5}$, $\lambda_2\!=\!10^{-5}$, $\lambda_3\!=\!10^{-1}$, and $\lambda_4\!=\!10^{-1}$. {\color{blue}These values follow the conventions of prior works like TMR~\cite{petrovich2023tmr} and ReAlign~\cite{weng2025} to balance the magnitudes of the losses. A parameter sensitivity analysis is provided in App.~\sref{sec:lambda_sens}.}

\begin{figure*}[t]
    \centering
    \includegraphics[width=\textwidth]{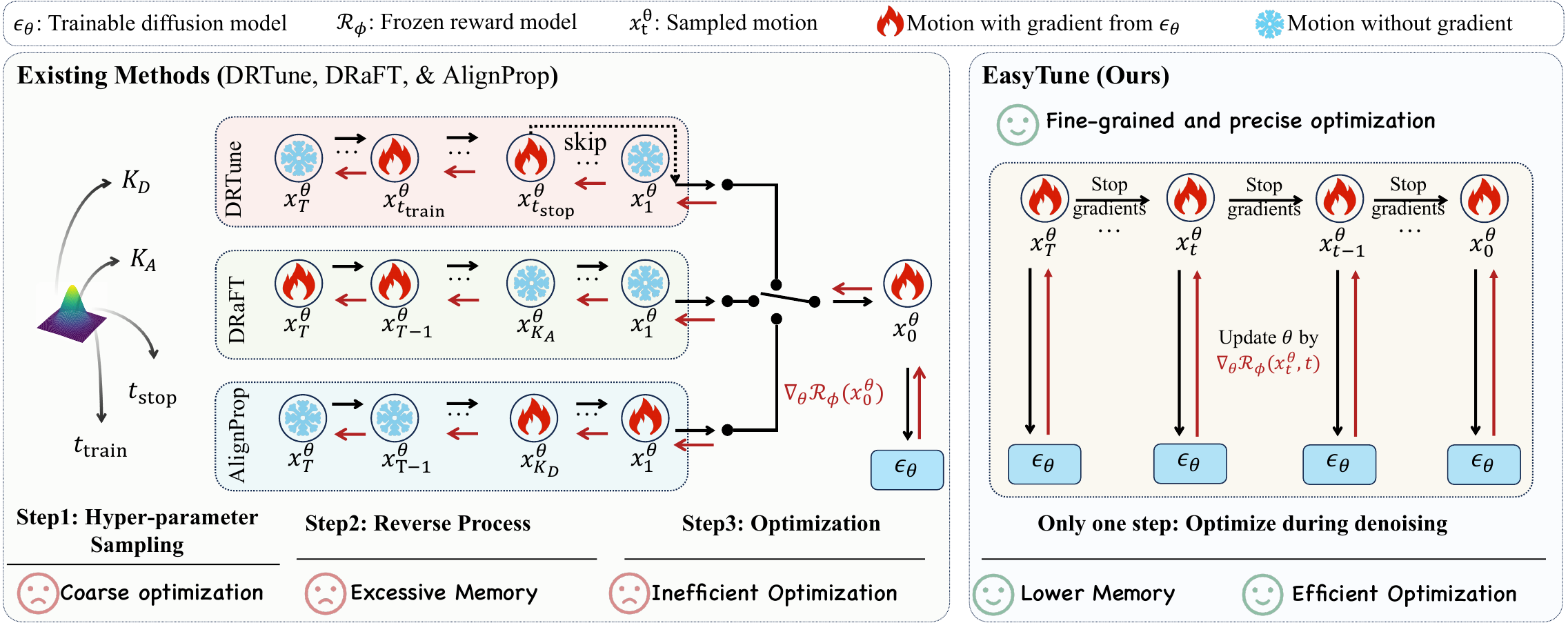}

    \caption{{The framework of existing differentiable reward-based methods (left) and our proposed EasyTune (right).} Existing methods backpropagate the gradients of the reward model through the overall denoising process, resulting in (1) excessive memory, (2) inefficient, and (3) coarse-grained optimization. In contrast, EasyTune optimizes the diffusion model by directly backpropagating the gradients at each denoising step, overcoming these issues.}

    \label{fig:framework}
\end{figure*}

\subsection{Multiple Preference Learning}
\label{sec:multi_pref}

Building upon the unified semantic latent, we aim to learn multiple rewards, including human preference and motion authenticity. To leverage such rich semantic space, we  attach lightweight task-specific adapters to the frozen backbone, enabling modular learning of multiple rewards.

\noindent \textbf{Lightweight Task-Specific Adapters.} We freeze the backbone parameters $\theta$ and attach two task-specific LoRA adapters, $\Delta\theta_{\psi}$ and $\Delta\theta_{\omega}$, for preference and authenticity modeling, respectively. We further add a critic head $h_{\psi}$ for preference scoring and a classifier head $h_{\omega}$ for deepfake detection.

\noindent \textbf{Preference Reward Loss.} To model human preference, we train MotionReward on preference pairs $(\mathbf{x}^{\mathrm{w},o},\mathbf{x}^{\mathrm{l},o})\in\mathbb{D}_{\mathrm{pref}}$, where $\mathbf{x}^{\mathrm{w},o}$ is preferred over $\mathbf{x}^{\mathrm{l},o}$.  Both motions are encoded by the adapted encoder $\mathcal{E}_\mathrm{M}^{\psi}=\mathcal{E}_\mathrm{M}+\Delta\theta_{\psi}$ and scored by the critic head $h_{\psi}$, optimized with a ranking loss:
\begin{equation}
\small
\mathcal{L}_{\mathrm{pref}}
=-\mathbb{E}_{(\mathbf{x}^{\mathrm{w},o},\mathbf{x}^{\mathrm{l},o})}\log\sigma\bigl(h_{\psi}(\mathbf{z}^{\mathrm{w}})-h_{\psi}(\mathbf{z}^{\mathrm{l}})\bigr),
\label{eq:pref_loss}
\end{equation}
where $\mathbf{z}^{\mathrm{w}}\!=\!\mathcal{E}_\mathrm{M}^{\psi}(\mathbf{x}^{\mathrm{w},o})$, $\mathbf{z}^{\mathrm{l}}\!=\!\mathcal{E}_\mathrm{M}^{\psi}(\mathbf{x}^{\mathrm{l},o})$, and $\sigma(\cdot)$ is the sigmoid function. This objective encourages the model to assign higher scores to preferred motions.

\noindent \textbf{Authenticity Reward Loss.} To assess motion authenticity, we train a binary classifier on labeled samples $(\mathbf{x}^{o},y)\in\mathbb{D}_{\mathrm{df}}$, where $y=1$ for real motions $\mathbf{x}^{+,o}$ and $y=0$ for fake motions $\mathbf{x}^{-,o}$. Motions are encoded by $\mathcal{E}_\mathrm{M}^{\omega}=\mathcal{E}_\mathrm{M}+\Delta\theta_{\omega}$ and classified by head $h_{\omega}$:
\begin{equation*}
\small
\mathcal{L}_{\mathrm{auth}}
=-\mathbb{E}\bigl[y\log h_{\omega}(\mathbf{z}_m^{+,\omega})+(1-y)\log(1-h_{\omega}(\mathbf{z}_m^{-,\omega}))\bigr].
\end{equation*}
where $\mathbf{z}_m^{+,\omega}\!=\!\mathcal{E}_\mathrm{M}^{\omega}(\mathbf{x}^{+,o})$ and $\mathbf{z}_m^{-,\omega}\!=\!\mathcal{E}_\mathrm{M}^{\omega}(\mathbf{x}^{-,o})$. This encourages the model to distinguish real motions from generated ones.

For both tasks, we freeze the backbone $\theta$ and update only the task-specific adapters and heads, namely $\{\Delta\theta_{\psi}, h_{\psi}\}$ for preference and $\{\Delta\theta_{\omega}, h_{\omega}\}$ for authenticity. At inference, we activate the corresponding adapter-head pair to produce the desired reward without affecting other tasks. Benefiting from this shared space, these lightweight LoRA adapters enable cross-task transfer and achieve performance comparable to, or even better than, task-specific reward models.

\subsection{Self-refinement for Semantic Enhancement}
\label{sec:preference_learning}

The semantic reward is derived from text-motion alignment in the unified representation space. However, due to the scarcity of text-motion paired data, contrastive learning can only establish coarse semantic correspondence, limiting its ability to distinguish fine-grained semantic differences. To this end, we propose {Self-refinement Preference Learning (SPL)}, which constructs preference pairs from existing data via hard negative mining. We detail the SPL as follows.

\noindent \textbf{Reward Model.} Given a motion \( \mathbf{x} \) and a text description \( c \), the semantic reward is computed as:
\begin{equation}
    \begin{aligned} \small
         \mathcal{R}_\phi (\mathbf{x}, c) = \cos\bigl(\mathcal{E}_\mathrm{M}(\mathbf{x}),\; \mathcal{E}_\mathrm{T}(c)\bigr),
    \end{aligned}
    \label{eq:relevance_score_v2}
\end{equation}
where \( \mathcal{E}_\mathrm{M} \) and \( \mathcal{E}_\mathrm{T} \) are the motion and text encoders from the unified representation.

\noindent \textbf{Hard Negative Mining.} To construct preference pairs for semantic refinement, we retrieve the top-\( k \) motions from the training set based on reward scores, denoted as the retrieval set \( \mathbb{D}_\mathrm{R} \). If the ground-truth motion \( \mathbf{x}^\mathrm{gt} \) is not among the top-\( k \), the highest-scoring incorrect motion serves as a hard negative:
\begin{equation}\small
\begin{aligned}
\mathbf{x}^\mathrm{w} &= \mathbf{x}^\mathrm{gt}, \quad
\mathbf{x}^\mathrm{l} =
\begin{cases}
\arg \max_{\mathbf{x} \in \mathbb{D}_\mathrm{R}} \mathcal{R}_\phi(\mathbf{x}, c), & \text{if } \mathbf{x}^\mathrm{gt} \notin \mathbb{D}_\mathrm{R}, \\
\mathbf{x}^\mathrm{gt}, & \text{otherwise}.
\end{cases}
\end{aligned}
\label{eq:hard_neg}
\end{equation}
\noindent \textbf{Preference-based Refinement.} Given a preference pair $(\mathbf{x}^\mathrm{w}, \mathbf{x}^\mathrm{l})$, we compute probabilities from their reward scores:
\begin{equation}\small
\mathcal{P} = \mathrm{Softmax}\bigl(\mathcal{R}_\phi(\mathbf{x}^\mathrm{w}, c),\; \mathcal{R}_\phi(\mathbf{x}^\mathrm{l}, c)\bigr),
\label{eq:reward_distribution}
\end{equation}
and align $\mathcal{P}$ with a target distribution $\mathcal{Q}$, where $\mathcal{Q}{=}(1.0, 0.0)$ if $\mathbf{x}^\mathrm{w} \neq \mathbf{x}^\mathrm{l}$ and $\mathcal{Q}{=}(0.5, 0.5)$ otherwise, by minimizing:
\begin{equation}\small
    \mathcal{L}_\mathrm{SPL} (\phi) = D_{\mathrm{KL}}(\mathcal{Q} \parallel \mathcal{P}).
    \label{eq:spl_loss_v2}
\end{equation}
This refines the semantic reward by encouraging higher scores for matched motions and lower scores for hard negatives, without additional human annotations.

\section{EasyTune}
\label{sec:easytune_method}

\begin{figure*}[t]
    \centering
    \begin{minipage}[t]{0.30\textwidth}
        \centering
        \adjustbox{height=3.6cm, max width=\linewidth}{%
            \includegraphics{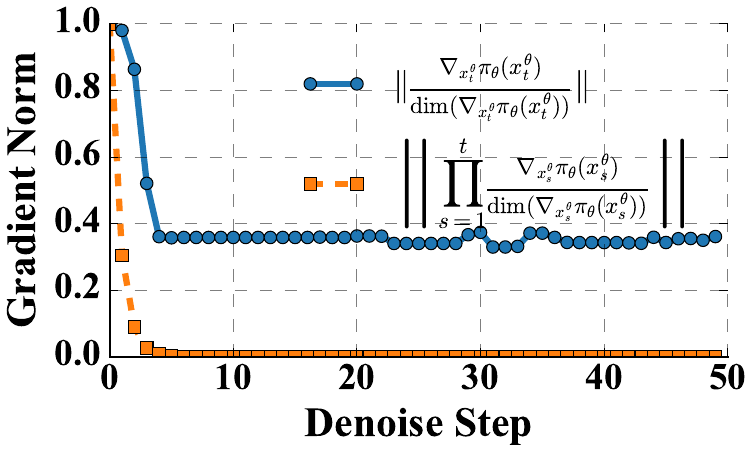}}
        \caption{{Gradient norm with respect to denoising steps.} Here, $\text{dim}(\cdot)$ denotes the gradient dimension.}
        \label{fig:norm}
    \end{minipage}
    \hfill
    \begin{minipage}[t]{0.30\textwidth}
        \centering
        \adjustbox{height=3.6cm, max width=\linewidth}{%
            \includegraphics{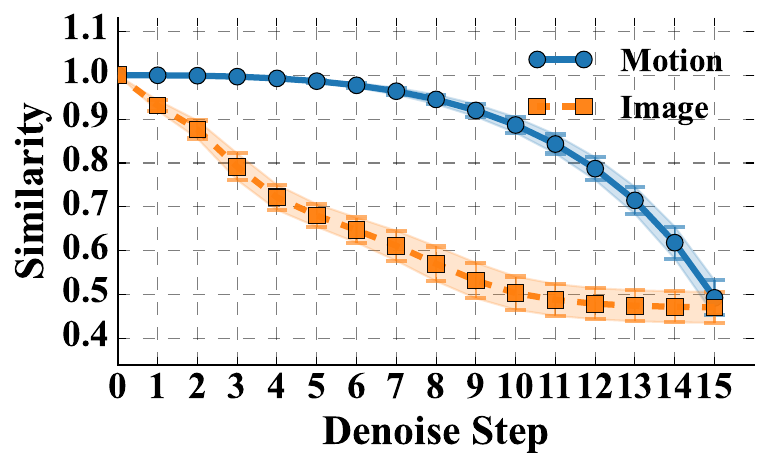}}
        \caption{Similarity between $t$-th step noised and clean motion. }
        \label{fig:motion_image_cmp}
    \end{minipage}
    \hfill
    \begin{minipage}[t]{0.30\textwidth}
        \centering
        \adjustbox{height=3.6cm, max width=\linewidth}{%
            \includegraphics{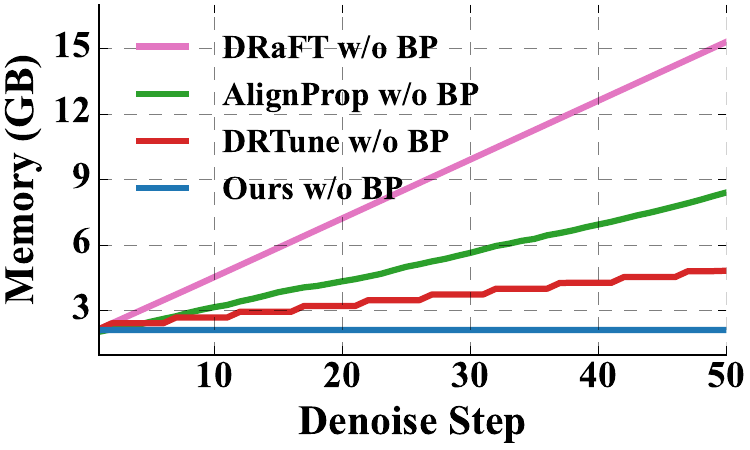}}
        \caption{{Memory usage comparison}. Here, ``w/o BP'' indicates memory measured without backpropagation.}
        \label{fig:memory}
    \end{minipage}
\end{figure*}

\subsection{Motivation: Rethinking Differentiable Reward}
\label{sec:motivation}

\noindent\textbf{Preliminaries.} Existing methods fine-tune a pre-trained motion diffusion model $\epsilon_\theta$ by maximizing the reward value $\mathcal{R}_\phi(\mathbf{x}_0^\theta, c)$ of the motion $\mathbf{x}_0^\theta$ generated via a $T$-step reverse process. As illustrated in Fig.~\ref{fig:framework}, this requires retaining gradients throughout the entire denoising trajectory $\{\mathbf{x}_T^\theta, \ldots, \mathbf{x}_0^\theta\}$, so that the model parameters can be updated by $\nabla_\theta \mathcal{R}_\phi(\mathbf{x}_0^\theta, c)$. Formally, the optimization objective is to fine-tune $\theta$ to maximize the expected reward, with the loss defined as:
\begin{equation} \small
\mathcal{L}(\theta) = -\mathbb{E}_{c \sim \mathbb{D}_\mathrm{T}, \mathbf{x}_0^\theta \sim \pi_\theta(\cdot|c)} \left[ \mathcal{R}_\phi(\mathbf{x}_0^\theta, c) \right],
\label{eq:loss}
\end{equation}
where $c$ is a text condition from the training set $\mathbb{D}_\mathrm{T}$, and $\mathbf{x}_0^\theta$ is the motion generated from noise $\mathbf{x}_T \sim \mathcal{N}(\mathbf{0}, \mathbf{I})$ via a $T$-step reverse process $\pi_\theta$. The $t$-th step of the reverse process is denoted as:
\begin{equation}\small
\mathbf{x}_{t-1}^\theta = \pi_\theta(\mathbf{x}_t^\theta, t, c) := \frac{1}{\sqrt{\alpha_t}} \left( \mathbf{x}_t^\theta - \frac{\beta_t}{\sqrt{1 - \bar{\alpha}_t}} \epsilon_\theta(\mathbf{x}_t^\theta, t, c) \right),
\label{eq:reverse_process}
\end{equation}
where $\mathbf{x}_{t-1}^\theta$ is the denoised motion at step $t-1$, and $\alpha_t$, $\beta_t$ are noise schedule parameters.

\noindent \textbf{Gradient Analysis.}  To optimize the loss in Eq.~\eqref{eq:loss}, we further analyze the gradient computation, where the gradient of $\mathcal{L}(\theta)$ w.r.t. the model parameters $\theta$ is computed via the chain rule:
\begin{equation} \small
\frac{\partial \mathcal{L}(\theta)}{\partial \theta} = -\mathbb{E}_{c \sim \mathbb{D}_\mathrm{T}, \mathbf{x}_0^\theta \sim \pi_\theta(\cdot|c)} \left[ \frac{\partial \mathcal{R}_\phi(\mathbf{x}_0^\theta, c)}{\partial \mathbf{x}_0^\theta} \cdot \frac{\partial \mathbf{x}_0^\theta}{\partial \theta} \right].
\label{eq:gradient}
\end{equation}
Here, $\tfrac{\partial \mathcal{R}_\phi(\mathbf{x}^\theta_0, c)}{\partial \mathbf{x}^\theta_0}$ represents the gradient of the reward model w.r.t. the generated motion, and $\tfrac{\partial \mathbf{x}^\theta_0}{\partial \theta}$ captures the dependence of the generated motion $\mathbf{x}^\theta_t$ on the model $\theta$ through the reverse trajectory.

Eq.~\eqref{eq:gradient} indicates that the gradient of loss function can be divided into two terms: ${\partial \mathcal{R}_\phi(\mathbf{x}^\theta_0, c)}/{\partial \mathbf{x}^\theta_0}$, which can be directly computed from the reward model, and ${\partial \mathbf{x}^\theta_0}/{\partial \theta}$, which depends on the denoising trajectory $\pi_\theta$. Here, we introduce Corollary \ref{thm:t1} to analyze this gradient (see proof in App.~\aref{supp:the}).
\begin{corollary}\label{thm:t1}
    Given the reverse process in Eq.~\eqref{eq:reverse_process}, $\mathbf{x}_{t-1}^\theta = \pi_\theta(\mathbf{x}_t^\theta, t, c)$, the gradient w.r.t diffusion model $\theta$, denoted as $\tfrac{\partial \mathbf{x}^\theta_{t-1}}{\partial \theta}$, can be expressed as:
    \begin{equation}\small
            {\frac{\partial \mathbf{x}^\theta_{t-1}}{\partial \theta}} =  \underbrace{\frac{\partial \pi_{{\theta}}(\mathbf{x}^{\theta}_t, t, c)}{\partial {\theta}}}_{\textnormal{direct term}} +
            \underbrace{\frac{\partial \pi_\theta(\mathbf{x}^{\theta}_t, t, c)}{\partial \mathbf{x}^{\theta}_t} \cdot {\frac{\partial \mathbf{x}^{\theta}_t}{\partial \theta}}}_{\textnormal{indirect term}}.
        \label{eq:recursive_gradient}
    \end{equation}
\end{corollary}
Corollary \ref{thm:t1} shows that the computation involves two parts: (1) a \emph{direct term} from the dependence of the diffusion model $\pi_{{\theta}}$ on ${\theta}$, and (2) an \emph{indirect term} that depends on the $t$-th step generated motion $\mathbf{x}^{{\theta}}_t$. However, the reverse process in diffusion models is inherently recursive, where the denoised motion $\mathbf{x}_{t-1}$ is relied on $\mathbf{x}_t$, which in turn depends on $\mathbf{x}_{t+1}$, resulting in substantial computational complexity for $T$ time steps intermediate variables.

To compute the full gradient ${\partial \mathcal{L}(\theta)}/{\partial \theta}$, we unroll the ${\partial \mathbf{x}_0^\theta}/{\partial \theta}$ using Corollary \ref{thm:t1} and substitute it into Eq.~\eqref{eq:gradient} resulting in (see proof in App.~\aref{supp:eq5}):
\begin{equation}\small
    \begin{split}
    \frac{\partial \mathcal{L}(\theta)}{\partial \theta} = &-\mathbb{E}_{c\sim \mathbb{D}_\mathrm{T}, \mathbf{x}^\theta_0 \sim \pi_\theta(\cdot|c)}  \Bigg[ \frac{\partial \mathcal{R}_\phi(\mathbf{x}^\theta_0)}{\partial \mathbf{x}^\theta_0}  \\ & \cdot \sum_{t=1}^T  \underbrace{\left( \prod_{s=1}^{t-1} \frac{\partial \pi_\theta(\mathbf{x}^\theta_s, s, c)}{\partial \mathbf{x}^\theta_s} \right)}_{\text{tend to 0 when t is larger}}
     \underbrace{\left(\frac{\partial \pi_\theta(\mathbf{x}^\theta_t, t, c)}{\partial \theta}\right)}_{\text{optimizing t-th step}} \Bigg].
    \end{split}
    \label{eq:final_gradient}
\end{equation}

\noindent \textbf{Limitations.} Eq.~\eqref{eq:final_gradient} reveals the core optimization mechanism of existing methods: the motions $\mathbf{x}^\theta_0$ are generated via the reverse process $\pi_\theta$, with the full computation graph preserved to enable the maximization of the reward $\mathcal{R}_\phi(\mathbf{x}^\theta_0, c)$. However, as shown in Fig. \ref{fig:framework}, this optimization incurs severe limitations:

\begin{enumerate}
    \item \emph{Memory-intensive and sparse optimization}: Gradient computation over $T$ reverse steps demands storing the entire trajectory ${\mathbf{x}^\theta_t}_{t=1}^T$ and corresponding Jacobians, leading to high memory consumption and inefficient, sparse optimization compared to the sampling process.
    \item \emph{Vanishing gradient due to coarse-grained optimization}: Eq.~\eqref{eq:final_gradient} indicates that the optimization of $t$-th noisy step relies on the gradient $\frac{\partial \pi_\theta(\mathbf{x}^\theta_t, t, c)}{\partial \theta}$ with a coefficient $\prod_{s=1}^{t-1} \frac{\partial \pi_\theta(\mathbf{x}^\theta_s, s, c)}{\partial \mathbf{x}^\theta_s}$. However, during optimization, the term $\frac{\partial \pi_\theta(\mathbf{x}^\theta_t, t, c)}{\partial \mathbf{x}^\theta_t}$ tends to converge to 0 (see the blue line in Fig.~\ref{fig:norm}), causing the coefficient $\prod_{s=1}^{t-1} \frac{\partial \pi_\theta(\mathbf{x}^\theta_s, s, c)}{\partial \mathbf{x}^\theta_s}$ to also approach 0 (see the orange line in Fig.~\ref{fig:norm}). Consequently, the optimization process tends to neglect the contribution of $\frac{\partial \pi_\theta(\mathbf{x}^\theta_t, t, c)}{\partial \theta}$. More importantly, the ignored optimization at these early noise steps may be more crucial than at later steps \cite{xie2025dymo, zhou2025golden,tan2026consistentrftreducingvisualhallucinations}.
\end{enumerate}

\noindent \textbf{Motivation.} To address the aforementioned limitations, we argue that the key issue lies in Corollary \ref{thm:t1}: {the computation of ${\partial \mathbf{x}_t^\theta}/{\partial \theta}$ recursively depends on ${\partial \mathbf{x}_{t+1}^\theta}/{\partial \theta}$, making the computation of ${\partial \mathbf{x}_0^\theta}/{\partial \theta}$ reliant on the entire T-step reverse process}. This dependency necessitates storing a large computation graph, resulting in substantial memory consumption and delayed optimization. To overcome this, an intuitive insight is introduced: \emph{{optimizing the gradient step-by-step during the reverse process.}} As illustrated in Fig.~\ref{fig:framework}, step-by-step optimization offers several advantages: \emph{(1) Lower memory consumption and dense optimization}: each update only requires the computation graph of the current step, allowing gradients to be computed and applied immediately instead of waiting until the end of the $T$-step reverse process. \emph{(2) Fine-grained optimization}: each step is optimized independently, so that the update of the $t$-th step does not depend on the vanishing product of coefficients $\prod_{s=1}^{t-1} \frac{\partial \pi_\theta(\mathbf{x}^\theta_s, s, c)}{\partial \mathbf{x}^\theta_s}$.

\begin{figure*}[t!]
    \centering
    \includegraphics[width=\textwidth]{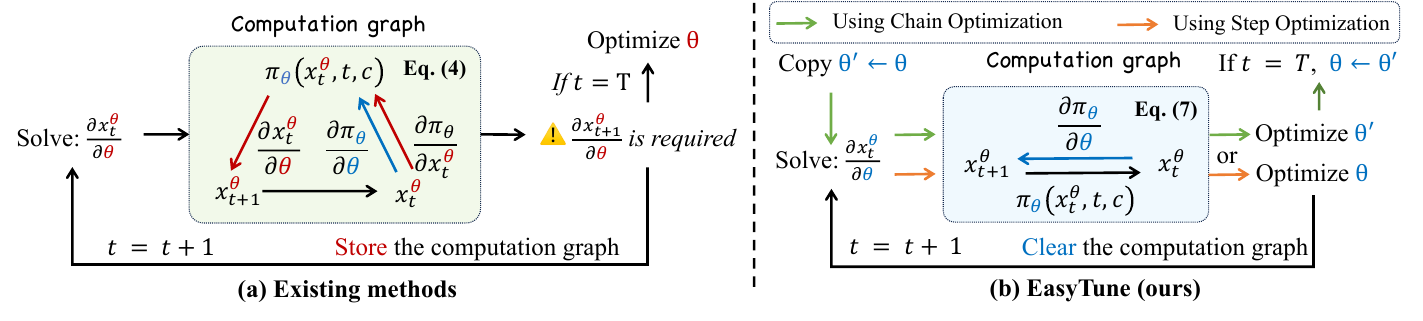}

    \caption{{Core insight of EasyTune.} By replacing the recursive gradient in Eq.\eqref{eq:recursive_gradient} with step-level ones in Eq.\eqref{eq:ourgradient}, \emph{EasyTune} removes recursive dependencies, enabling (1) step-wise graph storage, (2) efficiency, and (3) fine-grained optimization.}

    \label{fig:insight}
\end{figure*}
\subsection{Efficient Fine-Tuning for Motion Diffusion}
\label{sec:EasyTune}

Inspired by the above discussion, we propose EasyTune, a step-aware differentiable reward-based fine-tuning framework for diffusion models. Step-wise optimization requires a noise-aware reward model that can evaluate noised motion at each step. Unlike image generation, where rewards are predominantly output-level~\cite{xu2023imagereward} due to the complex semantics of noised states, motion representations exhibit simpler and more interpretable semantics, making step-aware reward viable (See Fig.~\ref{fig:motion_image_cmp}). This motivates us to leverage the unified step-aware reward model via noise-aware backbone from ReAlign~\cite{weng2025}.

\noindent\textbf{Noise-Aware Reward.} Step-wise optimization requires a reward model capable of evaluating noised motions. We consider two approaches: (1) following ReAlign~\cite{weng2025}, we train the reward model on noised motions with varying noise levels, enabling direct evaluation of $\mathcal{R}_\phi(\mathbf{x}_t, t, c)$ at any timestep $t$; (2) for ODE-based models~\cite{lu2022dpm, motionlcm-v2}, deterministic sampling yields a coarse clean motion $\hat{\mathbf{x}}_0 = \pi'_\theta(\mathbf{x}_t, t, c)$ via one-step prediction, whose reward serves as a proxy. The noise-aware reward is defined as:
\begin{equation}\small
\begin{aligned}
\mathcal{R}_\phi (\mathbf{x}_t, t, c)  &=
\begin{cases}
\mathcal{R}_\phi (\hat{\mathbf{x}}_0, 0, c) , & \text{ODE-based prediction}, \\
\mathcal{R}_\phi ({\mathbf{x}}_t, t, c) , & \text{SDE/ODE-based perception}.
\end{cases}
\end{aligned}
\label{eq:reward}
\end{equation}
We adopt the unified noise-aware perception as default settings, along with the ODE-based proxy in Tab.~\ref{tab:PlugAndPlay_H3D}.

\noindent\textbf{Unified Reward Aggregation.} At each timestep, the unified reward aggregates four terms. Given a noised motion $\mathbf{x}_t$ at step $t$ and text condition $c$, the total reward is a weighted sum of semantic alignment $\mathcal{R}^{s}$, human preference $\mathcal{R}^{p}$, authenticity $\mathcal{R}^{a}$, and motion-to-motion similarity $\mathcal{R}^{m}_\phi$:
\begin{equation}\small
\begin{split}
\mathcal{R}_\phi(\mathbf{x}_t, t, c) = \,& w_1\,\mathcal{R}^{s}_\phi(\mathbf{x}_t, t, c) + w_2\,\mathcal{R}^{p}_\phi(\mathbf{x}_t, t) \\
&+ w_3\,\mathcal{R}^{a}_\phi(\mathbf{x}_t, t) + w_4\,\mathcal{R}^{m}_\phi(\mathbf{x}_t, t),
\end{split}
\label{eq:unified_reward}
\end{equation}
where $\mathcal{R}^{s}_\phi$, $\mathcal{R}^{p}_\phi$, and $\mathcal{R}^{a}_\phi$ are the semantic, preference, and authenticity reward heads from MotionReward in Sec.~\ref{sec:unified_reward}. {\color{blue}The motion-to-motion similarity reward $\mathcal{R}^{m}_\phi$ is}
\begin{equation}\small
{\color{blue}
\mathcal{R}^{m}_\phi(\mathbf{x}_t, t)=\cos\bigl(\mathcal{E}_\mathrm{M}(\mathbf{x}_t, t),\,\mathcal{E}_\mathrm{M}(\mathbf{x}^{\mathrm{ref}}_t, t)\bigr),
}
\label{eq:sim_reward}
\end{equation}
{\color{blue}where $\mathcal{E}_\mathrm{M}$ is the frozen motion encoder and $\mathbf{x}^{\mathrm{ref}}_t$ is the corresponding ground-truth motion from the training set used to pre-train the generator, both encoded at the same timestep $t$. This term limits distribution shift during fine-tuning by keeping the motion close to the pre-training data distribution.} We set $w_1{=}1$, $w_2{=}0.002$, $w_3{=}0.002$, and $w_4{=}1$ throughout. {\color{blue}Both the text-motion reward $\mathcal{R}^{s}$ and the motion-to-motion similarity reward $\mathcal{R}^{m}$ may be treated as primary, while preference and authenticity may be treated as secondary. Such unequal weights are standard when combining primary objectives with auxiliary rewards~\cite{clark2024directly,peng2018deepmimic,peng2022ase}.}

\noindent \textbf{EasyTune.} Based on the noise-aware reward, we introduce EasyTune for fine-tuning motion diffusion models. As discussed in Sec.~\ref{sec:motivation}, limitations of existing methods~\cite{clark2024directly, wu2025drtune} stem from the recursive gradients computation. The key idea of EasyTune is to maximize the reward value at each step, allowing the parameter to be optimized at each step without storing the full trajectory, as shown in Fig.~\ref{fig:insight}. Specifically, the training objective function is defined as:
\begin{equation}\small
    \begin{aligned}
        \mathcal{L}_{\mathrm{EasyTune}}(\theta) = -\mathbb{E}_{c\sim \mathbb{D}_\mathrm{T}, \mathbf{x}^\theta_t \sim \pi_\theta(\cdot|c), t\sim \mathcal{U}(0, T)} \left[ \mathcal{R}_{\phi} (\mathbf{x}^\theta_t, t, c) \right],
    \end{aligned}
    \label{eq:loss_ours}
\end{equation}
where \(\mathcal{R}_{\phi} (\mathbf{{x}}^\theta_t, t, c)\) is the reward value of the \emph{stop gradient} noised motion \(\mathbf{{x}}^\theta_t\) at time step \(t\), and \(\mathcal{U}(0, T)\) is a uniform distribution over the time steps. Here, the {stop gradient} noised motion \(\mathbf{{x}}^\theta_t\) and its gradient w.r.t. the diffusion parameter $\theta$ are represented as:
\begin{equation}\small
    \begin{split}
        \mathbf{x}^\theta_{t-1} &= \pi_\theta \big({\mathrm{sg}}(\mathbf{x}^\theta_t), t, c \big) \\
        &:= \frac{1}{\sqrt{\alpha_t}} \left( {\mathrm{sg}}(\mathbf{x}^\theta_t) - \frac{\beta_t}{\sqrt{1 - \bar{\alpha}_t}} \epsilon_\theta\big({\mathrm{sg}}(\mathbf{x}^\theta_t), t, c\big) \right),
    \end{split}
    \label{eq:ourgradient}
\end{equation}
where \({\mathrm{sg}}(\cdot)\) denotes the stop gradient operations. Eq. \eqref{eq:loss_ours} and Eq. \eqref{eq:ourgradient} indicate that EasyTune aims to optimize the diffusion model by maximizing the reward value of the noised motion \(\mathbf{x}^\theta_t\) at each step \(t\). Here, we discuss the effectiveness of EasyTune from theoretical and empirical perspectives.

\begin{corollary}\label{the:t2}
    Given the reverse process in Eq. \eqref{eq:ourgradient}, the gradient w.r.t. diffusion model $\theta$ is denoted as:
    \begin{equation}\small
        \begin{aligned}
            {\frac{\partial \mathbf{x}^{{\theta}}_{t-1}}{\partial {\theta}}} = \frac{\partial \pi_{{\theta}}\big({\mathrm{sg}}(\mathbf{x}^\theta_t), t, c\big)}{\partial {\theta}}.
        \end{aligned}
        \label{eq:ourgradient_the}
    \end{equation}
\end{corollary}
Corollary~\ref{the:t2} shows that EasyTune overcomes the recursive gradient issue, enabling efficient, fine-grained updates with substantially reduced memory. As Fig.~\ref{fig:memory} illustrates, while prior methods incur $\mathcal{O}(T)$ memory by storing the multi-steps trajectory, EasyTune maintains a constant $\mathcal{O}(1)$ memory. Thus, we optimize the loss function $\mathcal{L}_{\mathrm{EasyTune}}(\theta)$ as follows:
    \begin{equation}\small
        \begin{split}
        \frac{\partial \mathcal{L}_{\mathrm{EasyTune}} (\theta)}{\partial \theta} & =  -\mathbb{E}_{c\sim \mathbb{D}_\mathrm{T}, \mathbf{x}^\theta_t \sim \pi_\theta(\cdot|c), t\sim \mathcal{U}(0, T)}  \\ & \quad  \frac{\partial \mathcal{R}_\phi(\mathbf{x}^\theta_{t}, t, c)}{\partial \mathbf{x}^\theta_{t}}
        \cdot \frac{\partial \pi_{{\theta}}\big({\mathrm{sg}}(\mathbf{x}^\theta_{t+1}), t+1, c\big)}{\partial {\theta}}.
        \end{split}
        \label{eq:loss_ours_expanded}
    \end{equation}

\noindent \textbf{Empirical Memory Analysis.} As shown in Fig.~\ref{fig:memory}, existing methods incur memory usage that grows linearly with the number of denoising steps ($\mathcal{O}(T)$), while EasyTune maintains a constant memory footprint ($\mathcal{O}(1)$).

\noindent \textbf{Curriculum Timestep Scheduling.}
As discussed above, the vanishing coefficient $\prod_{s=1}^{t-1} \frac{\partial \pi_\theta}{\partial \mathbf{x}^\theta_s}$ renders prior methods ineffective at high-noise steps. Although EasyTune removes this recursive dependency, uniformly sampling timesteps remains suboptimal, since early denoising steps mainly determine global structure while late steps refine local semantics~\cite{tan2026consistentrftreducingvisualhallucinations, xie2025dymo, zhou2025golden}. We therefore adopt a coarse-to-fine curriculum that gradually shifts optimization from high-noise to low-noise steps. Specifically, gradients are computed only within a sliding window of $k$ consecutive steps, $\mathcal{W}(p)=\{s(p),\ldots,s(p){+}k{-}1\}$, where the start index $s$ is scheduled by the training progress $p\in[0,1]$:
\begin{equation}\small
s(p) =
\begin{cases}
(T {-} k) - \mathrm{round}\!\left(\dfrac{p}{\rho} \cdot (T {-} k)\right), & p \leq \rho, \\[4pt]
0, & p > \rho,
\end{cases}
\label{eq:curriculum}
\end{equation}
where $\rho$ is the sweep ratio. The objective in Eq.~\eqref{eq:loss_ours} becomes:
\begin{equation}\small
\mathcal{L}_{\mathrm{EasyTune}}(\theta) = -\mathbb{E}_{c,\, \mathbf{x}^\theta_t \sim \pi_\theta(\cdot|c),\, t \in \mathcal{W}(p)} \!\left[ \mathcal{R}_{\phi}(\mathbf{x}^\theta_t, t, c) \right].
\label{eq:loss_curriculum}
\end{equation}
For $p \le \rho$, the window $\mathcal{W}(p)$ linearly sweeps from $\{T{-}k,\ldots,T{-}1\}$ to $\{0,\ldots,k{-}1\}$, progressively shifting focus from high-noise global structure to low-noise local details. Once $p > \rho$, the window is fixed at $\{0,\ldots,k{-}1\}$, concentrating the remaining budget on the final denoising steps where semantic details are most critical. We set $\rho{=}0.4$, $k{=}10$, and $T{=}50$ by default, so that the first 40\% of training addresses high-noise steps for global structure and the remaining 60\% focuses on low-noise steps for fine-grained refinement~\cite{tan2026consistentrftreducingvisualhallucinations}.

\section{Experiment}
\label{sec:Exp}
\subsection{Experimental Setup}
\label{sec:Exp_Setup}
\begin{table*}[b]
    \centering
    \renewcommand{\arraystretch}{1.1}
    \setlength{\tabcolsep}{6pt}
    \caption{\textbf{Text-motion retrieval results {\color{blue}on HumanML3D}.} {\color{blue}Retrieval is evaluated in both directions with a gallery of $32$, and ``Avg.'' averages the ten R@$k$ entries. ``Noise'' marks noise-aware encoders, and ``Rep.'' the representations a single model supports (K: kinematic, J: joint, R: rotation). Bold: best; underline: second best.} $^*$FlowMDM~\cite{barquero2024seamless} is reproduced and adapted to HumanML3D. {\color{blue}$^\dagger$ReAlign~\cite{weng2025} and TMR~\cite{petrovich2023tmr} are retrained by us on the joint and rotation features, one model per representation.}}
    \begin{adjustbox}{max width=\textwidth}
    \begin{tabular}{l|c|ccc|ccccc|ccccc|c}
        \toprule
        \multirow{2}{*}{Methods} & \multirow{2}{*}{Noise} & \multicolumn{3}{c|}{Rep.} & \multicolumn{5}{c|}{Text-Motion Retrieval$\uparrow$} & \multicolumn{5}{c|}{Motion-Text Retrieval$\uparrow$} & \multirow{2}{*}{Avg.$\uparrow$} \\
        \cmidrule(lr){3-5}
        \cmidrule(lr){6-10} \cmidrule(lr){11-15}
        & & K & J & R & R@1 & R@2 & R@3 & R@5 & R@10 & R@1 & R@2 & R@3 & R@5 & R@10 & \\
        \midrule
        \multicolumn{16}{c}{\textit{Evaluation with Kinematic Representation}} \\
        \midrule
        TEMOS \cite{petrovich2022temos} & \ding{55} & \ding{51} & \ding{55} & \ding{55} & $40.49$ & $53.52$ & $61.14$ & $70.96$ & $84.15$ & $39.96$ & $53.49$ & $61.79$ & $72.40$ & $85.89$ & $62.38$\\
        T2M \cite{guo2022generating} & \ding{55} & \ding{51} & \ding{55} & \ding{55} & $52.48$ & $71.05$ & $80.65$ & $89.66$ & \underline{$96.58$} & $52.00$ & $71.21$ & $81.11$ & $89.87$ & {$96.78$} & $78.14$\\
        TMR \cite{petrovich2023tmr} & \ding{55} & \ding{51} & \ding{55} & \ding{55} & $67.16$ & $81.32$ & $86.81$ & $91.43$ & $95.36$ & $67.97$ & $81.20$ & $86.35$ & $91.70$ & $95.27$ & $84.46$\\
        LaMP \cite{li2025lamp} & \ding{55} & \ding{51} & \ding{55} & \ding{55} & $67.18$ & $81.90$ & {$87.04$} & \underline{$92.00$} & $95.73$ & {$68.02$} & {$82.10$} & {$87.50$} & {$92.20$} & \underline{$96.90$} & $85.06$\\
        {ReAlign} \cite{weng2025} & {\ding{51}} & \ding{51} & \ding{55} & \ding{55} & \underline{$67.59$} & \underline{$82.24$} & \underline{$87.44$} & $91.97$ & $96.28$ & \underline{$68.94$} & \underline{$82.86$} & \underline{$87.95$} & \underline{$92.44$} & $96.28$ & \underline{$85.40$}\\
        \rowcolor{gray!20} {MotionReward(Ours)} & {\ding{51}} & \ding{51} & \ding{51} & \ding{51} & \textbf{67.74} & \textbf{83.35} & \textbf{89.79} & \textbf{94.86} & \textbf{98.37} & \textbf{67.92} & \textbf{83.50} & \textbf{90.03} & \textbf{95.16} & \textbf{98.44} & \textbf{86.92}\\
        \midrule
        \multicolumn{16}{c}{\textit{Evaluation with Joint Representation}} \\
        \midrule
        ACMDM-Eval.~\cite{meng2025absolute} & \ding{55} & \ding{55} & \ding{51} & \ding{55} & $49.78$ & $69.61$ & $79.87$ & $89.05$ & $96.06$ & $50.28$ & $70.30$ & $80.26$ & $89.20$ & $96.57$ & $77.10$\\
        {\color{blue}TMR$^\dagger$ \cite{petrovich2023tmr}} & {\color{blue}\ding{55}} & {\color{blue}\ding{55}} & {\color{blue}\ding{51}} & {\color{blue}\ding{55}} & {\color{blue}$62.50$} & {\color{blue}$78.18$} & {\color{blue}$84.56$} & {\color{blue}$90.03$} & {\color{blue}$95.21$} & {\color{blue}$62.69$} & {\color{blue}$78.16$} & {\color{blue}$84.29$} & {\color{blue}$90.11$} & {\color{blue}$95.07$} & {\color{blue}$82.08$}\\
        {\color{blue}ReAlign$^\dagger$ \cite{weng2025}} & {\color{blue}\ding{51}} & {\color{blue}\ding{55}} & {\color{blue}\ding{51}} & {\color{blue}\ding{55}} & {\color{blue}\underline{$63.08$}} & {\color{blue}\underline{$79.24$}} & {\color{blue}\underline{$85.31$}} & {\color{blue}\underline{$90.65$}} & {\color{blue}$96.02$} & {\color{blue}\underline{$63.52$}} & {\color{blue}\underline{$79.58$}} & {\color{blue}\underline{$85.64$}} & {\color{blue}\underline{$91.02$}} & {\color{blue}$96.21$} & {\color{blue}\underline{$83.03$}}\\
        \rowcolor{gray!20} {MotionReward(Ours)} & {\ding{51}} & \ding{51} & \ding{51} & \ding{51} & \textbf{65.86} & \textbf{81.42} & \textbf{88.08} & \textbf{94.18} & \textbf{98.07} & \textbf{65.47} & \textbf{81.55} & \textbf{88.40} & \textbf{94.05} & \textbf{98.07} & \textbf{85.52}\\
        \midrule
        \multicolumn{16}{c}{\textit{Evaluation with Rotation Representation}} \\
        \midrule
        FlowMDM-Eval.$^*$~\cite{barquero2024seamless} & \ding{55} & \ding{55} & \ding{55} & \ding{51} & $47.39$ & $67.49$ & $76.45$ & $88.44$ & $96.39$ & $48.34$ & $68.73$ & $77.16$ & $87.72$ & $95.65$ & $75.38$\\
        {\color{blue}TMR$^\dagger$ \cite{petrovich2023tmr}} & {\color{blue}\ding{55}} & {\color{blue}\ding{55}} & {\color{blue}\ding{55}} & {\color{blue}\ding{51}} & {\color{blue}$60.63$} & {\color{blue}$76.27$} & {\color{blue}$83.28$} & {\color{blue}$89.05$} & {\color{blue}$94.14$} & {\color{blue}$61.30$} & {\color{blue}$77.07$} & {\color{blue}$83.51$} & {\color{blue}$89.76$} & {\color{blue}$94.65$} & {\color{blue}$80.97$}\\
        {\color{blue}ReAlign$^\dagger$ \cite{weng2025}} & {\color{blue}\ding{51}} & {\color{blue}\ding{55}} & {\color{blue}\ding{55}} & {\color{blue}\ding{51}} & {\color{blue}\underline{$61.28$}} & {\color{blue}\underline{$77.42$}} & {\color{blue}\underline{$84.16$}} & {\color{blue}\underline{$89.75$}} & {\color{blue}$95.08$} & {\color{blue}\underline{$62.35$}} & {\color{blue}\underline{$78.60$}} & {\color{blue}\underline{$84.95$}} & {\color{blue}\underline{$90.62$}} & {\color{blue}\underline{$95.87$}} & {\color{blue}\underline{$82.01$}}\\
        \rowcolor{gray!20} {MotionReward(Ours)} & {\ding{51}} & \ding{51} & \ding{51} & \ding{51} & \textbf{64.19} & \textbf{82.13} & \textbf{89.27} & \textbf{93.75} & \textbf{97.87} & \textbf{63.42} & \textbf{80.85} & \textbf{88.42} & \textbf{94.12} & \textbf{97.72} & \textbf{85.17}\\
        \bottomrule
    \end{tabular}
    \end{adjustbox}
    \label{tab:retrieval}
\end{table*}

\noindent \textbf{Datasets.} Our experiments span three motion representations, kinematic, joint-based, and rotation, and three reward dimensions, motion semantics, human preference, and authenticity. {\color{blue}Specifically, the representations are derived directly from the raw SMPL~\cite{loper2015smpl} source motions via forward kinematics.} We train the semantic reward on the kinematic-based HumanML3D~\cite{guo2022generating} and the preference reward on MotionPercept~\cite{motioncritic2025}, extending both to kinematic and joint-based versions. For authenticity, we use MLD~\cite{chen2023executing} and ACMDM~\cite{meng2025absolute} to generate fake motions and pair them with ground-truth real motions, forming a multi-representation deepfake detection dataset.

\noindent \textbf{Evaluation.} We evaluate MotionReward on semantic alignment, human preference, and authenticity, and further assess the fine-tuned generators. For semantic reward, we report text-to-motion and motion-to-text retrieval {\color{blue}on the HumanML3D test split}{\color{blue}. Following the standard protocol of Guo et al.~\cite{guo2022generating} and TMR~\cite{petrovich2023tmr}, the shuffled test split is partitioned into complete batches of 32, and the Top-$k$ recall is averaged over these batches. Here, 32 represents the gallery size, which is the number of candidates per query in each evaluation batch. We also report retrieval under the more challenging all-set protocol in App.~\sref{sec:allset}.} For human preference reward, we use the official MotionCritic test set {\color{blue}and report accuracy under the original pairwise protocol}{\color{blue}. Because the benchmark exclusively comprises positive-ordered pairs, with the preferred motion always positioned first, it inherently admits only pairwise ranking accuracy. To additionally evaluate precision, recall, and F1 score, which necessitate balanced label distributions, we randomly invert the ordering of 50\% of the pairs to construct a balanced binary classification task over candidate positions}. For authenticity reward, we classify real versus generated motions and report the same four metrics. For generative models, following HumanML3D~\cite{guo2022generating}, we report R-Precision, Fr\'echet Inception Distance (FID), Multi Modal Distance (MMDist), and Diversity, along with peak memory usage, critic score, and authenticity score. {\color{blue}We additionally report MultiModality, evaluated via the standard text-to-motion evaluation network of Guo et al.~\cite{guo2022generating} for kinematic models and the FlowMDM 135-d evaluator~\cite{barquero2024seamless} for rotation-based models, citing published baseline figures when available under the identical evaluator.}

\noindent \textbf{Implementation.} Training consists of two stages: MotionReward training and generative model fine-tuning. For MotionReward, we train the shared backbone end-to-end for text-motion alignment with a learning rate of $1\times10^{-4}$, batch size 128, and 50 epochs across kinematic, joint-based, and rotation representations. Loss weights follow the settings in Sec.\ref{sec:unified_reward}. {\color{blue}Each projector $\phi^{o}$ is a two-layer MLP with GELU and dropout $0.1$, mapping input dimension $d_o\in\{263,66,135\}$ through a hidden width of $256$ to the shared width $d{=}256$.} We then freeze the backbone and train task-specific LoRA adapters with rank $r{=}16$ and $\alpha{=}2r$; the critic head trains for up to 2{,}000 epochs and the detection head for 50 epochs.

{\color{blue}We evaluate EasyTune~\cite{tan2026easytune} across pre-trained diffusion models, including kinematic-based models MLD~\cite{chen2023executing}, MLD++~\cite{motionlcm-v2}, joint-based ACMDM~\cite{meng2025absolute}, and rotation-based HY-Motion~\cite{hymotion2025}. Results on HumanML3D~\cite{guo2022generating} and KIT-ML~\cite{plappert2016kit} are reported in Tab.~\ref{tab:sota_humanmld3d} and App.~\sref{sec:kit}, respectively.} We fine-tune with a learning rate of $1\times10^{-5}$, cosine scheduling, 100 warmup steps, and gradient clipping at max norm 1.0. ACMDM uses batch size 32 for 10 epochs and {\color{blue}HY-Motion} uses batch size 8 for 4 epochs. {\color{blue}For rotation-based HY-Motion~\cite{hymotion2025}, we evaluate it in Tab.~\ref{tab:sota_humanmld3d} using a $135$-d rotation evaluator following the protocol of Guo et al.~\cite{guo2022generating}. Additional evaluations under a retrieval evaluator following the ReAlign architecture~\cite{weng2025} and the released SSAE protocol~\cite{hymotion2025} are provided in App.~\sref{sec:hy_eval} and Tabs.~\sref{tab:hy_realign} and~\sref{tab:hy_ssae}; for SSAE, we use Qwen3-VL-30B-A3B-Instruct as the VLM judge.} We compare other reinforcement fine-tune methods to evaluate the effectiveness of ours, including ReFL~\cite{xu2023imagereward}, DRaFT~\cite{clark2024directly}, DRTune~\cite{wu2025drtune}, and AlignProp~\cite{prabhudesai2023aligning}.

\subsection{Effectiveness on Unified Motion Reward Model}
\label{sec:reward_eval}

\begin{table}[t]
    \centering
    \renewcommand{\arraystretch}{1.1}
    \setlength{\tabcolsep}{6pt}
    \caption{\textbf{Evaluation on human preference prediction {\color{blue}on MotionCritic~\cite{motioncritic2025}}.} {\color{blue}Each test item is a motion pair with a human preference label.} $^\dagger$MotionCritic was originally designed for 24-joint representations; we reproduce it on our 22-joint setting.}
    \begin{adjustbox}{max width=\columnwidth}
    \begin{tabular}{l|ccc|cccc}
        \toprule
        \multirow{2}{*}{Methods} & \multicolumn{3}{c|}{Rep.} & \multicolumn{4}{c}{Human Preference} \\
        \cmidrule(lr){2-4} \cmidrule(lr){5-8}
        & K & J & R & Acc$\uparrow$ & Precision$\uparrow$ & Recall$\uparrow$ & F1$\uparrow$ \\
        \midrule
        \multicolumn{8}{c}{\textit{Kinematic Representation}} \\
        \midrule
        MotionCritic$^\dagger$~\cite{motioncritic2025} & \ding{51} & \ding{55} & \ding{55} & 85.08 & 84.76 & 85.61 & 85.18 \\
        \rowcolor{gray!20} {MotionReward (Ours)} & \ding{51} & \ding{51} & \ding{51} & 86.09 & 86.20 & 86.02 & 86.11 \\
        \midrule
        \multicolumn{8}{c}{\textit{Joint Representation}} \\
        \midrule
        MotionCritic$^\dagger$~\cite{motioncritic2025} & \ding{55} & \ding{51} & \ding{55} & 84.18 & 85.13 & 83.55 & 84.33 \\
        \rowcolor{gray!20} {MotionReward (Ours)} & \ding{51} & \ding{51} & \ding{51} & 85.52 & 86.29 & 85.10 & 85.69 \\
        \midrule
        \multicolumn{8}{c}{\textit{Rotation Representation}} \\
        \midrule
        MotionCritic$^\dagger$~\cite{motioncritic2025} & \ding{55} & \ding{55} & \ding{51} & 84.34 & 84.19 & 84.36 & 84.28 \\
        \rowcolor{gray!20} {MotionReward (Ours)} & \ding{51} & \ding{51} & \ding{51} & 85.61 & 86.01 & 84.88 & 85.44 \\
        \bottomrule
    \end{tabular}
    \end{adjustbox}
    \label{tab:pref}
\end{table}

We evaluate MotionReward on HumanML3D~\cite{guo2022generating}, the MotionPercept benchmark~\cite{motioncritic2025}, and our multi-representation deepfake detection dataset, under kinematic features, joint coordinates, and rotation representations, following the experimental setup in Sec.~\ref{sec:Exp_Setup}. Tab.~\ref{tab:retrieval} reports semantic alignment results. MotionReward achieves the best text-motion retrieval accuracy across all three representations, {\color{blue}improving Avg.\ over ReAlign~\cite{weng2025} from $85.40$ to $86.92$ on kinematic features, from $83.03$ to $85.52$ on joint features, and from $82.01$ to $85.17$ on rotation features}. {\color{blue}Considering that prior studies~\cite{meng2024rethinking} have identified redundancy in standard kinematic features, we provide further discussion and evaluations on non-redundant representations in App.~\sref{sec:meng_eval}.} Tab.~\ref{tab:pref} reports human preference prediction, where MotionReward consistently improves accuracy and F1 over MotionCritic~\cite{motioncritic2025} across representations, while MotionCritic requires separate models for each representation. {\color{blue}Tab.~\ref{tab:faithful}} shows strong authenticity modeling via deepfake detection, with accuracy above 91\%, F1 above 95\%, and recall over 99\%. Overall, these results indicate that a single MotionReward model can provide stronger rewards across the three tasks, consistent with our unified representation learning and task-specific LoRA adaptation.

\begin{table}[t]
    \centering
    \renewcommand{\arraystretch}{1.1}
    \setlength{\tabcolsep}{6pt}
    \caption{\textbf{Evaluation on motion authenticity} {\color{blue}for real versus generated motions under kinematic and joint representations.}}
    \begin{tabular}{lcccc}
        \toprule
        Representation & Accuracy$\uparrow$ & Precision$\uparrow$ & Recall$\uparrow$ & F1-Score$\uparrow$ \\
        \midrule
        Kinematic & 91.78 & 91.43 & 99.67 & 95.37 \\
        Joint & 93.92 & 93.58 & 99.69 & 96.54 \\
        \bottomrule
    \end{tabular}
    \label{tab:faithful}
\end{table}

\subsection{Effectiveness on Reinforcement Fine-Tuning}
\label{sec:rft_eval}

\noindent \textbf{Memory Efficiency.} As illustrated in Fig.~\ref{fig:memoryexp}, we compare the memory trajectory of EasyTune~\cite{tan2026easytune} and existing fine-tuning methods throughout the optimization process. Our method performs multiple optimization steps within a single denoising trajectory, leading to high average memory utilization. However, the peak memory consumption of EasyTune is significantly lower than that of existing methods, mainly due to the $\mathcal{O}(1)$ memory growth of the denoising process. A lower peak footprint reduces hardware requirements, while higher utilization indicates more efficient use of GPU resources.

\begin{figure}[tbp]
    \centering
    \includegraphics[width=\columnwidth]{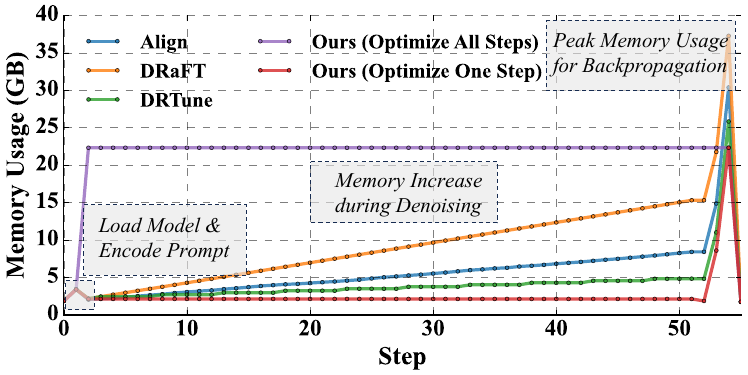}
    \caption{{Memory trajectory during optimization.} {\color{blue}The x-axis is the denoising step and the y-axis is GPU memory.}}
    \label{fig:memoryexp}
\end{figure}

\noindent \textbf{Comparison with SoTA Fine-Tuning Methods.} To assess the effectiveness and efficiency of {EasyTune}~\cite{tan2026easytune}, we compare it with recent state-of-the-art fine-tuning methods, including DRaFT, AlignProp, and DRTune, as shown in Tab.~\ref{tab:addlabel}. {EasyTune} achieves the best overall performance across key metrics. Specifically, {\color{blue}step optimization attains the lowest FID of 0.132, improving by 70.7\%, and the closest Diversity to real motion at 9.465; chain optimization attains the best MM Dist of 2.614 with a 14.4\% gain. Peak GPU memory is 22.10\,GB, tied with ReFL for the lowest among all compared methods.} We attribute these gains to two core designs: optimizing rewards at each denoising step for finer supervision, and discarding redundant computation graphs to reduce memory usage. {\color{blue}We additionally report MultiModality in Tab.~\ref{tab:sota_humanmld3d}--\ref{tab:addlabel}. While set-level Diversity remains close to real motions, fine-tuning for semantic alignment naturally sharpens per-prompt consistency, leading to moderate decreases in MultiModality as observed across post-trained generative models~\cite{kirk2024understanding,tan2026perceptual}.}

\definecolor{cDRaFT}{HTML}{1F77B4}
\definecolor{cAlignProp}{HTML}{FF7F0E}
\definecolor{cDRTune}{HTML}{2CA02C}
\definecolor{cReFL}{HTML}{9467BD}
\definecolor{cOurs}{HTML}{D62728}
\definecolor{cPref}{HTML}{1F77B4}
\definecolor{cAuth}{HTML}{FF7F0E}
\begin{figure}[tbp]
    \centering
    \begin{minipage}[t]{0.48\linewidth}
        \centering
        \resizebox{\linewidth}{!}{%
        \begin{tikzpicture}
        \begin{axis}[
            width=6.5cm,
            height=4.6cm,
            ylabel={Time (s)},
            ylabel style={font=\fontsize{9}{10}\selectfont\fontfamily{ptm}\selectfont},
            xlabel={Reward Score},
            xlabel style={font=\fontsize{9}{10}\selectfont\fontfamily{ptm}\selectfont},
            xmin=0.675, xmax=0.875,
            xtick={0.70,0.75,0.80,0.85},
            xticklabel style={font=\fontsize{8}{9}\selectfont\fontfamily{ptm}\selectfont},
            ymin=0, ymax=2900,
            ytick={0,500,1000,1500,2000,2500},
            yticklabel style={font=\fontsize{8}{9}\selectfont\fontfamily{ptm}\selectfont},
            legend style={
                at={(0.5,1.03)}, anchor=south,
                legend columns=3,
                font=\fontsize{7.5}{8.5}\selectfont\fontfamily{ptm}\selectfont,
                draw=none,
                fill=none,
                column sep=3pt,
                /tikz/every even column/.append style={column sep=5pt},
            },
            ymajorgrids=true,
            xmajorgrids=false,
            grid style={gray!50, dashed},
            axis x line*=bottom,
            axis y line*=left,
            axis line style={black!70},
            tick style={black!70},
            every axis plot/.append style={thick},
            clip=false,
        ]
        \draw[->, >=stealth, thick, black!60] (axis cs:0.82, 2750) -- (axis cs:0.86, 2750);
        \node[font=\fontsize{7}{8}\selectfont\fontfamily{ptm}\selectfont, anchor=east, black!60] at (axis cs:0.82, 2750) {\textit{better}};
        \addplot[color=cDRaFT, mark=triangle*, mark size=2pt, thick] coordinates {
            (0.70, 466.27) (0.75, 2616.54)
        };
        \node[font=\fontsize{7}{8}\selectfont, anchor=west, cDRaFT] at (axis cs:0.753, 2617) {2617};
        \node[font=\fontsize{6}{7}\selectfont\fontfamily{ptm}\selectfont, anchor=south, text=white, fill=cDRaFT, draw=cDRaFT, rounded corners=1pt, inner sep=1.5pt] at (axis cs:0.75, 2740) {\textit{convergence}};
        \addplot[color=cAlignProp, mark=square*, mark size=1.8pt, thick] coordinates {
            (0.70, 271.99) (0.75, 971.55)
        };
        \node[font=\fontsize{7}{8}\selectfont, anchor=west, cAlignProp] at (axis cs:0.753, 972) {972};
        \node[font=\fontsize{6}{7}\selectfont\fontfamily{ptm}\selectfont, anchor=south, text=white, fill=cAlignProp, draw=cAlignProp, rounded corners=1pt, inner sep=1.5pt] at (axis cs:0.75, 1095) {\textit{convergence}};
        \addplot[color=cDRTune, mark=diamond*, mark size=2pt, thick] coordinates {
            (0.70, 554.77) (0.75, 2009.59)
        };
        \node[font=\fontsize{7}{8}\selectfont, anchor=west, cDRTune] at (axis cs:0.753, 2010) {2010};
        \node[font=\fontsize{6}{7}\selectfont\fontfamily{ptm}\selectfont, anchor=south, text=white, fill=cDRTune, draw=cDRTune, rounded corners=1pt, inner sep=1.5pt] at (axis cs:0.75, 2130) {\textit{convergence}};
        \addplot[color=cReFL, mark=pentagon*, mark size=2pt, thick] coordinates {
            (0.70, 820.29)
        };
        \node[font=\fontsize{6}{7}\selectfont\fontfamily{ptm}\selectfont, anchor=south, text=white, fill=cReFL, draw=cReFL, rounded corners=1pt, inner sep=1.5pt] at (axis cs:0.70, 920) {\textit{convergence}};
        \addplot[color=cOurs, mark=*, mark size=2pt, very thick] coordinates {
            (0.70, 263.36) (0.75, 358.17) (0.80, 452.53) (0.85, 1025.17)
        };
        \node[font=\fontsize{7}{8}\selectfont, anchor=west, cOurs] at (axis cs:0.753, 358) {\textbf{358}};
        \node[font=\fontsize{7}{8}\selectfont, anchor=west, cOurs] at (axis cs:0.803, 453) {\textbf{453}};
        \node[font=\fontsize{7}{8}\selectfont, anchor=west, cOurs] at (axis cs:0.853, 1025) {\textbf{1025}};
        \node[font=\fontsize{6}{7}\selectfont\fontfamily{ptm}\selectfont, anchor=south, text=white, fill=cOurs, draw=cOurs, rounded corners=1pt, inner sep=1.5pt] at (axis cs:0.85, 1150) {\textit{convergence}};
        \legend{DRaFT, AlignProp, DRTune, ReFL, \textbf{EasyTune}}
        \end{axis}
        \end{tikzpicture}%
        }
        \caption{{Computational overhead for training.} {\color{blue}Training time to reach given reward scores; curves terminate at convergence.}}
        \label{fig:overhead}
    \end{minipage}
    \hfill
    \begin{minipage}[t]{0.48\linewidth}
        \centering
        \resizebox{\linewidth}{!}{%
        \begin{tikzpicture}
        \begin{axis}[
            width=6.5cm,
            height=5.2cm,
            ylabel={Preference Reward},
            ylabel style={font=\fontsize{9}{10}\selectfont\fontfamily{ptm}\selectfont},
            xlabel={Training Step},
            xlabel style={font=\fontsize{9}{10}\selectfont\fontfamily{ptm}\selectfont},
            xmin=-30, xmax=1230,
            xtick={0,300,600,900,1200},
            xticklabel style={font=\fontsize{8}{9}\selectfont\fontfamily{ptm}\selectfont},
            ymin=-11.2, ymax=-7.2,
            ytick={-11,-10,-9,-8},
            yticklabel style={font=\fontsize{8}{9}\selectfont\fontfamily{ptm}\selectfont},
            legend style={
                at={(0.5,1.03)}, anchor=south,
                legend columns=2,
                font=\fontsize{7.5}{8.5}\selectfont\fontfamily{ptm}\selectfont,
                draw=none,
                fill=none,
                column sep=5pt,
                /tikz/every even column/.append style={column sep=5pt},
            },
            ymajorgrids=true,
            xmajorgrids=false,
            grid style={gray!50, dashed},
            axis x line*=bottom,
            axis y line*=left,
            axis line style={black!70},
            tick style={black!70},
            clip=true,
        ]
        \addplot[color=cPref, mark=none, line width=0.5pt, opacity=0.25, forget plot] coordinates {
            (0,-10.754) (20,-10.589) (40,-10.385) (60,-10.539) (80,-10.394)
            (96,-10.442) (100,-10.230) (120,-9.587) (140,-9.363) (160,-9.507)
            (180,-9.478) (192,-9.214) (200,-9.431) (220,-8.887) (240,-8.957)
            (260,-9.324) (280,-9.155) (288,-8.846) (300,-8.943) (320,-8.862)
            (340,-8.828) (360,-9.263) (380,-9.441) (384,-9.391) (400,-9.039)
            (420,-9.421) (440,-8.701) (460,-9.096) (480,-9.211) (500,-9.420)
            (520,-8.807) (540,-8.690) (560,-8.954) (576,-8.993) (580,-8.533)
            (600,-8.700) (620,-8.027) (640,-8.818) (660,-9.066) (672,-8.790)
            (680,-8.541) (700,-8.621) (720,-9.097) (740,-9.043) (760,-8.965)
            (768,-8.832) (780,-8.909) (800,-8.118) (820,-8.223) (840,-8.657)
            (860,-8.500) (864,-8.807) (880,-8.091) (900,-8.393) (920,-8.731)
            (940,-7.997) (960,-8.433) (980,-8.501) (1000,-8.429) (1020,-8.765)
            (1040,-8.173) (1056,-8.357) (1060,-8.247) (1080,-8.328) (1100,-7.630)
            (1120,-8.418) (1140,-8.194) (1152,-7.951) (1160,-8.127) (1180,-8.064)
            (1200,-7.954)
        };
        \addplot[color=cPref, mark=none, thick, forget plot] coordinates {
            (0,-10.754) (40,-10.576) (80,-10.440) (120,-10.003)
            (160,-9.697) (200,-9.419) (240,-9.192) (280,-9.125)
            (320,-9.064) (360,-9.025) (400,-9.030) (440,-9.000)
            (480,-8.950) (520,-8.870) (560,-8.810) (600,-8.720)
            (640,-8.680) (680,-8.680) (720,-8.750) (760,-8.730)
            (800,-8.510) (840,-8.480) (880,-8.410) (920,-8.390)
            (960,-8.350) (1000,-8.340) (1040,-8.270) (1080,-8.170)
            (1120,-8.100) (1160,-8.050) (1200,-7.954)
        };
        \addlegendimage{color=cPref, thick}
        \addlegendentry{Preference}
        \addlegendimage{color=cAuth, thick}
        \addlegendentry{Authenticity ($\times 10^3$)}
        \end{axis}
        \begin{axis}[
            width=6.5cm,
            height=5.2cm,
            axis y line*=right,
            axis x line=none,
            ylabel={Authenticity ($\times 10^3$)},
            ylabel style={font=\fontsize{9}{10}\selectfont\fontfamily{ptm}\selectfont},
            xmin=-30, xmax=1230,
            ymin=-0.1, ymax=0.9,
            ytick={0, 0.2, 0.4, 0.6, 0.8},
            yticklabel style={font=\fontsize{8}{9}\selectfont\fontfamily{ptm}\selectfont},
            axis line style={black!70},
            tick style={black!70},
            ymajorgrids=false,
            clip=true,
        ]
        \addplot[color=cAuth, mark=none, line width=0.5pt, opacity=0.25] coordinates {
            (0,0.064) (20,0.082) (40,0.073) (60,0.171) (80,0.290)
            (96,0.078) (100,0.305) (120,0.233) (140,0.085) (160,0.276)
            (180,0.161) (192,0.083) (200,0.166) (220,0.071) (240,0.072)
            (260,0.224) (280,0.166) (288,0.123) (300,0.126) (320,0.229)
            (340,0.180) (360,0.144) (380,0.258) (384,0.146) (400,0.133)
            (420,0.117) (440,0.130) (460,0.266) (480,0.121) (500,0.231)
            (520,0.251) (540,0.201) (560,0.196) (576,0.389) (580,0.435)
            (600,0.377) (620,0.402) (640,0.383) (660,0.353) (672,0.558)
            (680,0.168) (700,0.430) (720,0.329) (740,0.348) (760,0.236)
            (768,0.340) (780,0.657) (800,0.423) (820,0.464) (840,0.359)
            (860,0.332) (864,0.351) (880,0.438) (900,0.625) (920,0.188)
            (940,0.273) (960,0.230) (980,0.344) (1000,0.516) (1020,0.277)
            (1040,0.392) (1056,0.469) (1060,0.346) (1080,0.352) (1100,0.173)
            (1120,0.496) (1140,0.307) (1152,0.257) (1160,0.461) (1180,0.712)
            (1200,0.361)
        };
        \addplot[color=cAuth, mark=none, thick] coordinates {
            (0,0.064) (40,0.073) (80,0.145) (120,0.185)
            (160,0.175) (200,0.135) (240,0.120) (280,0.160)
            (320,0.180) (360,0.170) (400,0.160) (440,0.165)
            (480,0.195) (520,0.225) (560,0.270) (600,0.360)
            (640,0.380) (680,0.340) (720,0.340) (760,0.320)
            (800,0.400) (840,0.385) (880,0.390) (920,0.350)
            (960,0.300) (1000,0.370) (1040,0.385) (1080,0.340)
            (1120,0.340) (1160,0.400) (1200,0.361)
        };
        \end{axis}
        \end{tikzpicture}%
        }
        \caption{{Authenticity and preference reward curves during MotionRFT training.}}
        \label{fig:training_curve}
    \end{minipage}
\end{figure}

\begin{table*}[t]
    \centering

\caption{\textbf{Text-to-motion generation results on HumanML3D.} {\color{blue}Results are reported for kinematic, joint, and rotation representations. $\rightarrow$ denotes closer to ``Real'' is better, and parentheses give the change over the base model. Bold: best; underline: second best.}}
    \footnotesize
    \begin{adjustbox}{max width=\textwidth}
        \begin{tabular}{l ccccccc}
      \toprule
     \multirow{2}{*}{Method} & \multicolumn{3}{c}{R Precision $\uparrow$} & \multirow{2}{*}{FID $\downarrow$} & \multirow{2}{*}{MM Dist $\downarrow$} & \multirow{2}{*}{Diversity $\rightarrow$} & \multirow{2}{*}{\color{blue}MModality $\uparrow$} \\
      \cmidrule(lr){2-4}
       ~& Top 1 & Top 2 & Top 3 & & & & \\
      \midrule
      \multicolumn{8}{c}{\textit{Kinematic Representation}} \\
      \midrule
     Real       & 0.511 & 0.703 & 0.797 & 0.002 & 2.974 & 9.503 & {\color{blue}--} \\
      TM2T  \cite{guo2022tm2t}&0.424$^{\pm0.003}$ & 0.618$^{\pm0.003}$ & 0.729$^{\pm0.002}$ & 1.501$^{\pm0.017}$ & 3.467$^{\pm0.011}$ & 8.589$^{\pm0.076}$ & {\color{blue}2.424$^{\pm0.093}$}\\

      T2M   \cite{guo2022generating} & 0.455$^{\pm 0.003}$ & 0.636$^{\pm 0.003}$ & 0.736$^{\pm 0.002}$ & 1.087$^{\pm 0.021}$ & 3.347$^{\pm 0.008}$ & 9.175$^{\pm 0.083}$ & {\color{blue}2.219$^{\pm0.074}$} \\
      MDM  \cite{tevet2023human} & 0.455$^{\pm 0.006}$ & 0.645$^{\pm 0.007}$ & 0.749$^{\pm 0.006}$ & 0.489$^{\pm 0.047}$ & 3.330$^{\pm 0.025}$ & 9.920$^{\pm 0.083}$ & {\color{blue}\textbf{2.799$^{\pm0.072}$}} \\

      T2M-GPT \cite{zhang2023generating} & 0.492$^{\pm 0.003}$ & 0.679$^{\pm 0.002}$ & 0.775$^{\pm 0.002}$ & 0.141$^{\pm 0.005}$ & 3.121$^{\pm 0.009}$ & 9.722$^{\pm 0.082}$ & {\color{blue}1.831$^{\pm0.048}$}\\
      ReMoDiffuse  \cite{zhang2023remodiffuse}&0.510$^{\pm0.005}$ & 0.698$^{\pm0.006}$ & 0.795$^{\pm0.004}$ & 0.103$^{\pm0.004}$ & 2.974$^{\pm0.016}$ & 9.018$^{\pm0.075}$ & {\color{blue}1.795$^{\pm0.043}$}\\
      AttT2M \cite{zhong2023attt2m}&0.499$^{\pm 0.003}$ & 0.690$^{\pm 0.002}$ & 0.786$^{\pm 0.002}$ & 0.112$^{\pm 0.006}$ & 3.038$^{\pm 0.007}$ & 9.700$^{\pm 0.090}$ & {\color{blue}2.452$^{\pm0.051}$}\\

      MotionDiffuse  \cite{zhang2022motiondiffuse}  & 0.491$^{\pm 0.001}$ & 0.681$^{\pm 0.001}$ & 0.782$^{\pm 0.001}$ & 0.630$^{\pm 0.001}$ & 3.113$^{\pm 0.001}$ & 9.410$^{\pm 0.049}$ & {\color{blue}1.553$^{\pm0.042}$} \\
      MotionLCM  \cite{Dai2025} & 0.502$^{\pm 0.003}$ & 0.698$^{\pm 0.002}$ & 0.798$^{\pm 0.002}$ & 0.304$^{\pm 0.012}$ & 3.012$^{\pm 0.007}$ & 9.607$^{\pm 0.066}$ & {\color{blue}2.259$^{\pm0.092}$} \\
      MotionMamba  \cite{motionmaba} & 0.502$^{\pm 0.003}$ & 0.693$^{\pm 0.002}$ & 0.792$^{\pm 0.002}$ & 0.281$^{\pm 0.009}$ & 3.060$^{\pm 0.058}$  & 9.871$^{\pm 0.084}$ & {\color{blue}2.294$^{\pm0.058}$}\\
      CoMo  \cite{Huang2024CoMo}  & 0.502$^{\pm 0.002}$ & 0.692$^{\pm 0.007}$ & 0.790$^{\pm 0.002}$ & 0.262$^{\pm 0.004}$ & 3.032$^{\pm 0.015}$ & 9.936$^{\pm 0.066}$ & {\color{blue}1.013$^{\pm0.046}$} \\
      ParCo   \cite{zou2024parco}  & 0.515$^{\pm 0.003}$ & 0.706$^{\pm 0.003}$ & 0.801$^{\pm 0.002}$ & 0.109$^{\pm 0.005}$ & 2.927$^{\pm 0.008}$ & \underline{9.576$^{\pm 0.088}$} & {\color{blue}1.382$^{\pm0.060}$} \\
     SoPo \cite{tan2024sopo}  & {0.528$^{\pm 0.005}$} & {0.722$^{\pm 0.004}$} & {0.827$^{\pm 0.004}$ } & {0.174$^{\pm 0.005}$} & {2.939$^{\pm 0.011}$} & 9.584$^{\pm 0.074}$ & {\color{blue}2.301$^{\pm0.076}$} \\
      \midrule
      MLD \cite{chen2023executing} (Base Model)  & 0.504$^{\pm 0.002}$ & 0.698$^{\pm 0.003}$ & 0.796$^{\pm 0.002}$ & 0.450$^{\pm 0.011}$ & 3.052$^{\pm 0.009}$ & 9.634$^{\pm 0.064}$ & {\color{blue}2.267$^{\pm0.082}$} \\
      \rowcolor{gray!20} w/ EasyTune (Ours) &
      0.581$^{\pm 0.003}$\,\gain{(+15.3\%)} &
      0.769$^{\pm 0.002}$\,\gain{(+10.2\%)} &
      0.855$^{\pm 0.002}$\,\gain{(+7.4\%)} &
      0.132$^{\pm 0.005}$\,\gain{(+70.7\%)} &
      2.637$^{\pm 0.007}$\,\gain{(+13.6\%)} &
      \textbf{9.465$^{\pm 0.075}$}\,\gain{(+0.09)} & {\color{blue}1.621$^{\pm0.053}$}\,\drop{($-$0.646)} \\
      \rowcolor{gray!20} w/ MotionRFT (Ours) &
      \underline{0.593$^{\pm 0.002}$}\,\gain{(+17.6\%)} &
      \underline{0.781$^{\pm 0.002}$}\,\gain{(+11.8\%)} &
      \underline{0.863$^{\pm 0.002}$}\,\gain{(+8.4\%)} &
      0.101$^{\pm 0.004}$\,\gain{(+77.6\%)} &
      \underline{2.544$^{\pm 0.007}$}\,\gain{(+16.6\%)} &
      9.613$^{\pm 0.091}$\,\gain{(+0.02)} & {\color{blue}1.638$^{\pm0.064}$}\,\drop{($-$0.629)} \\[0.3em]
      {MLD++ \cite{motionlcm-v2} (Base Model) } & 0.548$^{\pm 0.003}$ & 0.738$^{\pm 0.003}$ & 0.829$^{\pm 0.002}$ & 0.073$^{\pm 0.003}$ & 2.810$^{\pm 0.008}$ & 9.658$^{\pm 0.089}$ & {\color{blue}1.675$^{\pm0.055}$} \\
      \rowcolor{gray!20} w/ EasyTune (Ours) &
      0.591$^{\pm 0.004}$\,\gain{(+7.8\%)} &
      0.777$^{\pm 0.002}$\,\gain{(+5.3\%)} &
      0.859$^{\pm 0.002}$\,\gain{(+3.6\%)} &
      \underline{0.069$^{\pm 0.003}$}\,\gain{(+5.5\%)} &
      2.592$^{\pm 0.008}$\,\gain{(+7.8\%)} &
      9.705$^{\pm 0.086}$\,\drop{(-0.05)} & 1.326$^{\pm0.083}$\,\drop{($-$0.349)} \\
      \rowcolor{gray!20} w/ MotionRFT (Ours) &
      \textbf{0.603$^{\pm 0.003}$}\,\gain{(+10.0\%)} &
      \textbf{0.786$^{\pm 0.003}$}\,\gain{(+6.5\%)} &
      \textbf{0.872$^{\pm 0.003}$}\,\gain{(+5.2\%)} &
      \textbf{0.052$^{\pm 0.002}$}\,\gain{(+28.8\%)} &
      \textbf{2.528$^{\pm 0.009}$}\,\gain{(+10.0\%)} &
      9.683$^{\pm 0.076}$\,\drop{(-0.03)} & 1.342$^{\pm0.055}$\,\drop{($-$0.333)} \\
      \midrule
      \multicolumn{8}{c}{\textit{Joint Representation}} \\
      \midrule
     Real       & 0.503 & 0.696 & 0.795 & 0.000 & 3.244 & {\color{blue}10.335} & {\color{blue}--} \\
      MDM$^{50S}$  \cite{tevet2023human} & 0.440$^{\pm 0.007}$ & 0.636$^{\pm 0.006}$ & 0.742$^{\pm 0.004}$ & 0.518$^{\pm 0.032}$ & 3.640$^{\pm 0.028}$ & 11.224 & {\color{blue}\textbf{3.604$^{\pm0.031}$}} \\
      MotionDiffuse  \cite{zhang2022motiondiffuse} & 0.450$^{\pm 0.006}$ & 0.641$^{\pm 0.005}$ & 0.753$^{\pm 0.005}$ & 0.778$^{\pm 0.005}$ & 3.490$^{\pm 0.023}$ & 11.163 & {\color{blue}3.179$^{\pm0.046}$} \\
      ReMoDiffuse  \cite{zhang2023remodiffuse} & 0.468$^{\pm 0.003}$ & 0.653$^{\pm 0.003}$ & 0.754$^{\pm 0.005}$ & 0.883$^{\pm 0.021}$ & 3.414$^{\pm 0.020}$ & 11.096 & {\color{blue}2.703$^{\pm0.154}$} \\
      MLD++  \cite{motionlcm-v2} & 0.500$^{\pm 0.003}$ & 0.691$^{\pm 0.002}$ & 0.789$^{\pm 0.001}$ & 2.027$^{\pm 0.021}$ & 3.220$^{\pm 0.008}$ & 10.951 & {\color{blue}1.924$^{\pm0.065}$} \\
      MotionLCM  \cite{Dai2025} & 0.501$^{\pm 0.002}$ & 0.693$^{\pm 0.002}$ & 0.790$^{\pm 0.002}$ & 2.267$^{\pm 0.023}$ & \underline{3.192$^{\pm 0.009}$} & 10.879 & {\color{blue}1.780$^{\pm0.062}$} \\
      MARDM-$\epsilon$  \cite{meng2024rethinking} & 0.492$^{\pm 0.006}$ & 0.690$^{\pm 0.005}$ & 0.790$^{\pm 0.005}$ & 0.116$^{\pm 0.004}$ & 3.349$^{\pm 0.010}$ & 11.018 & {\color{blue}2.470$^{\pm0.053}$} \\
      MARDM-$v$  \cite{meng2024rethinking} & 0.500$^{\pm 0.004}$ & 0.695$^{\pm 0.003}$ & 0.795$^{\pm 0.003}$ & 0.114$^{\pm 0.007}$ & 3.270$^{\pm 0.009}$ & 10.922$^{\pm0.058}$ & {\color{blue}2.231$^{\pm0.071}$} \\
      \midrule
      ACMDM  \cite{meng2025absolute} (Base Model) & \underline{0.508$^{\pm 0.002}$} & \underline{0.701$^{\pm 0.003}$} & \underline{0.798$^{\pm 0.003}$} & \underline{0.109$^{\pm 0.005}$} & 3.253$^{\pm 0.010}$ & {\color{blue}10.814$^{\pm0.065}$} & {\color{blue}2.156$^{\pm0.061}$} \\
      \rowcolor{gray!20} w/ MotionRFT (Ours) &
      \textbf{0.517$^{\pm 0.002}$}\,\gain{(+1.8\%)} &
      \textbf{0.734$^{\pm 0.002}$}\,\gain{(+4.7\%)} &
      \textbf{0.819$^{\pm 0.002}$}\,\gain{(+2.6\%)} &
      \textbf{0.084$^{\pm 0.004}$}\,\gain{(+22.9\%)} &
      \textbf{3.172$^{\pm 0.010}$}\,\gain{(+2.5\%)} & \textbf{10.583$^{\pm0.052}$}\,\gain{(+0.23)} & {\color{blue}1.832$^{\pm0.075}$}\,\drop{($-$0.324)} \\
      \midrule
      \multicolumn{8}{c}{\textit{Rotation Representation}} \\
      \midrule
     {\color{blue}Real} & {\color{blue}0.473$^{\pm 0.004}$} & {\color{blue}0.680$^{\pm 0.003}$} & {\color{blue}0.774$^{\pm 0.003}$} & {\color{blue}0.000} & {\color{blue}2.847} & {\color{blue}10.399$^{\pm 0.112}$} & {\color{blue}--} \\
      \midrule
      {\color{blue}HY-Motion-lite \cite{hymotion2025} (Base Model)} & {\color{blue}0.434$^{\pm 0.003}$} & {\color{blue}0.630$^{\pm 0.003}$} & {\color{blue}0.721$^{\pm 0.002}$} & {\color{blue}1.312} & {\color{blue}3.238} & {\color{blue}9.607$^{\pm 0.127}$} & {\color{blue}1.644$^{\pm0.047}$} \\
      \rowcolor{gray!20} {\color{blue}w/ MotionRFT (Ours)} &
      {\color{blue}\underline{0.469$^{\pm 0.003}$}\,\gain{(+8.1\%)}} &
      {\color{blue}\underline{0.676$^{\pm 0.003}$}\,\gain{(+7.3\%)}} &
      {\color{blue}\underline{0.766$^{\pm 0.002}$}\,\gain{(+6.2\%)}} &
      {\color{blue}\underline{1.173}\,\gain{(+10.6\%)}} &
      {\color{blue}\underline{3.028}\,\gain{(+6.5\%)}} &
      {\color{blue}9.659$^{\pm 0.121}$\,\gain{(+0.05)}} & {\color{blue}1.492$^{\pm0.044}$\,\drop{($-$0.152)}} \\[0.3em]
      {\color{blue}HY-Motion \cite{hymotion2025} (Base Model)} & {\color{blue}0.445$^{\pm 0.004}$} & {\color{blue}0.651$^{\pm 0.003}$} & {\color{blue}0.743$^{\pm 0.002}$} & {\color{blue}1.258} & {\color{blue}3.104} & {\color{blue}\underline{9.721$^{\pm 0.116}$}} & {\color{blue}\textbf{1.789$^{\pm0.033}$}} \\
      \rowcolor{gray!20} {\color{blue}w/ MotionRFT (Ours)} &
      {\color{blue}\textbf{0.499$^{\pm 0.004}$}\,\gain{(+12.1\%)}} &
      {\color{blue}\textbf{0.717$^{\pm 0.003}$}\,\gain{(+10.1\%)}} &
      {\color{blue}\textbf{0.807$^{\pm 0.002}$}\,\gain{(+8.6\%)}} &
      {\color{blue}\textbf{0.985}\,\gain{(+21.7\%)}} &
      {\color{blue}\textbf{2.775}\,\gain{(+10.6\%)}} &
      {\color{blue}\textbf{9.782$^{\pm 0.111}$}\,\gain{(+0.06)}} & {\color{blue}1.601$^{\pm0.031}$\,\drop{($-$0.188)}} \\
      \bottomrule
\end{tabular}
\end{adjustbox}

\label{tab:sota_humanmld3d}
\end{table*}

    \begin{table}[t]
      \centering
          \setlength{\tabcolsep}{0.5pt}
          \caption{\textbf{Performance enhancement of diffusion-based motion generation methods} {\color{blue} on HumanML3D.}
          {\color{blue}ODE-based models MLD, MLD++, and MLCM use the one-step prediction in the first case of Eq.~\eqref{eq:reward}.
          SDE-based models MDM and MotionDiffuse use the noise-aware perception in the second case.
          Tab.~\ref{tab:sota_humanmld3d} and Tab.~\ref{tab:addlabel} report the default noise-aware perception.}}

          \begin{adjustbox}{max width=\columnwidth}
          \begin{tabular}{l@{\hspace{2pt}} l l l l l l l}
              \toprule
              \multirow{2}{*}{Method} & \multicolumn{3}{c}{R Precision $\uparrow$} & \multirow{2}{*}{FID $\downarrow$} & \multirow{2}{*}{MM Dist $\downarrow$} & \multirow{2}{*}{Diversity $\rightarrow$} & \multirow{2}{*}{\color{blue}MModality $\uparrow$} \\
              \cmidrule(lr){2-4}& {Top 1} & {Top 2} & {Top 3}  & & & & \\
              \midrule
              Real & 0.511 & 0.703 & 0.797 & 0.002 & 2.974 & 9.503 & {\color{blue}--} \\
              \midrule
              {MLD \cite{chen2023executing}} & 0.504 & 0.698 & 0.796 & 0.450 & 3.052 & 9.634 & {\color{blue}2.267$^{\pm0.082}$} \\
              \rowcolor{gray!20} {w/ EasyTune} &
              0.568$\gain{_{+12.7\%}}$ &
              0.754$\gain{_{+8.0\%}}$ &
              0.846$\gain{_{+6.3\%}}$ &
              0.194$\gain{_{+56.9\%}}$ &
              2.672$\gain{_{+12.5\%}}$ &
              9.368$\drop{_{-0.00}}$ & {\color{blue}1.501$^{\pm0.051}$} \\
              \rowcolor{gray!20} {w/ MotionRFT} &
              0.577$\gain{_{+14.5\%}}$ &
              0.767$\gain{_{+9.9\%}}$ &
              0.854$\gain{_{+7.3\%}}$ &
              0.137$\gain{_{+69.6\%}}$ &
              2.622$\gain{_{+14.1\%}}$ &
              9.666$\drop{_{-0.03}}$ & {\color{blue}1.476$^{\pm0.062}$} \\
              \midrule
              {MLD++ \cite{motionlcm-v2}} & 0.548 & 0.738 & 0.829 & 0.073 & 2.810 & 9.658 & {\color{blue}1.675$^{\pm0.055}$} \\
              \rowcolor{gray!20} {{w/ EasyTune}} &
              {0.581$\gain{_{+6.0\%}}$} &
              {0.762$\gain{_{+3.3\%}}$} &
              {0.849$\gain{_{+2.4\%}}$} &
              {0.073$\gain{_{+0.0\%}}$} &
              {2.603$\gain{_{+7.4\%}}$} &
              {9.719$\drop{_{-0.06}}$} & {\color{blue}1.228$^{\pm0.036}$} \\
              \rowcolor{gray!20} {w/ MotionRFT} &
              \textbf{0.592}$\gain{_{+8.0\%}}$ &
              0.778$\gain{_{+5.4\%}}$ &
              0.861$\gain{_{+3.9\%}}$ &
              \textbf{0.064}$\gain{_{+12.3\%}}$ &
              \textbf{2.598}$\gain{_{+7.5\%}}$ &
              9.461$\gain{_{+0.11}}$ & {\color{blue}1.209$^{\pm0.043}$} \\
              \midrule
              MLCM$^{1S}$ \cite{Dai2025} & 0.502 & 0.701 & 0.803 & 0.467 & 3.022 & 9.631 & {\color{blue}2.172$^{\pm0.082}$} \\
              \rowcolor{gray!20} {w/ EasyTune} &
              0.571$\gain{_{+13.7\%}}$ &
              0.766$\gain{_{+9.3\%}}$ &
              0.854$\gain{_{+6.4\%}}$ &
              0.188$\gain{_{+59.7\%}}$ &
              2.647$\gain{_{+12.4\%}}$ &
              9.692$\drop{_{-0.06}}$ & {\color{blue}1.577$^{\pm0.051}$} \\
              \rowcolor{gray!20} {w/ MotionRFT} &
              0.583$\gain{_{+16.1\%}}$ &
              \textbf{0.779}$\gain{_{+11.1\%}}$ &
              \textbf{0.868}$\gain{_{+8.1\%}}$ &
              0.149$\gain{_{+68.1\%}}$ &
              2.627$\gain{_{+13.1\%}}$ &
              9.629$\gain{_{+0.00}}$ & {\color{blue}1.561$^{\pm0.062}$} \\
              \midrule
              {MLCM$^{4S}$} \cite{Dai2025} & 0.502 & 0.698 & 0.798 & 0.304 & 3.012 & 9.607 & {\color{blue}2.259$^{\pm0.092}$} \\
              \rowcolor{gray!20} {w/ EasyTune} &
              0.565$\gain{_{+12.5\%}}$ &
              0.760$\gain{_{+8.8\%}}$ &
              0.848$\gain{_{+6.3\%}}$ &
              0.200$\gain{_{+34.2\%}}$ &
              2.691$\gain{_{+10.7\%}}$ &
              9.812$\drop{_{-0.21}}$ & {\color{blue}1.672$^{\pm0.059}$} \\
              \rowcolor{gray!20} {w/ MotionRFT} &
              0.578$\gain{_{+15.1\%}}$ &
              0.772$\gain{_{+10.6\%}}$ &
              0.861$\gain{_{+7.9\%}}$ &
              0.170$\gain{_{+44.1\%}}$ &
              2.647$\gain{_{+12.1\%}}$ &
              \textbf{9.473}$\gain{_{+0.07}}$ & {\color{blue}1.620$^{\pm0.072}$} \\
              \midrule
              {MDM$^{50S}$} \cite{tevet2023human} & 0.455 & 0.645 & 0.749 & 0.489 & 3.330 & 9.920 & {\color{blue}2.799$^{\pm0.072}$} \\
              \rowcolor{gray!20} {w/ EasyTune} &
              0.472$\gain{_{+3.7\%}}$ &
              0.679$\gain{_{+5.3\%}}$ &
              0.787$\gain{_{+5.1\%}}$ &
              0.411$\gain{_{+16.0\%}}$ &
              3.117$\gain{_{+6.4\%}}$ &
              9.239$\gain{_{+0.15}}$ & {\color{blue}1.995$^{\pm0.047}$} \\
              \rowcolor{gray!20} {w/ MotionRFT} &
              0.480$\gain{_{+5.5\%}}$ &
              0.684$\gain{_{+6.0\%}}$ &
              0.801$\gain{_{+6.9\%}}$ &
              0.358$\gain{_{+26.8\%}}$ &
              3.067$\gain{_{+7.9\%}}$ &
              9.368$\gain{_{+0.28}}$ & {\color{blue}2.024$^{\pm0.056}$} \\
              \midrule
              {Mo.Diffuse} \cite{zhang2022motiondiffuse} & 0.491 & 0.681 & 0.782 & 0.630 & 3.113 & 9.410 & {\color{blue}1.553$^{\pm0.042}$} \\
              \rowcolor{gray!20} {w/ EasyTune} &
              0.488$\drop{_{-0.6\%}}$ &
              0.686$\gain{_{+0.7\%}}$ &
              0.788$\gain{_{+0.8\%}}$ &
              0.556$\gain{_{+11.7\%}}$ &
              3.068$\gain{_{+1.4\%}}$ &
              9.215$\drop{_{-0.20}}$ & {\color{blue}1.143$^{\pm0.027}$} \\
              \rowcolor{gray!20} {w/ MotionRFT} &
              0.496$\gain{_{+1.0\%}}$ &
              0.693$\gain{_{+1.8\%}}$ &
              0.795$\gain{_{+1.7\%}}$ &
              0.498$\gain{_{+21.0\%}}$ &
              2.997$\gain{_{+3.7\%}}$ &
              9.404$\drop{_{-0.01}}$ & {\color{blue}1.167$^{\pm0.033}$} \\
              \bottomrule
          \end{tabular}
          \end{adjustbox}
          \label{tab:PlugAndPlay_H3D}
  \end{table}

\noindent \textbf{Efficiency of  RFT Methods.} As shown in Fig.~\ref{fig:overhead}, EasyTune~\cite{tan2026easytune} takes only 1.47\,s per optimization step, about $3\!\times$ to $4\!\times$ faster than baselines, and is the only method that reaches reward scores of 0.80 and 0.85. All compared methods fail to surpass 0.80, confirming that trajectory-level updates~\cite{clark2024directly} struggle to provide sufficient gradient signal for continued improvement. By contrast, EasyTune optimizes rewards at each denoising step, yielding finer gradients that consistently learning well beyond the ceiling of existing approaches.

\noindent \textbf{Generalization across Different Pre-trained Models.} To evaluate the generalization of EasyTune~\cite{tan2026easytune} across pre-trained text-to-motion models, we applied it to MLD~\cite{chen2023executing}, MLD++~\cite{motionlcm-v2}, MLCM$^{1S}$~\cite{Dai2025}, and MDM$^{50S}$~\cite{tevet2023human}. Note that for ODE-based samplers, we adopt the one-step prediction reward (Eq.~\eqref{eq:reward}, ODE-based), which differs from the noise-aware reward used in Tab.~\ref{tab:addlabel} and Tab.~\ref{tab:sota_humanmld3d}. As shown in Tab.~\ref{tab:PlugAndPlay_H3D}, EasyTune consistently improved performance. For instance, MLD saw a 12.7\% increase in R-Precision at Top 1 to 0.568 and a 56.9\% reduction in FID to 0.194. MLD++ achieved a 6.0\% gain in R-Precision at Top 1 to 0.581 and a 7.4\% improvement in MM Dist to 2.603. MLCM$^{1S}$ and MDM$^{50S}$ also showed significant FID reductions of 59.7\% and 16.0\%, respectively. These results highlight the generalization of EasyTune across various diffusion-based architectures.

\begin{table*}[tbp]
    \centering
    \caption{\textbf{Comparison of fine-tuning methods on HumanML3D.} {\color{blue}``Memory'' is peak GPU memory during fine-tuning, and subscripts give the change over the base model.}}
    \footnotesize
    \begin{adjustbox}{max width=\textwidth}
    \begin{tabular}{l cccccccc}
        \toprule
        \multirow{2}{*}{Method} & \multicolumn{3}{c}{R Precision $\uparrow$} & \multirow{2}{*}{FID $\downarrow$} & \multirow{2}{*}{MM Dist $\downarrow$} & \multirow{2}{*}{Diversity $\rightarrow$} & \multirow{2}{*}{\color{blue}MModality $\uparrow$} & \multirow{2}{*}{Memory (GB) $\downarrow$} \\
        \cmidrule(lr){2-4}
        ~ & Top 1 & Top 2 & Top 3 & & & & & \\
        \midrule
        Real & 0.511 & 0.703 & 0.797 & 0.002 & 2.974 & 9.503 & {\color{blue}--} & -- \\
        MLD~\cite{chen2023executing} (Base Model) & 0.504 & 0.698 & 0.796 & 0.450 & 3.052 & 9.634 & {\color{blue}2.267$^{\pm0.082}$} & 15.21 \\
        \midrule
        w/ ReFL-10  \cite{clark2024directly}  & 0.533$\gain{_\text{{+5.8\%}}}$ & 0.720$\gain{_\text{{+3.2\%}}}$ & 0.821$\gain{_\text{{+3.1\%}}}$ & 0.207$\gain{_\text{{+54.0\%}}}$ & 2.852$\gain{_\text{{+6.6\%}}}$  & 10.129$\drop{_\text{-0.495}}$     & {\color{blue}1.693$^{\pm0.063}$} & \textbf{22.10$\drop{_\text{{+6.89}}}$}      \\
        w/ ReFL-20  \cite{clark2024directly}  & 0.528$\gain{_\text{{+4.8\%}}}$ & 0.718$\gain{_\text{{+2.9\%}}}$ & 0.813$\gain{_\text{{+2.1\%}}}$ & 0.241$\gain{_\text{{+46.4\%}}}$ & 2.883$\gain{_\text{{+5.5\%}}}$  & 10.189$\drop{_\text{-0.555}}$     & {\color{blue}1.706$^{\pm0.066}$} & \textbf{22.10$\drop{_\text{{+6.89}}}$}       \\
        w/ DRaFT-10 \cite{clark2024directly} & {0.565}$\gain{_\text{{+12.1\%}}}$ & 0.757$\gain{_\text{{+8.5\%}}}$ & {0.846}$\gain{_\text{{+6.3\%}}}$ & 0.195$\gain{_\text{{+56.7\%}}}$ & {2.703}$\gain{_\text{{+11.4\%}}}$ & 9.851$\drop{_\text{{-0.217}}}$ & {\color{blue}1.571$^{\pm0.048}$} & 26.56$\drop{_\text{{+11.35}}}$ \\
        w/ DRaFT-50 \cite{clark2024directly} & 0.528$\gain{_\text{{+4.8\%}}}$ & 0.724$\gain{_\text{{+3.7\%}}}$ & 0.819$\gain{_\text{{+2.9\%}}}$ & 0.197$\gain{_\text{{+56.2\%}}}$ & 2.872$\gain{_\text{{+5.9\%}}}$ & \underline{9.641}$\drop{_\text{{-0.007}}}$ & {\color{blue}1.664$^{\pm0.058}$} & 37.32$\drop{_\text{{+22.11}}}$ \\
        w/ AlignProp \cite{prabhudesai2023aligning} & 0.560$\gain{_\text{{+11.1\%}}}$ & {0.753}$\gain{_\text{{+7.9\%}}}$ & 0.841$\gain{_\text{{+5.7\%}}}$ & 0.266$\gain{_\text{{+40.9\%}}}$ & 2.739$\gain{_\text{{+10.3\%}}}$ & 9.877$\drop{_\text{{-0.243}}}$ & {\color{blue}1.638$^{\pm0.061}$} & 30.40$\drop{_\text{{+15.19}}}$ \\
        w/ DRTune \cite{wu2025drtune} & 0.549$\gain{_\text{{+8.9\%}}}$ & 0.746$\gain{_\text{{+6.9\%}}}$ & 0.836$\gain{_\text{{+5.0\%}}}$ & 0.313$\gain{_\text{{+30.4\%}}}$ & 2.795$\gain{_\text{{+8.4\%}}}$ & 9.930$\drop{_\text{{-0.296}}}$ & {\color{blue}1.588$^{\pm0.050}$} & {27.01}$\drop{_\text{{+11.80}}}$ \\
        \rowcolor{gray!20} w/ EasyTune (Ours,  Step Optimization) & \textbf{0.581$\gain{_\text{{+15.3\%}}}$} & \textbf{0.769$\gain{_\text{{+10.2\%}}}$} & \textbf{0.855$\gain{_\text{{+7.4\%}}}$} & \textbf{0.132$\gain{_\text{{+70.7\%}}}$} & \underline{2.637$\gain{_\text{{+13.6\%}}}$} & \textbf{9.465$\gain{_\text{{+0.093}}}$} & {\color{blue}1.621$^{\pm0.053}$} & \textbf{22.10$\drop{_\text{{+6.89}}}$} \\
        \rowcolor{gray!20} w/ EasyTune (Ours, Chain Optimization) & \underline{0.574$\gain{_\text{{+13.9\%}}}$} & \underline{0.766$\gain{_\text{{+9.7\%}}}$} & \underline{0.854$\gain{_\text{{+7.3\%}}}$} & \underline{0.172$\gain{_\text{{+61.8\%}}}$} & \textbf{2.614$\gain{_\text{{+14.4\%}}}$} & {9.348$\drop{_\text{{-0.024}}}$} & {\color{blue}1.607$^{\pm0.056}$} & \underline{{24.21$\drop{_\text{{+9.00}}}$}} \\
        \bottomrule
    \end{tabular}
    \end{adjustbox}
    \label{tab:addlabel}%
\end{table*}%

\begin{figure*}[t]
	\centering
	\includegraphics[width=\textwidth]{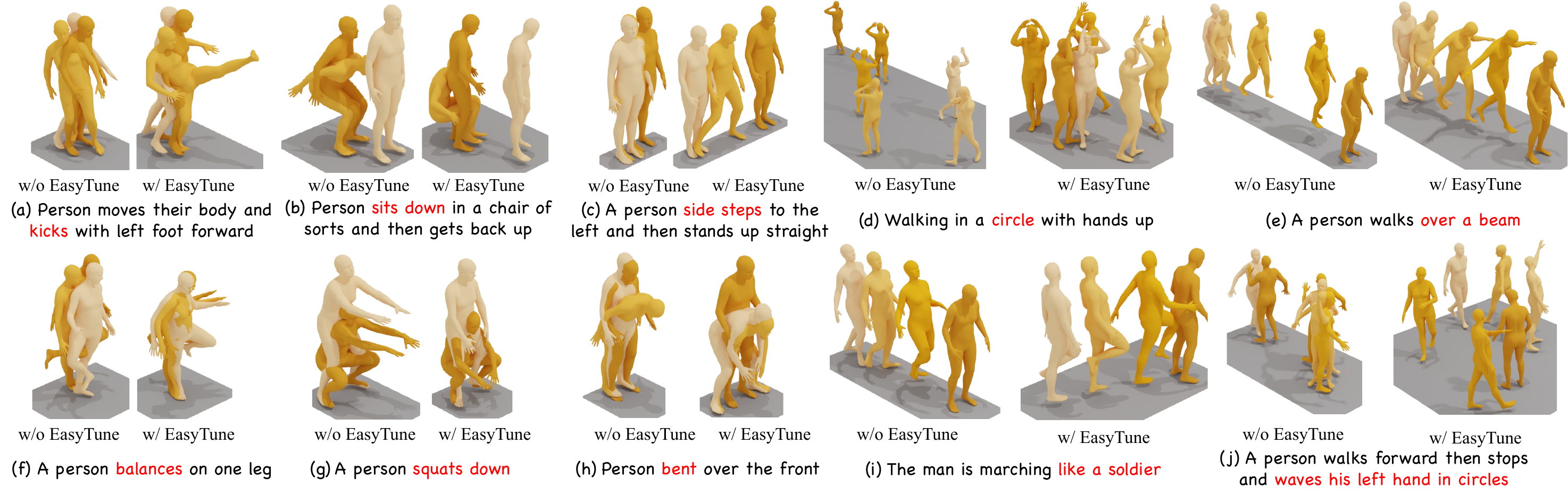}
	\caption{Visual results on HumanML3D dataset. ``w/o EasyTune'' refers to motions generated by the original MLD model~\cite{chen2023executing}, while ``w/ EasyTune'' indicates motions generated by the MLD model fine-tuned using our proposed EasyTune~\cite{tan2026easytune}.}
	\label{fig:vis}
\end{figure*}

\begin{figure}[tbp]
\centering
\begin{minipage}[t]{0.49\linewidth}
    \centering
    \includegraphics[width=\linewidth]{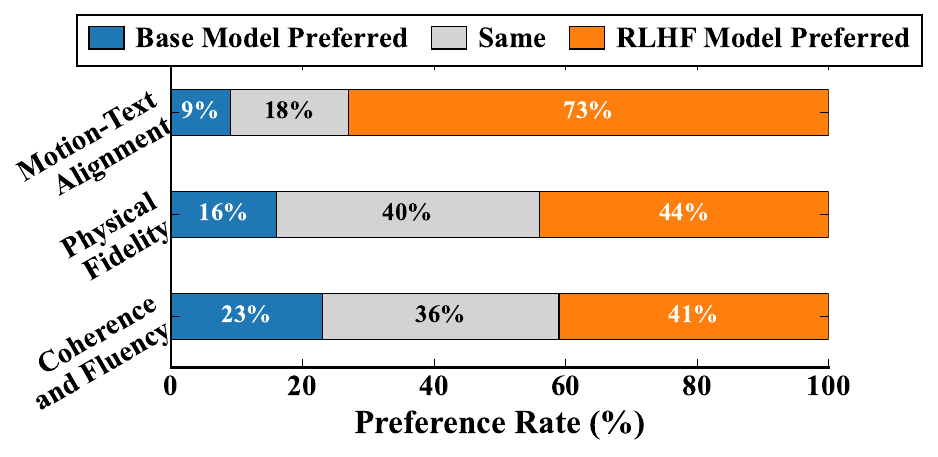}
    \caption{User study results {\color{blue}comparing EasyTune against the base MLD model.}}
    \label{fig:study}
\end{minipage}
\hfill
\begin{minipage}[t]{0.49\linewidth}
    \centering
    \includegraphics[width=\linewidth]{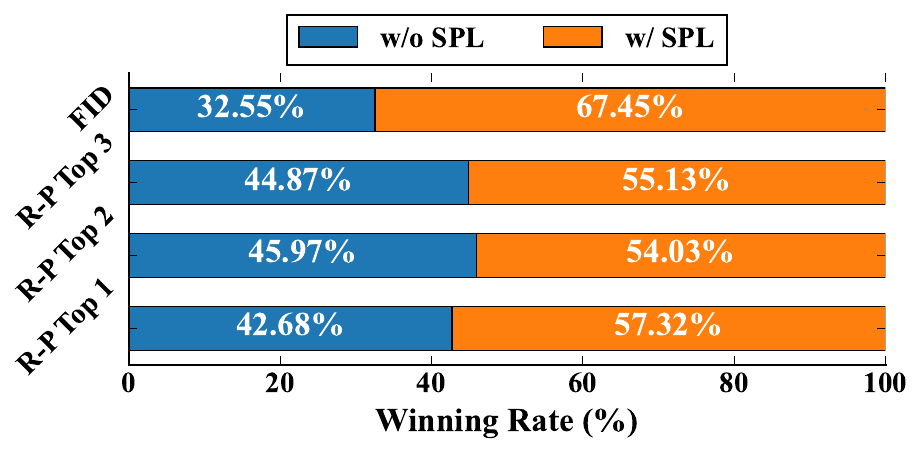}
    \caption{Ablation on SPL for generation, {\color{blue}reporting win rates with and without SPL across metrics.}}
    \label{fig:pair}
\end{minipage}
\end{figure}

\noindent \textbf{Comparison with SoTA Text-to-Motion Methods.} We evaluate EasyTune~\cite{tan2026easytune} on text-to-motion generation using MLD~\cite{chen2023executing} and MLD++~\cite{motionlcm-v2} as base models, comparing with state-of-the-art methods on the HumanML3D~\cite{guo2022generating} dataset {\color{blue}in Tab.~\ref{tab:sota_humanmld3d} and on the KIT-ML~\cite{plappert2016kit} dataset in App.~\sref{sec:kit}, where Tab.~\sref{tab:kit_gen} reports generation under the kinematic, joint, and rotation representations and Tab.~\sref{tab:kit_ret} reports text-motion retrieval}. On HumanML3D, EasyTune improves the R-Precision at Top 1 of MLD from 0.504 to 0.581 and MLD++ from 0.548 to 0.591, surpassing baselines such as ParCo~\cite{zou2024parco} at 0.515 and ReMoDiffuse~\cite{zhang2023remodiffuse} at 0.510. It also achieves the best MM Dist of 2.637 and 2.592, along with competitive FID of 0.132 and 0.069.

\begin{figure}[t]
\centering
\definecolor{cReAlign}{HTML}{1F77B4}
\definecolor{cSPL}{HTML}{D62728}
\resizebox{\columnwidth}{!}{%
\begin{tikzpicture}
\begin{axis}[
    name=plotH3D,
    ybar=1.5pt,
    bar width=9pt,
    width=0.55\columnwidth,
    height=3.4cm,
    xlabel={\fontsize{7}{8}\selectfont\fontfamily{ptm}\selectfont (a) HumanML3D},
    xlabel style={yshift=2pt},
    ylabel={Retrieval Accuracy (\%)},
    ylabel style={font=\fontsize{7}{8}\selectfont\fontfamily{ptm}\selectfont},
    symbolic x coords={R@3, R@5, R@10},
    xtick=data,
    xticklabel style={font=\fontsize{6.5}{7.5}\selectfont\fontfamily{ptm}\selectfont},
    ymin=84, ymax=100,
    ytick={84,86,88,90,92,94,96,98,100},
    yticklabel style={font=\fontsize{6}{7}\selectfont\fontfamily{ptm}\selectfont},
    legend style={
        at={(1.05,1.02)}, anchor=south,
        legend columns=2,
        font=\fontsize{6.5}{7.5}\selectfont\fontfamily{ptm}\selectfont,
        draw=none,
        fill=none,
        column sep=6pt,
    },
    ymajorgrids=true,
    xmajorgrids=false,
    grid style={gray!50, dashed},
    axis x line*=bottom,
    axis y line*=left,
    axis line style={black!70},
    tick style={black!70},
    enlarge x limits=0.3,
    clip=false,
    every node near coord/.append style={
        font=\fontsize{5}{6}\selectfont\fontfamily{ptm}\selectfont,
        anchor=south,
    },
    nodes near coords,
    nodes near coords align={vertical},
    point meta=explicit symbolic,
]
\addplot[fill=cReAlign!60, draw=cReAlign!90] coordinates {
    (R@3, 87.44) [87.44]
    (R@5, 91.97) [91.97]
    (R@10, 96.28) [96.28]
};
\addplot[fill=cSPL!60, draw=cSPL!90] coordinates {
    (R@3, 88.66) [88.66]
    (R@5, 92.81) [92.81]
    (R@10, 96.95) [96.75]
};
\legend{ReAlign~(Baseline), \textbf{w/ SPL (Ours)}}
\end{axis}
\begin{axis}[
    at={(plotH3D.east)},
    anchor=west,
    xshift=8pt,
    ybar=1.5pt,
    bar width=9pt,
    width=0.55\columnwidth,
    height=3.4cm,
    xlabel={\fontsize{7}{8}\selectfont\fontfamily{ptm}\selectfont (b) KIT-ML},
    xlabel style={yshift=2pt},
    symbolic x coords={R@3, R@5, R@10},
    xtick=data,
    xticklabel style={font=\fontsize{6.5}{7.5}\selectfont\fontfamily{ptm}\selectfont},
    ymin=80, ymax=100,
    ytick={80,82,84,86,88,90,92,94,96,98,100},
    yticklabel style={font=\fontsize{6}{7}\selectfont\fontfamily{ptm}\selectfont},
    ymajorgrids=true,
    xmajorgrids=false,
    grid style={gray!50, dashed},
    axis x line*=bottom,
    axis y line*=left,
    axis line style={black!70},
    tick style={black!70},
    enlarge x limits=0.3,
    clip=false,
    every node near coord/.append style={
        font=\fontsize{5}{6}\selectfont\fontfamily{ptm}\selectfont,
        anchor=south,
    },
    nodes near coords,
    nodes near coords align={vertical},
    point meta=explicit symbolic,
]
\addplot[fill=cReAlign!60, draw=cReAlign!90] coordinates {
    (R@3, 82.96) [82.96]
    (R@5, 91.19) [91.19]
    (R@10, 97.29) [97.59]
};
\addplot[fill=cSPL!60, draw=cSPL!90] coordinates {
    (R@3, 84.52) [84.52]
    (R@5, 93.18) [93.18]
    (R@10, 97.99) [97.73]
};
\end{axis}
\end{tikzpicture}%
}
\caption{Text-motion retrieval accuracy of ReAlign and SPL.}
\label{fig:spl_improvement}
\end{figure}

\noindent \textbf{Effectiveness of RFT using Motion Reward.}
We compare MotionRFT, which combines MotionReward with Curriculum Timestep Scheduling, against EasyTune~\cite{tan2026easytune} that uses the single-task ReAlign~\cite{weng2025} reward. As shown in Tab.~\ref{tab:PlugAndPlay_H3D}, MotionRFT consistently outperforms EasyTune on all six kinematic-based models. On MLD, MotionRFT reduces FID from 0.194 to 0.137; on MLD++, it further achieves an FID of 0.064 and MM Dist of 2.598. In Tab.~\ref{tab:sota_humanmld3d}, MotionRFT with MLD++ obtains the best R-Precision Top\,1 of 0.603 and FID of 0.052 on kinematic features; {\color{blue}with HY-Motion on rotation features, it reaches a Top\,1 of 0.499 and FID of 0.985}; on joint features, it improves ACMDM FID from 0.109 to 0.084. Fig.~\ref{fig:training_curve} further shows that the preference reward steadily improves while the authenticity reward remains stable, confirming that MotionRFT effectively leverages multi-dimensional rewards without degradation.

\noindent \textbf{Cross-Representation Generalization.}
A key advantage of our framework is its extensibility across heterogeneous motion representations. As shown in Tab.~\ref{tab:retrieval}, a single MotionReward model achieves strong text-motion retrieval accuracy on kinematic, joint, and rotation representations simultaneously, improving the average retrieval score from {\color{blue}77.10} to {\color{blue}85.52} over ACMDM-Eval on joint features, and from {\color{blue}75.38} to {\color{blue}85.17} over FlowMDM-Eval on rotation features. This unified reward directly enables cross-representation fine-tuning without retraining the reward model. Tab.~\ref{tab:sota_humanmld3d} validates this extensibility on the generation side: MotionRFT improves kinematic-based models such as MLD and MLD++, joint-based ACMDM, and rotation-based {\color{blue}HY-Motion} and {\color{blue}HY-Motion-lite}, all using the same MotionReward. On {\color{blue}HY-Motion}, MotionRFT achieves {\color{blue}12.1\%} and {\color{blue}21.7\%} gains in R-Precision Top\,1 and FID respectively; on ACMDM, it yields a 22.9\% FID reduction. These results confirm that our unified reward framework generalizes across kinematic, joint, and rotation representations, providing consistent supervision without representation-specific adaptation.

\subsection{Canonical vs. Multi-Representation}
\label{sec:canon_analysis}

\begin{figure}[!t]
\centering
\begin{minipage}[t]{0.48\columnwidth}
\centering
\includegraphics[width=\linewidth]{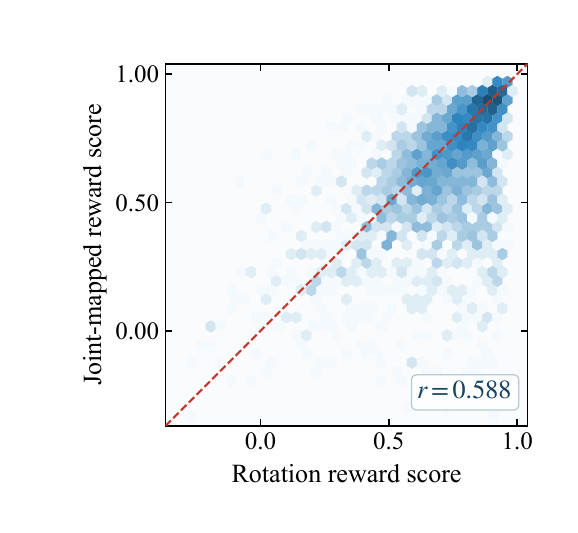}
{\small (a) Reward scores}
\end{minipage}\hfill
\begin{minipage}[t]{0.48\columnwidth}
\centering
\includegraphics[width=\linewidth]{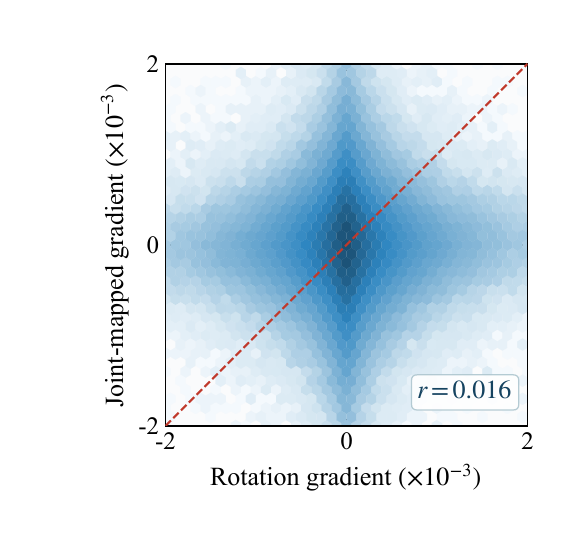}
{\small (b) Optimization gradients}
\end{minipage}
\caption{\textcolor{blue}{\textbf{Correlation between rotation and joint-mapped rewards on HumanML3D.} Evaluated on rotation features $\mathbf{x}^r$: (a) scalar reward scores $s_{\mathrm{rot}}$ vs. $s_{\mathrm{joint}}$; (b) optimization gradients $\mathbf{g}_{\mathrm{rot}}$ vs. $\mathbf{g}_{\mathrm{joint}}$.}}
\label{fig:canon_grad_corr}
\end{figure}

{\color{blue}To investigate whether a single canonical representation can substitute for multi-representation modeling, following rotation-based approaches such as PRIMAL~\cite{zhang2025primal} and Interact2Ar~\cite{ruizponce2026interact2ar}, we consider mapping rotation inputs to 3D joint coordinates via differentiable forward kinematics $f_{\mathrm{FK}}$. Specifically, for a rotation motion $\mathbf{x}^r$, we compare the rotation reward $s_{\mathrm{rot}} = \mathcal{R}_\phi^r(\mathbf{x}^r)$ and gradient $\mathbf{g}_{\mathrm{rot}} = \nabla_{\mathbf{x}^r} s_{\mathrm{rot}}$ against their joint-mapped counterparts $s_{\mathrm{joint}} = \mathcal{R}_\phi^j(f_{\mathrm{FK}}(\mathbf{x}^r))$ and $\mathbf{g}_{\mathrm{joint}} = \nabla_{\mathbf{x}^r} s_{\mathrm{joint}}$. As shown in Fig.~\ref{fig:canon_grad_corr}, while scalar rewards correlate moderately ($r=0.588$), gradient correlation collapses to $r=0.016$ since forward kinematics discards rotational degrees of freedom. Canonical joint supervision thus fails to provide effective gradient guidance for rotation generators. More comprehensive analyses across representations are detailed in App.~\sref{sec:supp_canon_exp}.}

\subsection{Ablation Study, User Study \& Visualization}
\label{sec:ablation}

To assess whether our fine-tuned model really improve the generation quality, we conducted a user study on the first 100 prompts of the HumanML3D test set, where each generated motion was independently evaluated by five participants across text-motion alignment, motion fidelity, and temporal coherence.  Fig.~\ref{fig:vis} visualizes representative examples: the fine-tuned model generates motions with substantially better semantic understanding. For instance, EasyTune correctly produces a marching motion for ``The man is marching like a soldier,'' whereas the original model fails to capture this nuanced behavior. As shown in Fig.~\ref{fig:study}, our method consistently outperforms the baseline MLD model across all criteria, confirming that the quality improvements are genuine rather than artifacts of reward exploitation.

\noindent \textbf{Effectiveness of Self-refinement Preference Learning.} To evaluate the effect of SPL, we compare text-to-motion retrieval accuracy with and without SPL on HumanML3D and KIT-ML, as shown in Fig.~\ref{fig:pair} and \ref{fig:spl_improvement}. SPL consistently improves over the ReAlign baseline across all metrics and both datasets. On KIT-ML, SPL improves R@3 from 82.96\% to 84.52\% and R@5 from 91.19\% to 93.18\%. On HumanML3D, R@3 improves from 87.44\% to 88.66\%. These results confirm that SPL enhances preference modeling without additional human annotations, yielding stronger reward signals for improved semantic alignment.

\section{Conclusion}
\label{sec:conclusion}

In this paper, we present MotionRFT, a unified reinforcement fine-tuning framework for text-to-motion generation, comprising MotionReward and EasyTune. MotionReward establishes a unified semantic space across heterogeneous motion representations and trains multi-dimensional rewards for semantic alignment, human preference, and motion authenticity. We further propose Self-refinement Preference Learning to dynamically construct preference pairs without additional annotations. EasyTune addresses the recursive gradient dependence in differentiable-reward optimization by performing step-aware per-step updates, enabling dense and memory-efficient fine-tuning. Extensive experiments validate strong cross-model and cross-representation generalization, achieving FID 0.132 with 22.10 GB peak memory, saving up to 15.22 GB over prior baselines, reducing FID by 22.9\% on joint-based ACMDM, and improving rotation-based {\color{blue}HY-Motion} by {\color{blue}12.1\%} in R-Precision and {\color{blue}21.7\%} in FID.

\noindent \textbf{Limitations and Future Work.} MotionReward remains fundamentally data driven. Although it accounts for data realism, it cannot fully resolve physical plausibility for real-world humanoid robotics. Future work will extend the framework to physics simulation for better enforce dynamics and contact, and extending MotionRFT to richer conditional generation, like video-driven motion generation {\color{blue}and event-driven collective motion~\cite{cao2026crowdmogen}}.

\ifCLASSOPTIONcaptionsoff
  \newpage
\fi

\bibliographystyle{IEEEtran}
\bibliography{merged_refs}

\begin{thebibliography}{10}
\providecommand{\url}[1]{#1}
\csname url@samestyle\endcsname
\providecommand{\newblock}{\relax}
\providecommand{\bibinfo}[2]{#2}
\providecommand{\BIBentrySTDinterwordspacing}{\spaceskip=0pt\relax}
\providecommand{\BIBentryALTinterwordstretchfactor}{4}
\providecommand{\BIBentryALTinterwordspacing}{\spaceskip=\fontdimen2\font plus
\BIBentryALTinterwordstretchfactor\fontdimen3\font minus
  \fontdimen4\font\relax}
\providecommand{\BIBforeignlanguage}[2]{{%
\expandafter\ifx\csname l@#1\endcsname\relax
\typeout{** WARNING: IEEEtran.bst: No hyphenation pattern has been}%
\typeout{** loaded for the language `#1'. Using the pattern for}%
\typeout{** the default language instead.}%
\else
\language=\csname l@#1\endcsname
\fi
#2}}
\providecommand{\BIBdecl}{\relax}
\BIBdecl

\bibitem{azadi2023make}
S.~Azadi, A.~Shah, T.~Hayes, D.~Parikh, and S.~Gupta, ``Make-an-animation:
  Large-scale text-conditional 3d human motion generation,'' in
  \emph{Proceedings of the IEEE/CVF International Conference on Computer
  Vision}, 2023, pp. 15\,039--15\,048.

\bibitem{mu2025smp}
Y.~Mu, Z.~Zhang, Y.~Shi, D.~Yang, M.~Matsumoto, K.~Imamura, G.~Tevet, C.~Guo,
  M.~Taylor, C.~Shu, P.~Xi, and X.~B. Peng, ``Smp: Reusable score-matching
  motion priors for physics-based character control,'' \emph{ACM Transactions
  on Graphics (Proceedings of SIGGRAPH 2026)}, vol.~45, no.~4, 2026.

\bibitem{tashakori2025flexmotion}
A.~Tashakori, A.~Tashakori, G.~Yang, Z.~J. Wang, and P.~Servati, ``Flexmotion:
  Lightweight, physics-aware, and controllable human motion generation,''
  \emph{arXiv preprint arXiv:2501.16778}, 2025.

\bibitem{ho2020denoising}
J.~Ho, A.~Jain, and P.~Abbeel, ``Denoising diffusion probabilistic models,''
  \emph{Advances in neural information processing systems}, vol.~33, pp.
  6840--6851, 2020.

\bibitem{lipman2023flow}
Y.~Lipman, R.~T.~Q. Chen, H.~Ben-Hamu, M.~Nickel, and M.~Le, ``Flow matching
  for generative modeling,'' in \emph{The Eleventh International Conference on
  Learning Representations}, 2023.

\bibitem{Ma2025Efficient}
Z.~Ma, Y.~Zhang, G.~Jia, L.~Zhao, Y.~Ma, M.~Ma, G.~Liu, K.~Zhang, N.~Ding,
  J.~Li, and B.~Zhou, ``Efficient diffusion models: A comprehensive survey from
  principles to practices,'' \emph{IEEE Transactions on Pattern Analysis and
  Machine Intelligence}, vol.~47, no.~9, pp. 7506--7525, 2025.

\bibitem{tevet2023human}
G.~Tevet, S.~Raab, B.~Gordon, Y.~Shafir, D.~Cohen-or, and A.~H. Bermano,
  ``Human motion diffusion model,'' in \emph{The Eleventh International
  Conference on Learning Representations}, 2023.

\bibitem{chen2023executing}
X.~Chen, B.~Jiang, W.~Liu, Z.~Huang, B.~Fu, T.~Chen, and G.~Yu, ``Executing
  your commands via motion diffusion in latent space,'' in \emph{Proceedings of
  the IEEE/CVF Conference on Computer Vision and Pattern Recognition}, 2023,
  pp. 18\,000--18\,010.

\bibitem{shen2025finextrolcontrollablemotiongeneration}
K.~Shen, B.~Wu, J.~Chen, X.~Wang, and L.~Shen, ``Finextrol: Controllable motion
  generation via fine-grained text,'' in \emph{Proceedings of the AAAI
  Conference on Artificial Intelligence}, vol.~40, no.~11, 2026, pp.
  8905--8913.

\bibitem{jiang2023motiongpt}
B.~Jiang, X.~Chen, W.~Liu, J.~Yu, G.~Yu, and T.~Chen, ``Motiongpt: Human motion
  as a foreign language,'' \emph{Advances in Neural Information Processing
  Systems}, vol.~36, pp. 20\,067--20\,079, 2023.

\bibitem{zhang2023generating}
J.~Zhang, Y.~Zhang, X.~Cun, Y.~Zhang, H.~Zhao, H.~Lu, X.~Shen, and Y.~Shan,
  ``Generating human motion from textual descriptions with discrete
  representations,'' in \emph{Proceedings of the IEEE/CVF conference on
  computer vision and pattern recognition}, 2023, pp. 14\,730--14\,740.

\bibitem{guo2023momask}
C.~Guo, Y.~Mu, M.~G. Javed, S.~Wang, and L.~Cheng, ``Momask: Generative masked
  modeling of 3d human motions,'' in \emph{Proceedings of the IEEE/CVF
  Conference on Computer Vision and Pattern Recognition}, 2024, pp. 1900--1910.

\bibitem{Zhu2024Human}
W.~Zhu, X.~Ma, D.~Ro, H.~Ci, J.~Zhang, J.~Shi, F.~Gao, Q.~Tian, and Y.~Wang,
  ``Human motion generation: A survey,'' \emph{IEEE Trans. Pattern Anal. Mach.
  Intell.}, vol.~46, no.~4, pp. 2430--2449, Apr. 2024.

\bibitem{zhang2022motiondiffuse}
M.~Zhang, Z.~Cai, L.~Pan, F.~Hong, X.~Guo, L.~Yang, and Z.~Liu,
  ``Motiondiffuse: Text-driven human motion generation with diffusion model,''
  \emph{IEEE Transactions on Pattern Analysis and Machine Intelligence},
  vol.~46, no.~6, pp. 4115--4128, 2024.

\bibitem{guo2022generating}
C.~Guo, S.~Zou, X.~Zuo, S.~Wang, W.~Ji, X.~Li, and L.~Cheng, ``Generating
  diverse and natural 3d human motions from text,'' in \emph{Proceedings of the
  IEEE/CVF Conference on Computer Vision and Pattern Recognition}, 2022, pp.
  5152--5161.

\bibitem{motioncritic2025}
\BIBentryALTinterwordspacing
H.~Wang, W.~Zhu, L.~Miao, Y.~Xu, F.~Gao, Q.~Tian, and Y.~Wang, ``Aligning human
  motion generation with human perceptions,'' in \emph{International Conference
  on Learning Representations}, 2025. [Online]. Available:
  \url{https://openreview.net/forum?id=QOHgjY5KDp}
\BIBentrySTDinterwordspacing

\bibitem{tan2024sopo}
X.~Tan, H.~Wang, X.~Geng, and P.~Zhou, ``Sopo: Text-to-motion generation using
  semi-online preference optimization,'' \emph{Advances in Neural Information
  Processing Systems}, vol.~38, pp. 67\,535--67\,568, 2025.

\bibitem{tan2026consistentrftreducingvisualhallucinations}
X.~Tan, J.~Liu, Y.~Fan, B.-B. Gao, X.~Jiang, X.~Chen, J.~Peng, C.~Wang,
  H.~Wang, and F.~Zheng, ``Consistentrft: Reducing visual hallucinations in
  flow-based reinforcement fine-tuning,'' \emph{arXiv preprint
  arXiv:2602.03425}, 2026.

\bibitem{HanReinDiffuse2025}
G.~Han, M.~Liang, J.~Tang, Y.~Cheng, W.~Liu, and S.~Huang, ``{ReinDiffuse:
  Crafting Physically Plausible Motions with Reinforced Diffusion Model},'' in
  \emph{2025 IEEE/CVF Winter Conference on Applications of Computer Vision
  (WACV)}.\hskip 1em plus 0.5em minus 0.4em\relax Los Alamitos, CA, USA: IEEE
  Computer Society, Mar. 2025, pp. 2218--2227.

\bibitem{schulman2017proximal}
J.~Schulman, F.~Wolski, P.~Dhariwal, A.~Radford, and O.~Klimov, ``Proximal
  policy optimization algorithms,'' \emph{arXiv preprint arXiv:1707.06347},
  2017.

\bibitem{kirstain2023pick}
Y.~Kirstain, A.~Polyak, U.~Singer, S.~Matiana, J.~Penna, and O.~Levy,
  ``Pick-a-pic: An open dataset of user preferences for text-to-image
  generation,'' \emph{Advances in Neural Information Processing Systems},
  vol.~36, pp. 36\,652--36\,663, 2023.

\bibitem{ziegler2019fine}
D.~M. Ziegler, N.~Stiennon, J.~Wu, T.~B. Brown, A.~Radford, D.~Amodei,
  P.~Christiano, and G.~Irving, ``Fine-tuning language models from human
  preferences,'' \emph{arXiv preprint arXiv:1909.08593}, 2019.

\bibitem{clark2024directly}
K.~Clark, P.~Vicol, K.~Swersky, and D.~J. Fleet, ``Directly fine-tuning
  diffusion models on differentiable rewards,'' in \emph{The Twelfth
  International Conference on Learning Representations}, 2024.

\bibitem{black2023training}
\BIBentryALTinterwordspacing
K.~Black, M.~Janner, Y.~Du, I.~Kostrikov, and S.~Levine, ``Training diffusion
  models with reinforcement learning,'' in \emph{The Twelfth International
  Conference on Learning Representations}, 2024. [Online]. Available:
  \url{https://openreview.net/forum?id=nMkdvj7BF8}
\BIBentrySTDinterwordspacing

\bibitem{xue2025dancegrpo}
Z.~Xue, J.~Wu, Y.~Gao, F.~Kong, L.~Zhu, M.~Chen, Z.~Liu, W.~Liu, Q.~Guo,
  W.~Huang \emph{et~al.}, ``Dancegrpo: Unleashing grpo on visual generation,''
  \emph{arXiv preprint arXiv:2505.07818}, 2025.

\bibitem{liu2025flow}
J.~Liu, G.~Liu, J.~Liang, Y.~Li, J.~Liu, X.~Wang, P.~Wan, D.~Zhang, and
  W.~Ouyang, ``{Flow-GRPO}: Training flow matching models via online {RL},''
  \emph{Advances in Neural Information Processing Systems}, vol.~38, pp.
  45\,650--45\,685, 2025.

\bibitem{Wallace_2024_CVPR}
B.~Wallace, M.~Dang, R.~Rafailov, L.~Zhou, A.~Lou, S.~Purushwalkam, S.~Ermon,
  C.~Xiong, S.~Joty, and N.~Naik, ``Diffusion model alignment using direct
  preference optimization,'' in \emph{Proceedings of the IEEE/CVF Conference on
  Computer Vision and Pattern Recognition (CVPR)}, June 2024, pp. 8228--8238.

\bibitem{petrovich2023tmr}
M.~Petrovich, M.~J. Black, and G.~Varol, ``Tmr: Text-to-motion retrieval using
  contrastive 3d human motion synthesis,'' in \emph{Proceedings of the IEEE/CVF
  International Conference on Computer Vision}, 2023, pp. 9488--9497.

\bibitem{weng2025}
W.~Weng, X.~Tan, J.~Wang, G.-S. Xie, P.~Zhou, and H.~Wang, ``Realign:
  Text-to-motion generation via step-aware reward-guided alignment,'' in
  \emph{Proceedings of the AAAI Conference on Artificial Intelligence},
  vol.~40, no.~13, 2026, pp. 10\,621--10\,629.

\bibitem{wang2025foundation}
H.~Wang, W.~Weng, J.~Wang, F.~Zhao, G.-S. Xie, X.~Geng, and L.~Wang,
  ``Foundation model for skeleton-based human action understanding,''
  \emph{IEEE Transactions on Pattern Analysis and Machine Intelligence},
  vol.~48, no.~1, pp. 47--61, 2026.

\bibitem{wang2025heterogeneous}
H.~Wang, X.~Ma, J.~Kuang, and J.~Gui, ``Heterogeneous skeleton-based action
  representation learning,'' in \emph{Proceedings of the Computer Vision and
  Pattern Recognition Conference}, 2025, pp. 19\,154--19\,164.

\bibitem{weng2025usdrl}
W.~Weng, H.~Wang, J.~Wang, L.~He, and G.-S. Xie, ``Usdrl: Unified
  skeleton-based dense representation learning with multi-grained feature
  decorrelation,'' in \emph{Proceedings of the AAAI Conference on Artificial
  Intelligence}, vol.~39, no.~8, 2025, pp. 8332--8340.

\bibitem{wu2025generalization}
Y.~Wu, Y.~Zhou, Z.~Ziheng, Y.~Peng, X.~Ye, X.~Hu, W.~Zhu, L.~Qi, M.-H. Yang,
  and X.~Yang, ``On the generalization of sft: A reinforcement learning
  perspective with reward rectification,'' \emph{arXiv preprint
  arXiv:2508.05629}, 2025.

\bibitem{meng2025absolute}
Z.~Meng, Z.~Han, X.~Peng, Y.~Xie, and H.~Jiang, ``Absolute coordinates make
  motion generation easy,'' \emph{arXiv preprint arXiv:2505.19377}, 2025.

\bibitem{zhao2024dartcontrol}
\BIBentryALTinterwordspacing
K.~Zhao, G.~Li, and S.~Tang, ``Dartcontrol: A diffusion-based autoregressive
  motion model for real-time text-driven motion control,'' in \emph{The
  Thirteenth International Conference on Learning Representations}, 2025.
  [Online]. Available: \url{https://openreview.net/forum?id=XNA3Mnnbvb}
\BIBentrySTDinterwordspacing

\bibitem{fan2025go}
K.~Fan, S.~Lu, M.~Dai, R.~Yu, L.~Xiao, Z.~Dou, J.~Dong, L.~Ma, and J.~Wang,
  ``Go to zero: Towards zero-shot motion generation with million-scale data,''
  in \emph{Proceedings of the IEEE/CVF International Conference on Computer
  Vision}, 2025, pp. 13\,336--13\,348.

\bibitem{hymotion2025}
{Tencent Hunyuan 3D Digital Human Team}, ``Hy-motion 1.0: Scaling flow matching
  models for text-to-motion generation,'' \emph{arXiv preprint
  arXiv:2512.23464}, 2025.

\bibitem{barquero2024seamless}
G.~Barquero, S.~Escalera, and C.~Palmero, ``Seamless human motion composition
  with blended positional encodings,'' in \emph{Proceedings of the IEEE/CVF
  Conference on Computer Vision and Pattern Recognition}, 2024, pp. 457--469.

\bibitem{prabhudesai2023aligning}
M.~Prabhudesai, A.~Goyal, D.~Pathak, and K.~Fragkiadaki, ``Aligning
  text-to-image diffusion models with reward backpropagation,'' \emph{arXiv
  preprint arXiv:2310.03739}, 2023.

\bibitem{wu2025drtune}
X.~Wu, Y.~Hao, M.~Zhang, K.~Sun, Z.~Huang, G.~Song, Y.~Liu, and H.~Li, ``Deep
  reward supervisions for tuning text-to-image diffusion models,'' in
  \emph{European Conference on Computer Vision}.\hskip 1em plus 0.5em minus
  0.4em\relax Cham: Springer Nature Switzerland, 2024, pp. 108--124.

\bibitem{motionlcm-v2}
\BIBentryALTinterwordspacing
W.~Dai, L.-H. Chen, Y.~Huo, J.~Wang, J.~Liu, B.~Dai, and Y.~Tang, ``Real-time
  controllable motion generation via latent consistency model,'' Technical
  report (MotionLCM-V2 / {MLD}++), 2024. [Online]. Available:
  \url{https://dai-wenxun.github.io/files/MotionLCM_V2.pdf}
\BIBentrySTDinterwordspacing

\bibitem{Dai2025}
W.~Dai, L.-H. Chen, J.~Wang, J.~Liu, B.~Dai, and Y.~Tang, ``Motionlcm:
  Real-time controllable motion generation via latent consistency model,'' in
  \emph{European Conference on Computer Vision}, A.~Leonardis, E.~Ricci,
  S.~Roth, O.~Russakovsky, T.~Sattler, and G.~Varol, Eds.\hskip 1em plus 0.5em
  minus 0.4em\relax Cham: Springer Nature Switzerland, 2024, pp. 390--408.

\bibitem{tan2026easytune}
X.~Tan, W.~Weng, H.~Lei, and H.~Wang, ``Easytune: Efficient step-aware
  fine-tuning for diffusion-based motion generation,'' in \emph{The Fourteenth
  International Conference on Learning Representations}, 2026.

\bibitem{holden2017phase}
D.~Holden, T.~Komura, and J.~Saito, ``Phase-functioned neural networks for
  character control,'' \emph{ACM Transactions on Graphics (TOG)}, vol.~36,
  no.~4, pp. 1--13, 2017.

\bibitem{gu2026bridgingsemantickinematicconditions}
\BIBentryALTinterwordspacing
C.~Gu, M.~Zhang, H.~Xie, Z.~Cai, L.~Yang, and Z.~Liu, ``Bridging semantic and
  kinematic conditions with diffusion-based discrete motion tokenizer,'' 2026.
  [Online]. Available: \url{https://arxiv.org/abs/2603.19227}
\BIBentrySTDinterwordspacing

\bibitem{cao2026crowdmogen}
Y.~Cao, X.~Guo, M.~Zhang, H.~Xie, C.~Gu, and Z.~Liu, ``{CrowdMoGen}:
  Event-driven collective human motion generation,'' \emph{International
  Journal of Computer Vision}, vol. 134, no.~1, p.~29, 2026.

\bibitem{meng2024rethinking}
Z.~Meng, Y.~Xie, X.~Peng, Z.~Han, and H.~Jiang, ``Rethinking diffusion for
  text-driven human motion generation: Redundant representations, evaluation,
  and masked autoregression,'' in \emph{Proceedings of the Computer Vision and
  Pattern Recognition Conference}, 2025, pp. 27\,859--27\,871.

\bibitem{martinez2017human}
J.~Martinez, M.~J. Black, and J.~Romero, ``On human motion prediction using
  recurrent neural networks,'' in \emph{Proceedings of the IEEE/CVF Conference
  on Computer Vision and Pattern Recognition (CVPR)}, 2017, pp. 2891--2900.

\bibitem{petrovich2021actor}
M.~Petrovich, M.~J. Black, and G.~Varol, ``Action-conditioned 3d human motion
  synthesis with transformer vae,'' in \emph{Proceedings of the IEEE/CVF
  International Conference on Computer Vision (ICCV)}, 2021, pp.
  10\,985--10\,995.

\bibitem{pavllo2018quaternet}
D.~Pavllo, D.~Grangier, and M.~Auli, ``Quaternet: A quaternion-based recurrent
  model for human motion,'' \emph{arXiv preprint arXiv:1805.06485}, 2018.

\bibitem{pavlakos2019expressive}
G.~Pavlakos, V.~Choutas, N.~Ghorbani, T.~Bolkart, A.~A. Osman, D.~Tzionas, and
  M.~J. Black, ``Expressive body capture: 3d hands, face, and body from a
  single image,'' in \emph{Proceedings of the IEEE/CVF conference on computer
  vision and pattern recognition}, 2019, pp. 10\,975--10\,985.

\bibitem{zhang2025primal}
Y.~Zhang, Y.~Feng, A.~Cseke, N.~Saini, N.~Bajandas, N.~Heron, and M.~J. Black,
  ``{PRIMAL}: Physically reactive and interactive motor model for avatar
  learning,'' in \emph{Proceedings of the IEEE/CVF International Conference on
  Computer Vision (ICCV)}, 2025, pp. 12\,725--12\,736.

\bibitem{ruizponce2026interact2ar}
P.~Ruiz-Ponce, S.~Escalera, J.~Garcia-Rodriguez, J.~Deng, and R.~A. Potamias,
  ``Interact2ar: Full-body human-human interaction generation via
  autoregressive diffusion models,'' in \emph{Proceedings of the IEEE/CVF
  Conference on Computer Vision and Pattern Recognition (CVPR)}, 2026, pp.
  23\,559--23\,569.

\bibitem{yuan2023physdiff}
Y.~Yuan, J.~Song, U.~Iqbal, A.~Vahdat, and J.~Kautz, ``Physdiff: Physics-guided
  human motion diffusion model,'' in \emph{Proceedings of the IEEE/CVF
  international conference on computer vision}, 2023, pp. 16\,010--16\,021.

\bibitem{tan2024frequency}
X.~Tan, H.~Wang, X.~Geng, and L.~Wang, ``Frequency-guided diffusion model with
  perturbation training for skeleton-based video anomaly detection,''
  \emph{arXiv preprint arXiv:2412.03044}, 2024.

\bibitem{NEURIPS2023_fc65fab8}
Y.~Fan, O.~Watkins, Y.~Du, H.~Liu, M.~Ryu, C.~Boutilier, P.~Abbeel,
  M.~Ghavamzadeh, K.~Lee, and K.~Lee, ``Dpok: Reinforcement learning for
  fine-tuning text-to-image diffusion models,'' in \emph{Advances in Neural
  Information Processing Systems}, A.~Oh, T.~Naumann, A.~Globerson, K.~Saenko,
  M.~Hardt, and S.~Levine, Eds., vol.~36.\hskip 1em plus 0.5em minus
  0.4em\relax Curran Associates, Inc., 2023, pp. 79\,858--79\,885.

\bibitem{zheng2025diffusionnft}
K.~Zheng, H.~Chen, H.~Ye, H.~Wang, Q.~Zhang, K.~Jiang, H.~Su, S.~Ermon, J.~Zhu,
  and M.-Y. Liu, ``Diffusionnft: Online diffusion reinforcement with forward
  process,'' \emph{arXiv preprint arXiv:2509.16117}, 2025.

\bibitem{kim2022diffusionclip}
G.~Kim, T.~Kwon, and J.~C. Ye, ``Diffusionclip: Text-guided diffusion models
  for robust image manipulation,'' in \emph{Proceedings of the IEEE/CVF
  conference on computer vision and pattern recognition}, 2022, pp. 2426--2435.

\bibitem{petrovich2022temos}
M.~Petrovich, M.~J. Black, and G.~Varol, ``Temos: Generating diverse human
  motions from textual descriptions,'' in \emph{European Conference on Computer
  Vision}.\hskip 1em plus 0.5em minus 0.4em\relax Springer, 2022, pp. 480--497.

\bibitem{xie2025dymo}
X.~Xie and D.~Gong, ``Dymo: Training-free diffusion model alignment with
  dynamic multi-objective scheduling,'' in \emph{Proceedings of the Computer
  Vision and Pattern Recognition Conference}, 2025, pp. 13\,220--13\,230.

\bibitem{zhou2025golden}
Z.~Zhou, S.~Shao, L.~Bai, S.~Zhang, Z.~Xu, B.~Han, and Z.~Xie, ``Golden noise
  for diffusion models: A learning framework,'' in \emph{Proceedings of the
  IEEE/CVF International Conference on Computer Vision}, 2025, pp.
  17\,688--17\,697.

\bibitem{xu2023imagereward}
J.~Xu, X.~Liu, Y.~Wu, Y.~Tong, Q.~Li, M.~Ding, J.~Tang, and Y.~Dong,
  ``Imagereward: Learning and evaluating human preferences for text-to-image
  generation,'' \emph{Advances in Neural Information Processing Systems},
  vol.~36, pp. 15\,903--15\,935, 2023.

\bibitem{lu2022dpm}
C.~Lu, Y.~Zhou, F.~Bao, J.~Chen, C.~Li, and J.~Zhu, ``Dpm-solver: A fast ode
  solver for diffusion probabilistic model sampling in around 10 steps,''
  \emph{Advances in neural information processing systems}, vol.~35, pp.
  5775--5787, 2022.

\bibitem{peng2018deepmimic}
X.~B. Peng, P.~Abbeel, S.~Levine, and M.~van~de Panne, ``{DeepMimic}:
  Example-guided deep reinforcement learning of physics-based character
  skills,'' \emph{ACM Trans. Graph.}, vol.~37, no.~4, pp. 143:1--143:14, 2018.

\bibitem{peng2022ase}
X.~B. Peng, Y.~Guo, L.~Halper, S.~Levine, and S.~Fidler, ``{ASE}: Large-scale
  reusable adversarial skill embeddings for physically simulated characters,''
  \emph{ACM Trans. Graph.}, vol.~41, no.~4, 2022.

\bibitem{li2025lamp}
Z.~Li, W.~Yuan, Y.~HE, L.~Qiu, S.~Zhu, X.~Gu, W.~Shen, Y.~Dong, Z.~Dong, and
  L.~T. Yang, ``La{MP}: Language-motion pretraining for motion generation,
  retrieval, and captioning,'' in \emph{The Thirteenth International Conference
  on Learning Representations}, 2025.

\bibitem{loper2015smpl}
\BIBentryALTinterwordspacing
M.~Loper, N.~Mahmood, J.~Romero, G.~Pons-Moll, and M.~J. Black, ``{SMPL}: a
  skinned multi-person linear model,'' \emph{ACM Trans. Graph.}, vol.~34,
  no.~6, Nov. 2015. [Online]. Available:
  \url{https://doi.org/10.1145/2816795.2818013}
\BIBentrySTDinterwordspacing

\bibitem{plappert2016kit}
M.~Plappert, C.~Mandery, and T.~Asfour, ``The kit motion-language dataset,''
  \emph{Big data}, vol.~4, no.~4, pp. 236--252, 2016.

\bibitem{kirk2024understanding}
R.~Kirk, I.~Mediratta, C.~Nalmpantis, J.~Luketina, E.~Hambro, E.~Grefenstette,
  and R.~Raileanu, ``Understanding the effects of {RLHF} on {LLM}
  generalisation and diversity,'' in \emph{International Conference on Learning
  Representations}, 2024.

\bibitem{tan2026perceptual}
X.~Tan, J.~Liu, B.-B. Gao, Y.~Fan, X.~Jiang, C.~Wang, H.~Wang, and F.~Zheng,
  ``When policy entropy constraint fails: Preserving diversity in flow-based
  {RLHF} via perceptual entropy,'' \emph{arXiv preprint arXiv:2605.12112},
  2026.

\bibitem{guo2022tm2t}
C.~Guo, X.~Zuo, S.~Wang, and L.~Cheng, ``Tm2t: Stochastic and tokenized
  modeling for the reciprocal generation of 3d human motions and texts,'' in
  \emph{European Conference on Computer Vision}.\hskip 1em plus 0.5em minus
  0.4em\relax Springer, 2022, pp. 580--597.

\bibitem{zhang2023remodiffuse}
M.~Zhang, X.~Guo, L.~Pan, Z.~Cai, F.~Hong, H.~Li, L.~Yang, and Z.~Liu,
  ``Remodiffuse: Retrieval-augmented motion diffusion model,'' in
  \emph{Proceedings of the IEEE/CVF International Conference on Computer
  Vision}, 2023, pp. 364--373.

\bibitem{zhong2023attt2m}
C.~Zhong, L.~Hu, Z.~Zhang, and S.~Xia, ``Attt2m: Text-driven human motion
  generation with multi-perspective attention mechanism,'' in \emph{Proceedings
  of the IEEE/CVF International Conference on Computer Vision}, 2023, pp.
  509--519.

\bibitem{motionmaba}
Z.~Zhang, A.~Liu, I.~Reid, R.~Hartley, B.~Zhuang, and H.~Tang, ``Motion mamba:
  Efficient and long sequence motion generation,'' in \emph{European Conference
  on Computer Vision}.\hskip 1em plus 0.5em minus 0.4em\relax Springer, 2024,
  pp. 265--282.

\bibitem{Huang2024CoMo}
Y.~Huang, W.~Wan, Y.~Yang, C.~Callison-Burch, M.~Yatskar, and L.~Liu, ``Como:
  Controllable motion generation through language guided pose code editing,''
  in \emph{Computer Vision – ECCV 2024: 18th European Conference, Milan,
  Italy, September 29–October 4, 2024, Proceedings, Part XXIX}.\hskip 1em
  plus 0.5em minus 0.4em\relax Berlin, Heidelberg: Springer-Verlag, 2024, p.
  180–196.

\bibitem{zou2024parco}
Q.~Zou, S.~Yuan, S.~Du, Y.~Wang, C.~Liu, Y.~Xu, J.~Chen, and X.~Ji, ``Parco:
  Part-coordinating text-to-motion synthesis,'' in \emph{European Conference on
  Computer Vision}.\hskip 1em plus 0.5em minus 0.4em\relax Springer, 2024, pp.
  126--143.

\bibitem{kong2023priority}
H.~Kong, K.~Gong, D.~Lian, M.~B. Mi, and X.~Wang, ``Priority-centric human
  motion generation in discrete latent space,'' in \emph{Proceedings of the
  IEEE/CVF International Conference on Computer Vision}, 2023, pp.
  14\,806--14\,816.

\bibitem{wang2023fg}
Y.~Wang, Z.~Leng, F.~W. Li, S.-C. Wu, and X.~Liang, ``Fg-t2m: Fine-grained
  text-driven human motion generation via diffusion model,'' in
  \emph{Proceedings of the IEEE/CVF International Conference on Computer
  Vision}, 2023, pp. 22\,035--22\,044.

\bibitem{lin2025quest}
J.~Lin, R.~Wang, J.~Lu, Z.~Huang, G.~Song, A.~Zeng, X.~Liu, C.~Wei, W.~Yin,
  Q.~Sun, Z.~Cai, L.~Yang, and Z.~Liu, ``The quest for generalizable motion
  generation: Data, model, and evaluation,'' \emph{arXiv preprint
  arXiv:2510.26794}, 2025.

\end{thebibliography}

\clearpage
\appendices
    \section{Proof} \label{supp:the}
    \subsection{Proof of {Corollary} 1} \label{supp:the1}
    Recall the {Corollary} 1.
    \begin{corollary*}
        Given the reverse process in Eq.~\meqref{eq:reverse_process}, $\mathbf{x}_{t-1}^\theta = \pi_\theta(\mathbf{x}_t^\theta, t, c)$, the gradient w.r.t diffusion model $\theta$, denoted as $\tfrac{\partial \mathbf{x}^\theta_{t-1}}{\partial \theta}$, can be expressed as:
        \begin{equation}
            \eqfit{{\frac{\partial \mathbf{x}^\theta_{t-1}}{\partial \theta}} =  \underbrace{\frac{\partial \pi_{{\theta}}(\mathbf{x}^{\theta}_t, t, c)}{\partial {\theta}}}_{\textnormal{direct term}} +
            \underbrace{\frac{\partial \pi_\theta(\mathbf{x}^{{\theta}}_t, t, c)}{\partial \mathbf{x}^{{\theta}}_t} \cdot {\frac{\partial \mathbf{x}^{{\theta}}_t}{\partial {\theta}}}}_{\textnormal{indirect term}}.}
            \label{supp:eq:recursive_gradient}
        \end{equation}
    \end{corollary*}

    \begin{proof}
        Let $u(\theta) = \mathbf{x}_{t}^\theta$, $v(\theta) = \theta$, and define $\tilde{F}(\theta) = F\bigl(u(\theta),\, v(\theta)\bigr) = \pi_{v(\theta)}\bigl(u(\theta), t, c\bigr)$. By the multivariate chain rule, the total derivative of $\tilde{F}$ w.r.t.\ $\theta$ is:
        \begin{equation}
            \eqfit{\frac{\partial \tilde{F}(\theta)}{\partial \theta} = \frac{\partial F}{\partial v}\cdot \frac{\partial v}{\partial \theta} + \frac{\partial F}{\partial u} \cdot \frac{\partial u}{\partial \theta}.}
        \end{equation}
         The first term $\frac{\partial v}{\partial \theta}$ can be expressed as:
        \begin{equation}
            \eqfit{\frac{\partial v}{\partial \theta} = \frac{\partial \theta}{\partial \theta} = I,}
        \end{equation}
        and the second term $\frac{\partial u}{\partial \theta}$ can be expressed as:
         \begin{equation}
            \eqfit{\frac{\partial u}{\partial \theta} = \frac{\partial \mathbf{x}_{t}^\theta}{\partial \theta}.}
         \end{equation}

        Hence, we can rewrite the equation as:
        \begin{equation}
            \eqfit{\frac{\partial \tilde{F}(\theta)}{\partial \theta} = \frac{\partial F}{\partial v} + \frac{\partial F}{\partial u} \cdot \frac{\partial \mathbf{x}_{t}^\theta}{\partial \theta}.}
        \end{equation}
        Furthermore, we substitute $F(u,v)$ with $\pi_{\theta}(\mathbf{x}_{t}^\theta, t, c)$, and thus the relationship described in Eq.~\eqref{supp:eq:recursive_gradient} holds:
        \begin{equation*}
            \eqfit{{\frac{\partial \mathbf{x}^\theta_{t-1}}{\partial \theta}} = \frac{\partial \pi_{{\theta}}(\mathbf{x}^{{\theta}}_t, t, c)}{\partial {\theta}} = \underbrace{\frac{\partial \pi_{{\theta}}(\mathbf{x}^{\theta}_t, t, c)}{\partial {\theta}}}_{\textnormal{direct term}} +
            \underbrace{\frac{\partial \pi_\theta(\mathbf{x}^{{\theta}}_t, t, c)}{\partial \mathbf{x}^{{\theta}}_t} \cdot {\frac{\partial \mathbf{x}^{{\theta}}_t}{\partial {\theta}}}}_{\textnormal{indirect term}}.}
        \end{equation*}

        The proof is completed.
    \end{proof}

    \subsection{Proof of Eq. (5)}\label{supp:eq5}
    \begin{proof}
        Given a diffusion model $\epsilon_\theta$, and a reward model $\mathcal{R}_\phi$, the diffusion model is fine-tuned by maximizing the differentiable reward value:
        \begin{equation}
            \eqfit{\frac{\partial \mathcal{L}(\theta)}{\partial \theta} = -\mathbb{E}_{\cdot} \left[ \frac{\partial \mathcal{R}_\phi(\mathbf{x}_0^\theta, c)}{\partial \mathbf{x}_0^\theta} \cdot \frac{\partial \mathbf{x}_0^\theta}{\partial \theta} \right].}
        \end{equation}
        where $\mathbb{E}_{\cdot}$ denotes $\mathbb{E}_{c \sim \mathbb{D}_\mathrm{T}, \mathbf{x}_0^\theta \sim \pi_\theta(\cdot|c)}$, and $\pi_\theta$ is the reverse process defined in Eq.~\meqref{eq:reverse_process}.

        Here, we introduce Corollary~\mref{thm:t1} to compute $\frac{\partial \mathbf{x}_0^\theta}{\partial \theta}$, and thus we have:
        \begin{equation*}
            \eqfit{\begin{aligned}
                \frac{\partial \mathcal{L}(\theta)}{\partial \theta} &= - \mathbb{E}_{\cdot}\, \frac{\partial \mathcal{R}_\phi (\mathbf{x}^\theta_0, c)}{\partial \mathbf{x}^\theta_0} \cdot {\color{blue}\underbrace{\frac{\partial \mathbf{x}^\theta_0}{\partial \theta}}_{\text{expand}}} \\
                &= - \mathbb{E}_{\cdot}\, \frac{\partial \mathcal{R}_\phi (\mathbf{x}^\theta_0, c)}{\partial \mathbf{x}^\theta_0} \cdot \left( \frac{\partial \pi_\theta (\mathbf{x}^\theta_1)}{\partial \theta} + \frac{\partial \pi_\theta(\mathbf{x}^\theta_1)}{\partial \mathbf{x}^\theta_1} \cdot {\color{blue}\underbrace{\frac{\partial \mathbf{x}^\theta_1}{\partial \theta}}_{\text{expand}}} \right)\\
                &= - \mathbb{E}_{\cdot}\, \frac{\partial \mathcal{R}_\phi (\mathbf{x}^\theta_0, c)}{\partial \mathbf{x}^\theta_0} \cdot \bigg( \frac{\partial \pi_\theta (\mathbf{x}^\theta_1)}{\partial \theta} + \frac{\partial \pi_\theta(\mathbf{x}^\theta_1)}{\partial \mathbf{x}^\theta_1} \cdot \frac{\partial \pi_\theta(\mathbf{x}_2^\theta)}{\partial \theta} \\
                &\quad + \frac{\partial \pi_\theta(\mathbf{x}^\theta_1)}{\partial \mathbf{x}^\theta_1} \cdot \frac{\partial \pi_\theta(\mathbf{x}^\theta_2)}{\partial \mathbf{x}^\theta_2} \cdot {\color{blue}\underbrace{\frac{\partial \mathbf{x}^\theta_2}{\partial \theta}}_{\text{expand}}} \bigg)\\
                &= \cdots\\
                &= - \mathbb{E}_{\cdot}\, \frac{\partial \mathcal{R}_\phi (\mathbf{x}^\theta_0, c)}{\partial \mathbf{x}^\theta_0} \cdot \left( \sum_{t=1}^{T} \left( \prod_{s=1}^{t-1} \frac{\partial \pi_\theta (\mathbf{x}^\theta_s)}{\partial \mathbf{x^\theta_s}} \right) \cdot \frac{\partial \pi_\theta (\mathbf{x}_t^\theta)}{\partial \theta} \right).
            \end{aligned}}
        \end{equation*}
        The proof is completed.
    \end{proof}

\section{Additional Experimental Results}
\label{sec:supp_exp}

\subsection{Generation and Retrieval on KIT-ML}
\label{sec:kit}

To evaluate cross-dataset generalization, we conduct text-to-motion generation and text-motion retrieval experiments on KIT-ML~\cite{plappert2016kit}. For generation, we evaluate EasyTune~\cite{tan2026easytune} and MotionRFT across kinematic models such as MDM~\cite{tevet2023human} and MotionDiffuse~\cite{zhang2022motiondiffuse}, joint models such as ACMDM~\cite{meng2025absolute}, and rotation models such as zero-shot HY-Motion~\cite{hymotion2025}. For retrieval, we benchmark MotionReward against representative baselines including TMR~\cite{petrovich2023tmr} and ReAlign~\cite{weng2025}.

As shown in Tab.~\ref{tab:kit_gen}, fine-tuning with EasyTune and MotionRFT consistently improves generative quality across all representations, reducing FID from $0.513$ to $0.269$ on MDM and from $0.237$ to $0.188$ on ACMDM, with notable gains in zero-shot transfer on HY-Motion. Diversity on KIT-ML stays close to the real motion. On MDM, Diversity moves from $10.73$ to $11.27$ under EasyTune and to $11.01$ under MotionRFT, against the real value $11.08$. MultiModality drops under reinforcement fine-tuning, consistent with stronger per-prompt alignment. For cross-modal retrieval in Tab.~\ref{tab:kit_ret}, MotionReward achieves the highest recall across both text-to-motion and motion-to-text tasks, raising average recall from $79.61$ to $82.14$ over ReAlign. These results suggest that our unified reward framework and reinforcement fine-tuning may generalize effectively to unseen datasets across diverse motion representations.

\begin{table*}[!t]
\centering
\caption{\textbf{Text-to-motion generation on KIT-ML.}
MDM results follow the author-corrected evaluation after the official bug fix, as requested by the MDM authors for future publications: \url{https://github.com/GuyTevet/motion-diffusion-model/issues/211}.
Joint-representation results are reproduced from ACMDM~\cite{meng2025absolute}.
Ground-truth Diversity is $11.08$ on kinematic features, $14.479$ on joint features, and $16.694$ on rotation features.
HY-Motion is evaluated zero-shot on KIT-ML.}
\label{tab:kit_gen}
\renewcommand{\arraystretch}{1.12}
\setlength{\tabcolsep}{5pt}
{\footnotesize
\begin{adjustbox}{max width=\textwidth}
\begin{tabular}{lccccccc}
\toprule
\multirow{2}{*}{Method} & \multicolumn{3}{c}{R Precision $\uparrow$} & \multirow{2}{*}{FID $\downarrow$} & \multirow{2}{*}{MM Dist $\downarrow$} & \multirow{2}{*}{Diversity $\rightarrow$} & \multirow{2}{*}{MModality $\uparrow$} \\
\cmidrule(lr){2-4}
~ & Top 1 & Top 2 & Top 3 & & & & \\
\midrule
\multicolumn{8}{c}{\textit{Kinematic Representation}} \\
\midrule
TM2T~\cite{guo2022tm2t} & 0.280 & 0.463 & 0.587 & 3.599 & 4.591 & 9.473 & 3.292 \\
T2M~\cite{guo2022generating} & 0.361 & 0.559 & 0.681 & 3.022 & 3.488 & 10.72 & 2.052 \\
M2DM~\cite{kong2023priority} & 0.416 & 0.628 & 0.743 & 0.515 & 3.015 & 11.42 & \textbf{3.325} \\
T2M-GPT~\cite{zhang2023generating} & 0.416 & 0.627 & 0.745 & 0.514 & 3.007 & 10.92 & 1.570 \\
Fg-T2M~\cite{wang2023fg} & 0.418 & 0.626 & 0.745 & 0.571 & 3.114 & 10.93 & 1.019 \\
AttT2M~\cite{zhong2023attt2m} & 0.413 & 0.632 & 0.751 & 0.870 & 3.039 & 10.96 & 2.281 \\
MotionMamba~\cite{motionmaba} & 0.419 & 0.645 & 0.765 & 0.307 & 3.021 & 11.02 & 1.678 \\
CoMo~\cite{Huang2024CoMo} & 0.422 & 0.638 & 0.765 & 0.332 & 2.873 & 10.95 & 1.249 \\
ParCo~\cite{zou2024parco} & 0.430 & 0.649 & 0.772 & 0.453 & 2.820 & 10.95 & 1.245 \\
\midrule
MDM~\cite{tevet2023human} & 0.404 & 0.606 & 0.731 & 0.513 & 3.096 & 10.73 & 1.806 \\
\rowcolor{gray!20} w/ EasyTune~\cite{tan2026easytune} & 0.442\,\gain{(+9.4\%)} & 0.655\,\gain{(+8.1\%)} & 0.773\,\gain{(+5.7\%)} & 0.284\,\gain{(+44.6\%)} & 2.755\,\gain{(+11.0\%)} & 11.27\,\gain{(+0.16)} & 1.528\,\drop{($-$0.278)} \\
\rowcolor{gray!20} w/ MotionRFT (Ours) & \textbf{0.450}\,\gain{(+11.4\%)} & \textbf{0.666}\,\gain{(+9.9\%)} & 0.780\,\gain{(+6.7\%)} & \textbf{0.269}\,\gain{(+47.6\%)} & \textbf{2.687}\,\gain{(+13.2\%)} & 11.01\,\gain{(+0.28)} & 1.534\,\drop{($-$0.272)} \\[0.3em]
MotionDiffuse~\cite{zhang2022motiondiffuse} & 0.417 & 0.621 & 0.739 & 1.954 & 2.958 & \textbf{11.10} & 0.730 \\
\rowcolor{gray!20} w/ EasyTune~\cite{tan2026easytune} & 0.438\,\gain{(+5.0\%)} & 0.649\,\gain{(+4.5\%)} & 0.777\,\gain{(+5.1\%)} & 1.719\,\gain{(+12.0\%)} & 2.892\,\gain{(+2.2\%)} & 10.63\,\drop{($-$0.43)} & 0.633\,\drop{($-$0.097)} \\
\rowcolor{gray!20} w/ MotionRFT (Ours) & 0.441\,\gain{(+5.8\%)} & 0.654\,\gain{(+5.3\%)} & \textbf{0.783}\,\gain{(+6.0\%)} & 1.684\,\gain{(+13.8\%)} & 2.877\,\gain{(+2.7\%)} & 11.03\,\drop{($-$0.03)} & 0.642\,\drop{($-$0.088)} \\
\midrule
\multicolumn{8}{c}{\textit{Joint Representation}} \\
\midrule
MDM~\cite{tevet2023human} & 0.333 & 0.561 & 0.689 & 0.585 & 4.002 & 13.601 & 1.681 \\
MotionDiffuse~\cite{zhang2022motiondiffuse} & 0.344 & 0.536 & 0.658 & 3.845 & 4.167 & 13.757 & 1.774 \\
ReMoDiffuse~\cite{zhang2023remodiffuse} & 0.356 & 0.572 & 0.706 & 1.725 & 3.735 & 13.894 & \textbf{1.928} \\
MARDM-$\epsilon$~\cite{meng2024rethinking} & 0.375 & 0.597 & 0.739 & 0.340 & 3.489 & 14.058 & 1.479 \\
MARDM-$v$~\cite{meng2024rethinking} & 0.387 & 0.610 & 0.749 & 0.242 & 3.374 & 14.169 & 1.312 \\
\midrule
ACMDM~\cite{meng2025absolute} & 0.391 & 0.615 & 0.752 & 0.237 & 3.368 & 14.257 & 1.267 \\
\rowcolor{gray!20} w/ MotionRFT (Ours) & \textbf{0.400}\,\gain{(+2.3\%)} & \textbf{0.647}\,\gain{(+5.2\%)} & \textbf{0.775}\,\gain{(+3.1\%)} & \textbf{0.188}\,\gain{(+20.7\%)} & \textbf{3.267}\,\gain{(+3.0\%)} & \textbf{14.396}\,\gain{(+0.14)} & 1.077\,\drop{($-$0.190)} \\
\midrule
\multicolumn{8}{c}{\textit{Rotation Representation}} \\
\midrule
HY-Motion-lite~\cite{hymotion2025} & 0.187 & 0.277 & 0.333 & 152.81 & 15.104 & 15.709 & 4.322 \\
\rowcolor{gray!20} w/ MotionRFT (Ours) & 0.205\,\gain{(+9.6\%)} & 0.301\,\gain{(+8.7\%)} & 0.356\,\gain{(+6.9\%)} & 141.50\,\gain{(+7.4\%)} & 14.168\,\gain{(+6.2\%)} & 15.746\,\gain{(+0.04)} & 4.019\,\drop{($-$0.303)} \\
HY-Motion~\cite{hymotion2025} & 0.228 & 0.317 & 0.370 & 134.22 & 14.163 & 16.288 & \textbf{4.985} \\
\rowcolor{gray!20} w/ MotionRFT (Ours) & \textbf{0.259}\,\gain{(+13.6\%)} & \textbf{0.350}\,\gain{(+10.4\%)} & \textbf{0.403}\,\gain{(+8.9\%)} & \textbf{107.91}\,\gain{(+19.6\%)} & \textbf{12.704}\,\gain{(+10.3\%)} & \textbf{16.342}\,\gain{(+0.05)} & 4.493\,\drop{($-$0.492)} \\
\bottomrule
\end{tabular}
\end{adjustbox}
}
\end{table*}

\begin{table*}[!t]
\centering
\caption{\textbf{Text-motion retrieval on KIT-ML.} Noise indicates whether the method handles noisy motion from the denoising process.}
\label{tab:kit_ret}
\renewcommand{\arraystretch}{1.12}
\setlength{\tabcolsep}{5pt}
{\footnotesize
\begin{tabular}{l c ccccc ccccc c}
\toprule
\multirow{2}{*}{Methods} & \multirow{2}{*}{Noise} & \multicolumn{5}{c}{Text-Motion Retrieval $\uparrow$} & \multicolumn{5}{c}{Motion-Text Retrieval $\uparrow$} & \multirow{2}{*}{Avg. $\uparrow$} \\
\cmidrule(lr){3-7} \cmidrule(lr){8-12}
& & R@1 & R@2 & R@3 & R@5 & R@10 & R@1 & R@2 & R@3 & R@5 & R@10 & \\
\midrule
TEMOS~\cite{petrovich2022temos} & \ding{55} & 43.88 & 58.25 & 67.00 & 74.00 & 84.75 & 41.88 & 55.88 & 65.62 & 75.25 & 85.75 & 65.23 \\
T2M~\cite{guo2022generating} & \ding{55} & 42.25 & 62.62 & 75.12 & 87.50 & 96.12 & 39.75 & 62.75 & 73.62 & 86.88 & 95.88 & 72.25 \\
TMR~\cite{petrovich2023tmr} & \ding{55} & 49.25 & 69.75 & 78.25 & 87.88 & 95.00 & 50.12 & 67.12 & 76.88 & 88.88 & 94.75 & 75.79 \\
LaMP~\cite{li2025lamp} & \ding{55} & 52.50 & 74.80 & 84.70 & 92.70 & 97.60 & 54.00 & 75.30 & 84.40 & 92.20 & 97.60 & 80.58 \\
ReAlign~\cite{weng2025} & \ding{51} & 52.88 & 71.62 & 83.00 & 91.25 & 97.62 & 53.00 & 72.88 & 84.38 & 92.62 & 96.88 & 79.61 \\
\rowcolor{gray!20} MotionReward (Ours) & \ding{51} & \textbf{54.12} & \textbf{75.50} & \textbf{86.50} & \textbf{94.50} & \textbf{98.88} & \textbf{55.62} & \textbf{76.12} & \textbf{87.38} & \textbf{94.50} & \textbf{98.25} & \textbf{82.14} \\
\bottomrule
\end{tabular}}
\end{table*}

\subsection{HY-Motion Evaluation}
\label{sec:hy_eval}

Because HY-Motion~\cite{hymotion2025} outputs rotation parameters, the original kinematic extractor of Guo et al.~\cite{guo2022generating} cannot be applied directly. While the rotation block of Tab.~\mref{tab:sota_humanmld3d} in the main paper uses a separately trained $135$-d rotation evaluator following the Guo et al.~\cite{guo2022generating} protocol, here we also evaluate generation with a retrieval evaluator following the ReAlign architecture~\cite{weng2025} and the released SSAE protocol~\cite{hymotion2025}. For SSAE, we use Qwen3-VL-30B-A3B-Instruct as the VLM judge under greedy decoding; motions are rendered at $5$ fps with at most $32$ frames.

As reported in Tab.~\ref{tab:hy_realign}, MotionRFT improves both Lite and Full model weights under the ReAlign-based evaluator, raising Top-1 R-Precision from $0.554$ to $0.601$ on Lite and from $0.563$ to $0.634$ on Full while consistently reducing FID. As shown in Tab.~\ref{tab:hy_ssae}, we evaluate the released Lite and Full model weights with and without MotionRFT. Full outperforms Lite across all categories, improving the overall score from $52.59\%$ to $53.98\%$. MotionRFT further improves the overall score to $53.15\%$ for Lite and $54.42\%$ for Full. Together, these evaluations provide supporting evidence under multiple protocols, including vision-language-model scoring.

\begin{table*}[!t]
\centering
\caption{\textbf{Text-to-motion generation on HumanML3D using the original ReAlign evaluator.} A retrieval evaluator following the ReAlign architecture~\cite{weng2025}.}
\label{tab:hy_realign}
\renewcommand{\arraystretch}{1.2}
\setlength{\tabcolsep}{5pt}
{\footnotesize
\begin{adjustbox}{max width=\textwidth}
    \begin{tabular}{lccccccc}
    \toprule
\multirow{2}{*}{Method} & \multicolumn{3}{c}{R Precision $\uparrow$} & \multirow{2}{*}{FID $\downarrow$} & \multirow{2}{*}{MM Dist $\downarrow$} & \multirow{2}{*}{Diversity $\rightarrow$} & \multirow{2}{*}{MModality $\uparrow$} \\
\cmidrule(lr){2-4}
~ & Top 1 & Top 2 & Top 3 & & & & \\
\midrule
Real & $0.588$ & $0.749$ & $0.820$ & $0.000$ & $0.763$ & $1.398$ & -- \\
\midrule
HY-Motion-lite \cite{hymotion2025} (Base Model) & $0.554^{\pm 0.003}$ & $0.718^{\pm 0.004}$ & $0.797^{\pm 0.002}$ & $0.067^{\pm 0.000}$ & $0.831^{\pm 0.002}$ & $1.356^{\pm 0.011}$ & $0.428^{\pm0.064}$ \\
\rowcolor{gray!20} w/ MotionRFT (Ours) &
\underline{0.601$^{\pm 0.003}$}\,\gain{(+8.5\%)} &
\underline{0.772$^{\pm 0.003}$}\,\gain{(+7.5\%)} &
\underline{0.843$^{\pm 0.003}$}\,\gain{(+5.8\%)} &
\underline{0.061$^{\pm 0.000}$}\,\gain{(+9.0\%)} &
\underline{0.772$^{\pm 0.003}$}\,\gain{(+7.1\%)} &
\underline{1.375$^{\pm 0.015}$}\,\gain{(+0.02)} &
$0.395^{\pm0.064}$\,\drop{($-$0.033)} \\
HY-Motion \cite{hymotion2025} (Base Model) & $0.563^{\pm 0.002}$ & $0.732^{\pm 0.003}$ & $0.808^{\pm 0.002}$ & $0.073^{\pm 0.000}$ & $0.818^{\pm 0.002}$ & $1.349^{\pm 0.013}$ & $\mathbf{0.433^{\pm0.061}}$ \\
\rowcolor{gray!20} w/ MotionRFT (Ours) &
\textbf{0.634$^{\pm 0.002}$}\,\gain{(+12.6\%)} &
\textbf{0.803$^{\pm 0.002}$}\,\gain{(+9.7\%)} &
\textbf{0.874$^{\pm 0.003}$}\,\gain{(+8.2\%)} &
\textbf{0.056$^{\pm 0.000}$}\,\gain{(+23.3\%)} &
\textbf{0.726$^{\pm 0.002}$}\,\gain{(+11.2\%)} &
\textbf{1.380$^{\pm 0.012}$}\,\gain{(+0.03)} &
$0.384^{\pm0.058}$\,\drop{($-$0.049)} \\
\bottomrule
\end{tabular}
\end{adjustbox}
}
\end{table*}

\begin{table*}[!t]
\centering
\caption{\textbf{SSAE~\cite{hymotion2025} results on the released HY-Motion model weights with and without MotionRFT.}}
\label{tab:hy_ssae}
\renewcommand{\arraystretch}{1.2}
\setlength{\tabcolsep}{4.5pt}
{\footnotesize
\resizebox{\textwidth}{!}{%
\begin{tabular}{lccccccc}
\toprule
Method
  & Locomotion
  & Sports \& Athletics
  & Fitness \& Outdoor Activities
  & Daily Activities
  & Social Interactions \& Leisure
  & Game Character Actions
  & Overall \\
\midrule
HY-Motion-lite & 55.89 & 54.90 & 50.84 & 47.06 & 50.82 & 56.53 & 52.59\% \\
\rowcolor{gray!20} w/ MotionRFT (Ours) & 56.37 & 55.52 & 51.02 & 47.90 & 51.36 & 57.43 & 53.15\% \\
HY-Motion      & 56.95 & 56.35 & 51.46 & 49.39 & 52.34 & 58.07 & 53.98\% \\
\rowcolor{gray!20} w/ MotionRFT (Ours) & \textbf{57.43} & \textbf{56.94} & \textbf{51.61} & \textbf{49.77} & \textbf{52.77} & \textbf{58.62} & \textbf{54.42\%} \\
\bottomrule
\end{tabular}}}
\end{table*}

\subsection{All-set Semantic Retrieval on HumanML3D}
\label{sec:allset}

To evaluate MotionReward under more challenging candidate pools beyond the standard small-batch protocol with gallery size 32, we conduct all-set semantic retrieval on the full HumanML3D test set following TMR protocol~b~\cite{petrovich2023tmr}. Each query is ranked against the entire test gallery of over 4{,}000 candidates, where a retrieval is counted as correct if candidate texts have a cosine similarity of at least $0.95$ to avoid penalizing valid paraphrases.

As reported in Tab.~\ref{tab:retrieval_threshold}, while absolute recalls drop due to the substantially expanded candidate pool, such as TMR text-to-motion R@1 falling from $67.16$ to $11.60$, MotionReward consistently achieves the highest retrieval accuracy across kinematic, joint, and rotation representations in both retrieval directions. These results suggest that our unified projection and cross-representation alignment may learn robust, representation-invariant semantic embeddings that scale effectively to large gallery sizes.

\begin{table*}[!t]
\centering
\caption{\textbf{All-set text-motion retrieval on HumanML3D under TMR~\cite{petrovich2023tmr} protocol~(b).}
$^*$FlowMDM is adapted to HumanML3D.
$^\dagger$Trained on joint or rotation features.}
\label{tab:retrieval_threshold}
\renewcommand{\arraystretch}{1.12}
\setlength{\tabcolsep}{3.5pt}
{\footnotesize
\begin{adjustbox}{max width=\textwidth}
\begin{tabular}{l|c|ccc|ccccc|ccccc}
\toprule
\multirow{2}{*}{Methods} & \multirow{2}{*}{Noise} & \multicolumn{3}{c|}{Rep.} & \multicolumn{5}{c|}{Text-Motion Retrieval} & \multicolumn{5}{c}{Motion-Text Retrieval} \\
\cmidrule(lr){3-5}
\cmidrule(lr){6-10}
\cmidrule(lr){11-15}
& & K & J & R & R@1$\uparrow$ & R@2$\uparrow$ & R@3$\uparrow$ & R@5$\uparrow$ & R@10$\uparrow$ & R@1$\uparrow$ & R@2$\uparrow$ & R@3$\uparrow$ & R@5$\uparrow$ & R@10$\uparrow$ \\
\midrule
\multicolumn{15}{c}{\textit{Evaluation with Kinematic Representation}} \\
\midrule
TEMOS~\cite{petrovich2022temos} & \ding{55} & \ding{51} & \ding{55} & \ding{55} & $5.21$ & $8.22$ & $11.14$ & $15.09$ & $22.12$ & $5.48$ & $6.19$ & $9.00$ & $12.01$ & $17.10$ \\
T2M~\cite{guo2022generating} & \ding{55} & \ding{51} & \ding{55} & \ding{55} & $5.30$ & $7.83$ & $10.75$ & $14.59$ & $22.51$ & $4.95$ & $5.68$ & $8.93$ & $11.64$ & $16.94$ \\
TMR~\cite{petrovich2023tmr} & \ding{55} & \ding{51} & \ding{55} & \ding{55} & $11.60$ & $15.39$ & $20.50$ & $27.72$ & $38.52$ & $13.20$ & $15.73$ & $22.03$ & $27.65$ & $37.63$ \\
ReAlign~\cite{weng2025} & \ding{51} & \ding{51} & \ding{55} & \ding{55} & $11.92$ & $15.89$ & $21.10$ & $28.54$ & $39.73$ & $14.20$ & $17.19$ & $23.97$ & $29.77$ & $40.57$ \\
\rowcolor{gray!20} MotionReward (Ours) & \ding{51} & \ding{51} & \ding{51} & \ding{51} & $\mathbf{12.67}$ & $\mathbf{17.08}$ & $\mathbf{22.99}$ & $\mathbf{31.23}$ & $\mathbf{43.06}$ & $\mathbf{14.61}$ & $\mathbf{17.90}$ & $\mathbf{25.43}$ & $\mathbf{31.78}$ & $\mathbf{43.06}$ \\
\midrule
\multicolumn{15}{c}{\textit{Evaluation with Joint Representation}} \\
\midrule
ACMDM-Eval.~\cite{meng2025absolute} & \ding{55} & \ding{55} & \ding{51} & \ding{55} & $5.27$ & $8.04$ & $11.10$ & $14.93$ & $23.40$ & $4.93$ & $5.66$ & $8.97$ & $11.71$ & $17.21$ \\
TMR$^\dagger$~\cite{petrovich2023tmr} & \ding{55} & \ding{55} & \ding{51} & \ding{55} & $10.80$ & $14.32$ & $19.09$ & $25.80$ & $35.84$ & $11.62$ & $13.47$ & $19.25$ & $24.52$ & $33.04$ \\
ReAlign$^\dagger$~\cite{weng2025} & \ding{51} & \ding{55} & \ding{51} & \ding{55} & $11.12$ & $14.84$ & $19.68$ & $26.64$ & $37.08$ & $12.40$ & $14.57$ & $20.68$ & $26.26$ & $35.46$ \\
\rowcolor{gray!20} MotionReward (Ours) & \ding{51} & \ding{51} & \ding{51} & \ding{51} & $\mathbf{11.42}$ & $\mathbf{15.30}$ & $\mathbf{20.41}$ & $\mathbf{27.67}$ & $\mathbf{38.33}$ & $\mathbf{12.85}$ & $\mathbf{15.34}$ & $\mathbf{21.94}$ & $\mathbf{27.28}$ & $\mathbf{37.10}$ \\
\midrule
\multicolumn{15}{c}{\textit{Evaluation with Rotation Representation}} \\
\midrule
FlowMDM-Eval.$^*$~\cite{barquero2024seamless} & \ding{55} & \ding{55} & \ding{55} & \ding{51} & $5.27$ & $8.17$ & $11.07$ & $15.37$ & $24.61$ & $5.07$ & $5.78$ & $9.06$ & $12.10$ & $18.01$ \\
TMR$^\dagger$~\cite{petrovich2023tmr} & \ding{55} & \ding{55} & \ding{55} & \ding{51} & $10.48$ & $13.88$ & $18.52$ & $25.02$ & $34.77$ & $11.37$ & $13.22$ & $18.84$ & $24.25$ & $32.51$ \\
ReAlign$^\dagger$~\cite{weng2025} & \ding{51} & \ding{55} & \ding{55} & \ding{51} & $10.82$ & $14.41$ & $19.13$ & $25.87$ & $36.03$ & $12.19$ & $14.36$ & $20.34$ & $26.03$ & $35.09$ \\
\rowcolor{gray!20} MotionReward (Ours) & \ding{51} & \ding{51} & \ding{51} & \ding{51} & $\mathbf{11.16}$ & $\mathbf{14.89}$ & $\mathbf{19.89}$ & $\mathbf{26.92}$ & $\mathbf{37.35}$ & $\mathbf{12.60}$ & $\mathbf{14.73}$ & $\mathbf{21.03}$ & $\mathbf{26.71}$ & $\mathbf{36.21}$ \\
\bottomrule
\end{tabular}
\end{adjustbox}
}
\end{table*}

\subsection{Semantic and Physical Motion Quality on MLD}
\label{sec:phys_mld}

Reward hacking typically raises an optimized semantic proxy while degrading out-of-domain physical quality. We report both semantic metrics and physical metrics, including Foot Floating, Jitter Degree, and Dynamic Degree, on HumanML3D under MLD, comparing the base model weight, single-reward EasyTune~\cite{tan2026easytune}, and multi-reward MotionRFT. Physical metrics follow the Motion Quality protocol of MBench in ViMoGen~\cite{lin2025quest}.

As shown in Tab.~\ref{tab:phys_mld}, MotionRFT improves semantic and physical metrics together, rather than collapsing physical quality under a single proxy. Specifically, while single-reward fine-tuning reduces FID to $0.194$, its Jitter Degree slightly increases from $0.0101$ to $0.0127$. In contrast, MotionRFT achieves an FID of $0.137$ and a Top-1 R-Precision of $0.577$, while simultaneously improving Foot Floating to $0.1224$, Jitter Degree to $0.00853$, and Dynamic Degree to $0.0357$ over base MLD.

\begin{table*}[t]
\centering
\renewcommand{\arraystretch}{1.1}
\setlength{\tabcolsep}{4pt}
\caption{\textbf{Semantic and physical motion-quality metrics on HumanML3D under MLD.} Best results among Base MLD, EasyTune, and MotionRFT are in bold.}
\label{tab:phys_mld}
\begin{tabular}{lcccccccc}
\toprule
Method & Top\,1 $\uparrow$ & Top\,2 $\uparrow$ & Top\,3 $\uparrow$ & FID $\downarrow$ & MM Dist $\downarrow$ & Foot Floating $\downarrow$ & Jitter Degree $\downarrow$ & Dynamic Degree $\uparrow$ \\
\midrule
GT & 0.511 & 0.703 & 0.797 & 0.002 & 2.974 & 0.0781 & 0.0219 & 0.0452 \\
Base MLD & 0.504 & 0.698 & 0.796 & 0.450 & 3.052 & 0.1941 & 0.0101 & 0.0341 \\
EasyTune & 0.568 & 0.754 & 0.846 & 0.194 & 2.672 & 0.1685 & 0.0127 & 0.0317 \\
MotionRFT & \textbf{0.577} & \textbf{0.767} & \textbf{0.854} & \textbf{0.137} & \textbf{2.622} & \textbf{0.1224} & \textbf{0.00853} & \textbf{0.0357} \\
\bottomrule
\end{tabular}
\end{table*}

\subsection{Correlation between Step-wise Rewards and Final Quality}
\label{sec:reward_corr}

EasyTune optimizes rewards at each denoising step. This experiment measures the sample-level correlation between the step-wise reward and final clean-motion quality across both continuous alignment (matching score) and discrete retrieval ranking (R-Precision), using pretrained MLD~\cite{chen2023executing} on HumanML3D~\cite{guo2022generating} and following Fig.~\mref{fig:motion_image_cmp} of the main paper. At DDIM step $s$, the one-step prediction is $\hat{\mathbf{x}}_0 = \pi'_\theta(\mathbf{x}_{t_s}, t_s, c)$ as in Eq.~\meqref{eq:reward} of the main paper. We score this prediction with the unified reward in Eq.~\meqref{eq:unified_reward} of the main paper, with $w_1{=}w_4{=}1$ and $w_2{=}w_3{=}0.002$, where the similarity term follows Eq.~\meqref{eq:sim_reward}. Pearson correlation between the step-wise reward $\mathcal{R}$ and the final matching score $m$ is
\[
\eqfit{r(\mathcal{R}, m) = \frac{\sum_{i}(\mathcal{R}_i-\bar{\mathcal{R}})(m_i-\bar{m})}{\sqrt{\sum_{i}(\mathcal{R}_i-\bar{\mathcal{R}})^2}\,\sqrt{\sum_{i}(m_i-\bar{m})^2}},}
\]
where $\mathcal{R}_i=\mathcal{R}_\phi(\hat{\mathbf{x}}_0, 0, c)$ is the unified reward of sample $i$ after one-step prediction, $m_i$ is the final T2M matching score of the fully denoised motion of sample $i$, and $\bar{\mathcal{R}}$ and $\bar{m}$ denote their respective sample means. Spearman correlation is
\[
\eqfit{\rho(\mathcal{R}, m) = r(u, v) = \frac{\sum_{i}(u_i-\bar{u})(v_i-\bar{v})}{\sqrt{\sum_{i}(u_i-\bar{u})^2}\,\sqrt{\sum_{i}(v_i-\bar{v})^2}},}
\]
where $u_i=\mathrm{rank}(\mathcal{R}_i)$ is the rank of $\mathcal{R}_i$, $v_i=\mathrm{rank}(m_i)$ is the rank of $m_i$, and $\bar{u}$ and $\bar{v}$ denote their respective sample means. Because $m_i$ evaluates only the continuous text-motion matching score, we also evaluate discrete ranking performance via R-Precision: let $p_i$ denote the induced R-Precision rank of the fully denoised motion of sample $i$, and let $\tilde{\mathcal{R}}_i$ denote the mean mid-to-late reward of sample $i$. We likewise report $r(\mathcal{R}, p)$ and $\rho(\mathcal{R}, p)$ against $p$, and $r(\tilde{\mathcal{R}}, m)$ and $\rho(\tilde{\mathcal{R}}, m)$ against $m$. As an implementation detail, we use sample-level matching score and R-Precision rank rather than FID because FID is defined per distribution, not per sample.

As shown in Fig.~\ref{fig:reward_corr}, under a $16$-step DDIM schedule the correlation increases rapidly and stabilizes by the middle of the denoising process. Pearson $r$ and Spearman $\rho$ with final matching reach $0.509$ and $0.598$, and those with the R-Precision rank reach $0.421$ and $0.403$. The mean reward over the latter half of the trajectory further correlates with final matching at $r{=}0.509$ and $\rho{=}0.597$. The corresponding $p$-value supports the reliability of this association. These results suggest that intermediate-step rewards may provide a useful signal of final motion quality and support the step-wise optimization design of EasyTune.

\begin{figure*}[t]
\centering
\begin{minipage}[t]{0.32\textwidth}
\centering
\includegraphics[width=\linewidth]{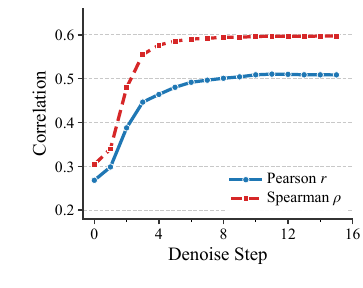}
{\small (a) Matching score and reward}
\end{minipage}\hfill
\begin{minipage}[t]{0.32\textwidth}
\centering
\includegraphics[width=\linewidth]{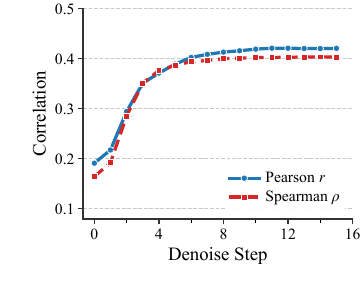}
{\small (b) R-Precision and reward}
\end{minipage}\hfill
\begin{minipage}[t]{0.32\textwidth}
\centering
\includegraphics[width=\linewidth]{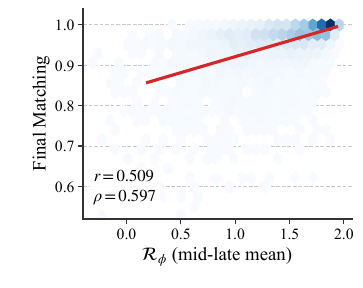}
{\small (c) Matching score and mid-late reward}
\end{minipage}
\caption{Sample-level correlation on HumanML3D under pretrained MLD. (a)~Matching score and reward. (b)~R-Precision and reward. (c)~Matching score and mid-late reward. Solid and dashed curves in (a) and (b) are Pearson $r$ and Spearman $\rho$ versus the denoising step.}
\label{fig:reward_corr}
\end{figure*}

\subsection{Sensitivity of the Semantic Training Coefficients}
\label{sec:lambda_sens}

The semantic reward training objective sets $\lambda_1{=}\lambda_2{=}10^{-5}$ and $\lambda_3{=}\lambda_4{=}10^{-1}$ to balance the magnitudes of the losses, following the conventions of prior works like TMR~\cite{petrovich2023tmr} and ReAlign~\cite{weng2025}. To evaluate the robustness of these settings, we conduct a sensitivity analysis by sweeping the shared scale of the tied VAE regularizers $\lambda_1{=}\lambda_2$ and the alignment weights $\lambda_3$ and $\lambda_4$ independently. As shown in Tab.~\ref{tab:lambda_sens}, the performance is relatively stable across a range of values around our default settings, though extreme values such as $\lambda_3{=}10^{-3}$ lead to performance drops to $85.20$. This suggests that the model may not be overly sensitive to these hyperparameters and that adopting established values to balance loss magnitudes may be an effective strategy.

\begin{table}[t]
\centering
\renewcommand{\arraystretch}{1.12}
\setlength{\tabcolsep}{4.5pt}
\caption{\textbf{Sensitivity of the semantic-training coefficients on HumanML3D under the kinematic protocol of Tab.~\mref{tab:retrieval}.}}
\label{tab:lambda_sens}
\begin{adjustbox}{max width=\columnwidth}
{\footnotesize
\begin{tabular}{c|cccc|c}
\toprule
\multirow{2}{*}{Weight} & \multicolumn{2}{c}{Text--Motion $\uparrow$} & \multicolumn{2}{c|}{Motion--Text $\uparrow$} & \multirow{2}{*}{Avg.\ $\uparrow$} \\
\cmidrule(lr){2-3} \cmidrule(lr){4-5}
 & R@1 & R@3 & R@1 & R@3 & \\
\midrule
\multicolumn{6}{c}{\textit{Tied VAE regularizers $\lambda_1{=}\lambda_2$}} \\
\midrule
$0$ & $66.70$ & $88.99$ & $67.00$ & $89.34$ & $86.19$ \\
$10^{-6}$ & $67.81$ & $89.71$ & $67.97$ & $89.94$ & $86.89$ \\
\rowcolor{gray!20} $10^{-5}$ & $67.74$ & $89.79$ & $67.92$ & $90.03$ & $86.92$ \\
$10^{-3}$ & $65.90$ & $88.42$ & $66.20$ & $88.76$ & $85.65$ \\
\midrule
\multicolumn{6}{c}{\textit{Contrastive weight $\lambda_3$}} \\
\midrule
$10^{-3}$ & $65.21$ & $87.96$ & $65.62$ & $88.30$ & $85.20$ \\
$10^{-2}$ & $67.83$ & $89.98$ & $68.04$ & $90.19$ & $87.04$ \\
\rowcolor{gray!20} $10^{-1}$ & $67.74$ & $89.79$ & $67.92$ & $90.03$ & $86.92$ \\
$1$ & $66.13$ & $88.65$ & $66.43$ & $88.99$ & $85.85$ \\
\midrule
\multicolumn{6}{c}{\textit{Cross-representation weight $\lambda_4$}} \\
\midrule
$0$ & $67.88$ & $89.57$ & $68.01$ & $89.84$ & $86.79$ \\
$10^{-2}$ & $67.78$ & $89.91$ & $67.95$ & $90.12$ & $86.98$ \\
\rowcolor{gray!20} $10^{-1}$ & $67.74$ & $89.79$ & $67.92$ & $90.03$ & $86.92$ \\
$1$ & $67.21$ & $89.38$ & $67.42$ & $89.66$ & $86.55$ \\
\bottomrule
\end{tabular}}
\end{adjustbox}
\end{table}

\subsection{Robustness to Representation Redundancy}
\label{sec:meng_eval}

Meng et al.~\cite{meng2024rethinking} show that the $263$-d HumanML3D representation contains redundant dimensions, and that the conventional evaluator over-responds to those dimensions. To test whether MotionRFT's reported gains depend on this bias, we re-evaluate MLD fine-tuned with the original $263$-d MotionReward under their essential-dimension evaluator, and we further retrain MotionReward on their non-redundant $67$-d representation and fine-tune MLD with that reward. As shown in Tab.~\ref{tab:meng_eval}, both settings improve over base MLD, with FID falling from $0.541$ to $0.134$ and $0.128$ respectively. The gain is therefore not an artifact of the conventional evaluator's preference for redundant dimensions.

\begin{table}[t]
\centering
\renewcommand{\arraystretch}{1.1}
\setlength{\tabcolsep}{5pt}
\caption{\textbf{Evaluation of text-to-motion models on HumanML3D under the essential-dimension evaluator of Meng et al.~\cite{meng2024rethinking}.} The table reports base MLD and MLD fine-tuned with MotionReward trained on the original $263$-d features or on the non-redundant $67$-d essential features.}
\label{tab:meng_eval}
\begin{adjustbox}{max width=\columnwidth}
{\small
\begin{tabular}{lccccc}
\toprule
 & FID $\downarrow$ & R@1 $\uparrow$ & R@2 $\uparrow$ & R@3 $\uparrow$ & Matching $\downarrow$ \\
\midrule
Ground truth & $0.000^{\pm.000}$ & $0.505^{\pm.002}$ & $0.699^{\pm.002}$ & $0.795^{\pm.002}$ & $3.247^{\pm.006}$ \\
MLD & $0.541^{\pm.012}$ & $0.489^{\pm.003}$ & $0.682^{\pm.003}$ & $0.784^{\pm.002}$ & $3.299^{\pm.008}$ \\
\rowcolor{gray!20} MLD + MotionRFT 263-d & $0.134^{\pm.003}$ & $\mathbf{0.561}^{\pm.003}$ & $0.754^{\pm.003}$ & $\mathbf{0.846}^{\pm.002}$ & $\mathbf{2.875}^{\pm.008}$ \\
\rowcolor{gray!20} MLD + MotionRFT 67-d & $\mathbf{0.128}^{\pm.004}$ & $0.559^{\pm.003}$ & $\mathbf{0.757}^{\pm.003}$ & $0.844^{\pm.002}$ & $2.881^{\pm.008}$ \\
\bottomrule
\end{tabular}}
\end{adjustbox}
\end{table}

\section{Additional Discussions}
\label{sec:supp_discuss}

\subsection{Empirical Analysis on Single Canonical Representations}
\label{sec:supp_canon_exp}

Standard motion representations cannot serve as a universal canonical space for step-wise reward modeling because no format simultaneously satisfies differentiability, lossless conversion, and real-time computation due to ill-posed inverse kinematics, unobserved rotational null spaces, and non-differentiable operations. We conduct empirical evaluations comparing our multi-projection design against canonical alternatives.

\noindent\textbf{Cross-modal retrieval performance.} We evaluate text-motion retrieval by retaining only the joint branch of MotionReward. Specifically, kinematic features are decoded to 3D joint coordinates using their $67$ essential dimensions following Meng et al.~\cite{meng2024rethinking}, while rotation features are mapped via differentiable forward kinematics. As reported in Tab.~\ref{tab:single_rep_retrieval}, kinematic retrieval matches the multi-projection MotionReward at text-to-motion R@1 of $67.74$, whereas rotation retrieval drops from $64.19$ to $52.79$. This drop stems from the fact that forward kinematics inherently discards $30$ rotational degrees of freedom, including axial bone twists and orientations of childless leaf joints.

\begin{table*}[t]
\centering
\caption{\textbf{Text-motion retrieval on HumanML3D under multi-projection versus joint-canonical reward models.} Multi-projection denotes MotionReward with dedicated projection heads per representation. Joint-only maps all inputs to 3D joint coordinates and scores them with the joint head alone. \textbf{Note:} For the $263$-d kinematic representation, the joint-only setting uses only the essential $67$ dimensions following Meng et al.~\cite{meng2024rethinking}.}
\label{tab:single_rep_retrieval}
\renewcommand{\arraystretch}{1.12}
\setlength{\tabcolsep}{5pt}
{\footnotesize
\begin{tabular}{l c ccccc ccccc c}
\toprule
\multirow{2}{*}{Setting} & \multirow{2}{*}{Rep.} & \multicolumn{5}{c}{Text-Motion Retrieval $\uparrow$} & \multicolumn{5}{c}{Motion-Text Retrieval $\uparrow$} & \multirow{2}{*}{Avg. $\uparrow$} \\
\cmidrule(lr){3-7} \cmidrule(lr){8-12}
& & R@1 & R@2 & R@3 & R@5 & R@10 & R@1 & R@2 & R@3 & R@5 & R@10 & \\
\midrule
\multicolumn{13}{c}{\textit{Kinematic representation of 263 dimensions}} \\
\midrule
MotionReward (multi-projection) & K & $67.74$ & $83.35$ & $89.79$ & $94.86$ & $98.37$ & $67.92$ & $83.50$ & $90.03$ & $95.16$ & $98.44$ & $86.92$ \\
MotionReward (joint-only) & K & $67.74$ & $83.35$ & $89.79$ & $94.86$ & $98.37$ & $67.92$ & $83.50$ & $90.03$ & $95.16$ & $98.44$ & $86.92$ \\
\midrule
\multicolumn{13}{c}{\textit{Rotation representation of 135 dimensions}} \\
\midrule
MotionReward (multi-projection) & R & $64.19$ & $82.13$ & $89.27$ & $93.75$ & $97.87$ & $63.42$ & $80.85$ & $88.42$ & $94.12$ & $97.72$ & $85.17$ \\
MotionReward (joint-only) & R & $52.79$ & $73.68$ & $81.65$ & $89.78$ & $96.36$ & $52.43$ & $70.59$ & $79.52$ & $87.65$ & $95.15$ & $77.96$ \\
\bottomrule
\end{tabular}}
\end{table*}

\begin{table*}[!t]
\centering
\caption{\textbf{Representation mapping configurations and correlation metrics.}
Features $135$, $66$, $263$, and $67$ denote rotation, 3D joint, complete kinematic, and essential kinematic formats, respectively.
$\mathbf{x}_{A\to 66}$ maps representation $A$ to 3D joints via forward kinematics FK, rotation decoding \textit{rot}, or relative coordinate recovery \textit{ric}.
$\mathcal{R}_{66}$ and $\mathcal{R}_D$ are joint and native reward evaluators with $D \in \{135, 263\}$, and tabulated values are Pearson correlation coefficients $r$ for reward scores and gradients.}
\label{tab:corr_settings}
\renewcommand{\arraystretch}{1.2}
\setlength{\tabcolsep}{6pt}
\resizebox{\linewidth}{!}{%
{\small
\begin{tabular}{ccc}
\toprule
\multicolumn{3}{c}{\textit{Settings (1)--(3)}: Joint-space reward $\mathcal{R}_{66}$ on mapped inputs vs.\ ground-truth joints $\mathbf{x}_{66}$} \\
\midrule
\textbf{Setting (1)} & \textbf{Setting (2)} & \textbf{Setting (3)} \\
$135 \to 66$ (FK) & $263 \to 66$ (rot) & $67 \to 66$ (ric) \\
\addlinespace[3pt]
$r\big(\mathcal{R}_{66}(\mathbf{x}_{135\to 66}),\,\mathcal{R}_{66}(\mathbf{x}_{66})\big){=}0.693$ & $r\big(\mathcal{R}_{66}(\mathbf{x}_{263\to 66}),\,\mathcal{R}_{66}(\mathbf{x}_{66})\big){=}0.951$ & $r\big(\mathcal{R}_{66}(\mathbf{x}_{67\to 66}),\,\mathcal{R}_{66}(\mathbf{x}_{66})\big){=}1.000$ \\
\addlinespace[3pt]
$r\big(\nabla_{\mathbf{x}_{66}}\mathcal{R}_{66}(\mathbf{x}_{135\to 66}),\,\nabla_{\mathbf{x}_{66}}\mathcal{R}_{66}(\mathbf{x}_{66})\big){=}0.667$ & $r\big(\nabla_{\mathbf{x}_{66}}\mathcal{R}_{66}(\mathbf{x}_{263\to 66}),\,\nabla_{\mathbf{x}_{66}}\mathcal{R}_{66}(\mathbf{x}_{66})\big){=}0.849$ & $r\big(\nabla_{\mathbf{x}_{66}}\mathcal{R}_{66}(\mathbf{x}_{67\to 66}),\,\nabla_{\mathbf{x}_{66}}\mathcal{R}_{66}(\mathbf{x}_{66})\big){=}1.000$ \\
\midrule
\multicolumn{3}{c}{\textit{Settings (4)--(6)}: Joint-mapped reward $\mathcal{R}_{66}(\mathbf{x}_{A\to 66})$ vs.\ native reward $\mathcal{R}_{D}(\mathbf{x}_{D})$} \\
\midrule
\textbf{Setting (4)} & \textbf{Setting (5)} & \textbf{Setting (6)} \\
$135 \to 66$ (FK) & $263 \to 66$ (rot) & $67 \to 66$ (ric) \\
\addlinespace[3pt]
$r\big(\mathcal{R}_{66}(\mathbf{x}_{135\to 66}),\,\mathcal{R}_{135}(\mathbf{x}_{135})\big){=}0.588$ & $r\big(\mathcal{R}_{66}(\mathbf{x}_{263\to 66}),\,\mathcal{R}_{263}(\mathbf{x}_{263})\big){=}0.899$ & $r\big(\mathcal{R}_{66}(\mathbf{x}_{67\to 66}),\,\mathcal{R}_{263}(\mathbf{x}_{263})\big){=}0.935$ \\
\addlinespace[3pt]
$r\big(\nabla_{\mathbf{x}_{135}}\mathcal{R}_{66}(\mathbf{x}_{135\to 66}),\,\nabla_{\mathbf{x}_{135}}\mathcal{R}_{135}(\mathbf{x}_{135})\big){=}0.016$ & $r\big(\nabla_{\mathbf{x}_{263}}\mathcal{R}_{66}(\mathbf{x}_{263\to 66}),\,\nabla_{\mathbf{x}_{263}}\mathcal{R}_{263}(\mathbf{x}_{263})\big){=}0.021$ & $r\big(\nabla_{\mathbf{x}_{263}}\mathcal{R}_{66}(\mathbf{x}_{67\to 66}),\,\nabla_{\mathbf{x}_{263}}\mathcal{R}_{263}(\mathbf{x}_{263})\big){=}0.121$ \\
\bottomrule
\end{tabular}
}}%

\end{table*}

\begin{figure*}[!b]
\centering
\includegraphics[width=\linewidth]{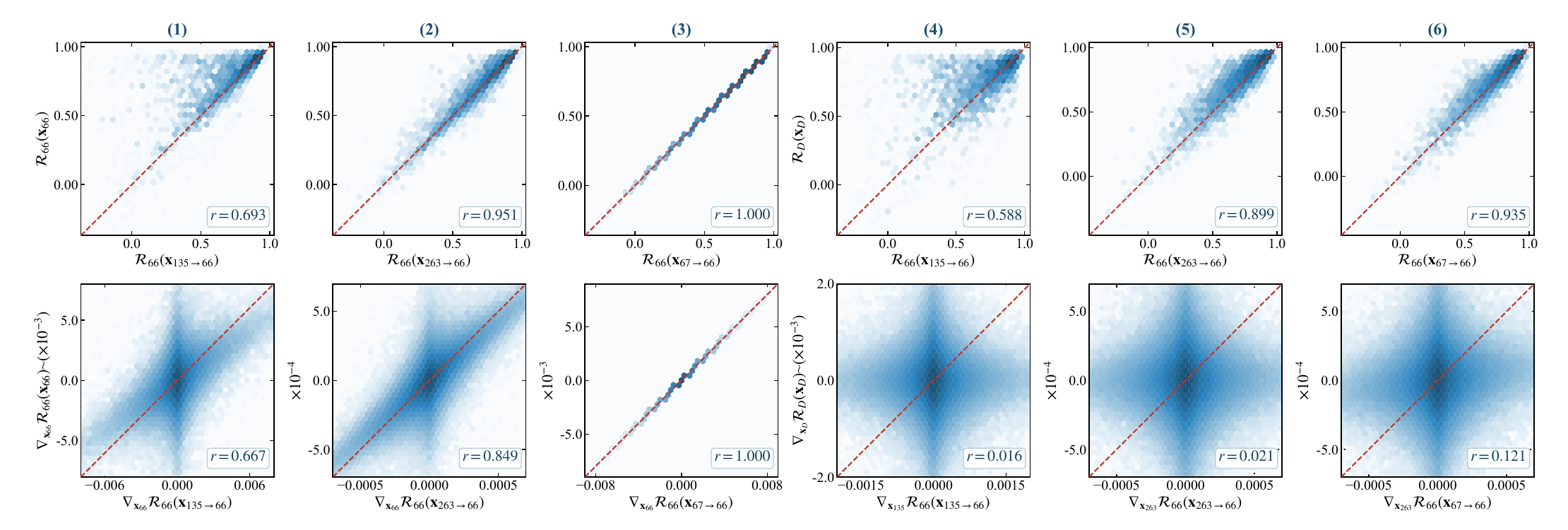}
\caption{\textbf{Pearson correlation and gradient alignment across representation mappings on HumanML3D.} Top row: scalar reward scores for \textbf{Settings (1)--(3)} (left three panels) and \textbf{Settings (4)--(6)} (right three panels). Bottom row: optimization gradient alignment for the same settings. Panel labels match Tab.~\ref{tab:corr_settings}.}
\label{fig:supp_grad_corr}
\end{figure*}

\begin{table*}[t]
\centering
\caption{\textbf{Text-to-motion generation on HumanML3D under multi-projection versus joint-canonical rewards.} Multi-projection denotes MotionRFT with native representation-specific reward heads. Joint-only fine-tunes generators using a single joint-space reward head. Generative Model Output is the generator feature dimension; Reward Input is the feature consumed by the reward.}
\label{tab:single_rep_rft}
\renewcommand{\arraystretch}{1.0}
\setlength{\tabcolsep}{4pt}
{\fontsize{7.5}{8.5}\selectfont
\begin{tabular}{l ccc ccc ccc}
\toprule
\multirow{2}{*}{Method} & \multirow{2}{*}{Reward} & \multirow{2}{*}{\makecell{Generative\\Model Output}} & \multirow{2}{*}{\makecell{Reward\\Input}} & \multicolumn{3}{c}{R-Precision $\uparrow$} & \multirow{2}{*}{FID $\downarrow$} & \multirow{2}{*}{MM Dist $\downarrow$} & \multirow{2}{*}{Diversity $\rightarrow$} \\
\cmidrule(lr){5-7}
& & & & Top1 & Top2 & Top3 & & & \\
\midrule
\multicolumn{10}{c}{\textit{Kinematic representation}} \\
\midrule
Base MLD++ & -- & -- & -- & $0.548$ & $0.738$ & $0.829$ & $0.073$ & $2.810$ & $9.658$ \\
w/ MotionRFT & Kinematic(Native) & $263$ & $263$ & $0.603$ & $0.786$ & $0.872$ & $0.052$ & $2.528$ & $9.683$ \\
w/ MotionRFT & Joint & $263$ & $66$ & $0.580$ & $0.768$ & $0.855$ & $0.062$ & $2.630$ & $9.668$ \\
w/ MotionRFT & Joint & retain $67$ of $263$ & $66$ & $0.596$ & $0.777$ & $0.868$ & $0.061$ & $2.675$ & $9.668$ \\
\midrule
\multicolumn{10}{c}{\textit{Rotation representation}} \\
\midrule
Base HY-Motion-lite & -- & -- & -- & $0.434$ & $0.630$ & $0.721$ & $1.312$ & $3.238$ & $9.607$ \\
w/ MotionRFT & Rotation(Native) & $135$ & $135$ & $0.469$ & $0.676$ & $0.766$ & $1.173$ & $3.028$ & $9.659$ \\
w/ MotionRFT & Joint & $135$ & $66$ & $0.443$ & $0.645$ & $0.735$ & $1.275$ & $3.170$ & $9.622$ \\
\midrule
Base HY-Motion & -- & -- & -- & $0.445$ & $0.651$ & $0.743$ & $1.258$ & $3.104$ & $9.721$ \\
w/ MotionRFT & Rotation(Native) & $135$ & $135$ & $0.499$ & $0.717$ & $0.807$ & $0.985$ & $2.775$ & $9.782$ \\
w/ MotionRFT & Joint & $135$ & $66$ & $0.458$ & $0.670$ & $0.762$ & $1.195$ & $3.005$ & $9.735$ \\
\bottomrule
\end{tabular}}
\end{table*}

\begin{table}[t]
\centering
\caption{\textbf{Text-to-motion retrieval on the full HumanML3D test set using the rotation branch of MotionReward.} Native rotation scores original rotation features. Analytical IK converts other representations to rotation parameters via inverse kinematics before scoring.}
\label{tab:canon_rot}
\renewcommand{\arraystretch}{1.12}
\setlength{\tabcolsep}{5pt}
\begin{adjustbox}{max width=\columnwidth}
{\footnotesize
\begin{tabular}{lcccccc}
\toprule
Input & $N$ & R@1 & R@2 & R@3 & R@5 & R@10 \\
\midrule
Native rotation & $4648$ & $\mathbf{64.19}$ & $\mathbf{82.13}$ & $\mathbf{89.27}$ & $\mathbf{93.75}$ & $\mathbf{97.87}$ \\
Analytical IK & $4648$ & $14.26$ & $23.50$ & $31.64$ & $42.74$ & $61.71$ \\
\bottomrule
\end{tabular}}
\end{adjustbox}
\end{table}

\noindent\textbf{Reward gradients and feature correlations.} To systematically examine whether a single canonical representation can substitute for multi-projection reward modeling, we evaluate reward score correlations and optimization gradient alignments across representations on HumanML3D. The experimental configurations and results are summarized in Tab.~\ref{tab:corr_settings} and Fig.~\ref{fig:supp_grad_corr}.

\textit{Setting.} Tab.~\ref{tab:corr_settings} examines six canonical configurations, specifying evaluation metrics and mappings for $\mathbf{x}_{A\to 66}$, $\mathcal{R}_{66}$, and $\mathcal{R}_{D}$. Settings 1 to 3 evaluate surrogate fidelity against ground-truth joints $\mathbf{x}_{66}$, whereas Settings 4 to 6 measure alignment between joint-mapped rewards and native evaluators $\mathcal{R}_{D}$ on generator outputs.
For kinematic decoding, \textit{rot} uses 130 dimensions out of the complete 263-dimensional vector, namely 4 root trajectory parameters covering yaw velocity, horizontal speeds, and height, alongside 126 local continuous 6D rotations for the 21 body joints, while excluding relative coordinates, velocities, and foot contacts. It temporally integrates root trajectory, forms full 22-joint rotation matrices, and applies forward kinematics along five kinematic chains with template bone lengths to produce 66-dimensional 3D joint coordinates. In contrast, \textit{ric} decodes 3D coordinates from the 67-dimensional essential subset comprising 4 root trajectory parameters and 63 root-relative joint positions. For rotation representations, forward kinematics maps 135-dimensional continuous rotations directly through the kinematic tree.

\textit{Result.} As shown in Tab.~\ref{tab:corr_settings} and Fig.~\ref{fig:supp_grad_corr}, the correlation analysis demonstrates two key findings:
First, joint canonicalization fails for rotation generators. In Setting 4, despite moderate reward correlation of $r=0.588$, gradient correlation drops to $r=0.016$. Because forward kinematics discards unconstrained rotational degrees of freedom such as axial bone twists and leaf joint orientations, the resulting null space deprives rotation policies of meaningful gradient supervision.
Second, joint canonicalization substitutes for kinematic representations only within the non-redundant essential subset. In Setting 3, essential coordinates decoded via \textit{ric} achieve perfect correlation of $r=1.000$ across both reward values and gradients. Although analytical rotation decoding in Setting 2 exhibits high joint-space correlation, it leaves the remaining 133 dimensions of velocities and contacts unconstrained, reducing gradient correlation with the full native reward in Setting 5 to $r=0.021$. Consequently, fine-tuning 263-dimensional kinematic models directly under joint rewards yields sub-optimal synthesis with Top-1 R-Precision of $0.580$, whereas restricting optimization to the 67-dimensional essential coordinates restores native performance at $0.596$.
These results confirm that single joint canonicalization induces severe gradient voids, establishing the technical necessity of our multi-projection framework.

\noindent\textbf{Text-to-motion fine-tuning across base generators.} We next use this joint-only reward to fine-tune pre-trained generators across kinematic and rotation representations. Specifically, kinematic outputs are decoded to 3D joints via essential dimensions and rotation outputs are mapped via differentiable forward kinematics to receive joint-reward gradients. As shown in Tab.~\ref{tab:single_rep_rft}, fine-tuning under the joint-only reward indicates that on kinematic models, joint reward directly on the complete $263$-d representation underperforms, whereas on essential $67$-d coordinates it is fully viable, reaching $0.596$ Top-1 R-Precision and $0.061$ FID. The slight remaining gap may stem from evaluator bias on redundant features as analyzed in Tab.~\ref{tab:meng_eval}. In sharp contrast, rotation generators degrade severely under joint supervision (with HY-Motion Top-1 R-Precision dropping from $0.499$ to $0.458$ and FID degrading from $0.985$ to $1.195$), as forward kinematics discards $30$ rotational degrees of freedom. These findings may support our theoretical analysis: while 3D joint coordinates may substitute for kinematic representations on essential dynamics, they may not replace native rotation modeling. Dedicated native-representation reward heads may remain indispensable for rotation-based generators.

\noindent\textbf{Evaluation on rotation canonicalization via inverse kinematics.} To evaluate whether on-the-fly canonicalization onto rotations is viable, joint coordinates and kinematic features are converted to rotation parameters via analytical inverse kinematics and scored with the same rotation branch used for native rotation features. As reported in Tab.~\ref{tab:canon_rot}, native rotation features reach Top-1 recall of $64.19$ and Top-10 recall of $97.87$, whereas analytical inverse kinematics drops to $14.26$ and $61.71$ on the full test set, failing representation fidelity. Crucially, intermediate canonicalization introduces information loss that distorts optimization gradients, which may support the necessity of our multi-projection framework.

\end{document}